\documentclass[11pt]{article}

\usepackage[english]{babel}
\usepackage{natbib}
\usepackage[letterpaper,margin=1in]{geometry}

\usepackage{newpxtext,newpxmath}

\usepackage{graphicx} 
\usepackage{bmpsize}

\usepackage[
    colorlinks=true,
    linkcolor=blue,
    anchorcolor=blue,
    citecolor=blue
]{hyperref}

\usepackage[shortlabels]{enumitem}

\usepackage{xcolor}

\usepackage{amsthm}

\usepackage{
  mathtools,
  thmtools,
  thm-restate,
  bbm,
  bm,
  cancel,
  url,
}

\newtheorem{theorem}{Theorem}[section]
\newtheorem{lemma}{Lemma}[section]

\newtheorem{corollary}[lemma]{Corollary}

\newtheorem{assumption}{Assumption}

\newenvironment{remark}[1][Remark]
  {
  \begin{proof}[\textnormal{\textbf{#1}}]}
  {\end{proof}}

\newcommand{\eps}{\varepsilon}
\newcommand{\rd}{\mathrm{d}}
\newcommand{\R}{\mathbb{R}}
\renewcommand{\S}{\mbb{S}}
\newcommand{\Id}{\mbf{I}}
\newcommand{\x}{{\mbf{x}}}

\newcommand{\inprod}[2]{\left\langle #1, #2\right\rangle}
\newcommand{\norm}[1]{\left\|#1\right\|}
\newcommand{\abs}[1]{ {\left| #1 \right|} }
\newcommand{\braces}[1]{ \left\{ #1 \right\} }

\newcommand{\inv}{^{-1}}
\newcommand{\trans}{^\top}
\newcommand{\ps}[1]{^{(#1)}}

\newcommand{\indi}{\mathbbm{1}}
\newcommand{\E}{\mathop{\mathbb{E\/}}}
\renewcommand{\P}{\mathop{\mathbb{P\/}}}
\newcommand{\Var}{\mathop{\mathbf{Var\/}}}

\DeclareMathOperator*{\argmin}{argmin}
\DeclareMathOperator*{\argmax}{argmax}
\DeclareMathOperator{\poly}{poly}

\newcommand{\mbb}{\mathbb}
\newcommand{\mrm}{\mathrm}
\newcommand{\mbf}{\bm}
\newcommand{\mcal}{\mathcal}

\newcommand{\tnbf}[1]{\textnormal{\textbf{#1}}}

\newcommand{\X}{\mbf{X}}

\renewcommand{\v}{{\mbf{v}}}
\newcommand{\z}{{\mbf{z}}}

\renewcommand{\a}{\mbf{a}}

\newcommand{\Loss}{\mcal{L}}
\newcommand{\Gaussian}[2]{\mcal{N}\left(#1, #2\right)}

\newcommand{\Term}{\textnormal{\texttt{T}}}

\newcommand{\e}{\mbf{e}}

\newcommand{\cF}{\mathcal{F}}

\newcommand{\bsigma}{\mbf{\sigma}}

\renewcommand{\u}{\mbf{u}}

\newcommand{\V}{\mbf{V}}

\newcommand{\Z}{\mbf{Z}}

\newcommand{\tnabla}{\tilde\nabla}

\newcommand{\IE}{\mrm{IE}}
\newcommand{\Unif}{\mrm{Unif}}
\newcommand{\MSE}{\mrm{MSE}}
\newcommand{\corr}{\mrm{corr}}
\newcommand{\erf}{\mrm{erf}}

\usepackage{graphicx}
\usepackage{caption}
\usepackage{subcaption}

\title{Learning Orthogonal Multi-Index Models: A Fine-Grained Information Exponent Analysis}

\author{
  Yunwei Ren \\ Princeton University \\ \texttt{yunwei.ren@princeton.edu} 
  \and 
  Jason D.~Lee \\ Princeton University \\ \texttt{jasonlee@princeton.edu}
}

\begin{document}

\maketitle

\begin{abstract}
  The information exponent (\cite{ben_arous_online_2021}) and its extensions --- which are equivalent to the lowest 
  degree in the Hermite expansion of the link function (after a potential label transform) for Gaussian single-index 
  models --- have played an important role in predicting the sample complexity of online stochastic gradient descent 
  (SGD) in various learning tasks. In this work, we demonstrate that, for multi-index models, focusing solely on the 
  lowest degree can miss key structural details of the model and result in suboptimal rates. 

  Specifically, we consider the task of learning target functions of form $f_*(\x) = \sum_{k=1}^{P} \phi(\v_k^* \cdot \x)$, 
  where $P \ll d$, the ground-truth directions $\{ \v_k^* \}_{k=1}^P$ are orthonormal, and the information exponent of 
  $\phi$ is $L$. Based on the theory of information exponent, when $L = 2$, only the relevant subspace (not the exact 
  directions) can be recovered due to the rotational invariance of the second-order terms, and when $L > 2$, 
  recovering the directions using online SGD require $\tilde{O}(P d^{L-1})$ samples. In this work, we show that by 
  considering both second- and higher-order terms, we can first learn the relevant space using the second-order
  terms, and then the exact directions using the higher-order terms, and the overall sample and complexity of online 
  SGD is $\tilde{O}( d P^{L-1} )$. 
\end{abstract}

\section{Introduction}

In many learning problems, the target function exhibits or is assumed to exhibit a low-dimensional structure. 
A classical model of this type is the multi-index model, where the target function depends only on a 
$P$-dimensional subspace of the ambient space $\R^d$, with $P$ typically much smaller than $d$. 
When the relevant dimension $P = 1$, the model is known as the single-index model, which dates back to at least 
\cite{ichimura_semiparametric_1993}. 
Both single- and multi-index models have been widely studied, especially in the context of neural network
and stochastic gradient descent (SGD) in recent years, sometimes under the name ``feature learning'' 
\cite{ben_arous_online_2021,bietti_learning_2022,damian_neural_2022,abbe_merged-staircase_2022,abbe_sgd_2023,dandi_how_2024,%
damian_computational-statistical_2024,oko_learning_2024,dandi_benefits_2024}. 

In \cite{ben_arous_online_2021}, the authors show that for single-index models, the behavior of online SGD can be split 
into two phases: an initial ``searching'' phase, where most of the samples are used to boost the correlation with the 
relevant (one-dimensional) subspace to a constant, and a subsequent ``descending'' phase, where the correlation further 
increases to $1$. 
They introduce the concept of the information exponent (IE), defined as the index of the first nonzero coefficient 
in the Taylor expansion of the population loss around $0$, which also corresponds to the lowest degree in the Hermite 
expansion of the link function in Gaussian single-index models. 
They prove that the sample complexity of online SGD is $\tilde{O}(d^{ (\IE-1) \vee 1 })$.
After that, various lower and upper bounds have been established for single-index models in \cite{bietti_learning_2022,%
damian_smoothing_2023,damian_computational-statistical_2024}.
Similar results for certain multi-index models have also been derived in \cite{abbe_merged-staircase_2022,abbe_sgd_2023,%
bietti_learning_2023,oko_learning_2024}. In all cases, the sample complexity of online SGD scales with $d^{\IE - 1}$ 
when $\IE \ge 3$.

Later, it was realized that the notion of information exponent is not stable under modifications of the algorithm. 
In particular, the information exponent of a link function can be greatly reduced by reusing batches or applying a suitable 
label transformation \cite{arnaboldi_repetita_2024,dandi_benefits_2024,lee_neural_2024,damian_computational-statistical_2024}. 
For example, the IE of any fixed degree polynomial can be reduced to at most $2$ via monomial transformations.
This observation leads to the notion of generative exponent (GE) \cite{damian_computational-statistical_2024}, which is 
defined as the lowest information exponent among all $L^2$ transform of the original link function. It yields 
bounds that match the previous results for non-gradient-based methods \cite{chen_learning_2020,troiani_fundamental_2024,%
barbier_optimal_2019}. Despite the improvement over the vanilla information exponent, in the framework of 
generative exponents, still only the lowest order is considered. 
As we will discuss later, this makes it suffer from the 
same issue of information exponent in the context of multi-index models.

Consider multi-index models of form $f_*(\x) = \sum_{k=1}^{P} \phi( \v_k^* \cdot \x )$, where 
$\{\v_k^*\}_k$ are orthonormal vectors.
In this setting, there are two types of recovery: recovering each direction $\v_k^*$ and recovering 
the subspace spanned by $\{ \v_k^* \}_k$. The former notion is stronger, and once the directions are learned, the 
learning task essentially reduces to learning the one-dimensional $\phi: \R \to \R$.
However, directional recovery is not always possible. To see this, consider the case 
$\phi(z) = h_2(z)$, where $h_l$ is the $l$-th (normalized) Hermite polynomial. One can show that this corresponds 
to decomposing the projection matrix (a second-order tensor) of the subspace $\mrm{span}\{\v_k^*\}_k$.  
Hence, recovering the directions is impossible due to the rotational 
invariance (see Section~\ref{subsec: gf: stage 1.1} for more discussion). 
In other words, in any framework that considers only the lowest order (IE or GE), if the lowest order is $2$,\footnote{
  The $\IE = 2$ case is particularly important as the information exponent of many functions, including all fixed degree
  polynomials, can be reduced to at most $2$ by a suitable label transform \cite{damian_computational-statistical_2024}, 
  and the information exponent of any even functions is at least $2$.
}  
we cannot get guarantees beyond subspace recovery due to the existence of the $\phi = h_2$ example. 

On the other hand, if $\phi$ contains some higher-order terms, e.g., $\phi = h_2 + h_4$, then one should expect
that directional convergence is possible, even though $\IE(\phi)$ is still $2$, because of identifiability property 
of (higher-order) orthogonal tensor decomposition problem \cite{ge_learning_2018,li_learning_2020,ge_understanding_2021}. 
In addition, we should be able to first recover the subspace using the second-order terms, which should require
$\tilde{O}(d \poly P)$ samples, and then recover the directions using the higher-order terms. Moreover, since 
we have learned the relevant subspace, the number of samples needed in the second step should be much smaller than what 
is needed if there were no second-order terms. 

In this work, we formally prove the above conjecture and show that the overall sample complexity
is $\tilde{O}( d P^{L-1} )$, where $L$ is the (lowest) order of the higher-order terms. 
Note that this bound scales linearly with the ambient dimension $d$, and it matches the sample complexity of 
separately learning $P$ independent single-index models with the relevant subspace known but the noise scales with 
$d$, up to potential logarithmic terms. More formally, we prove the following theorem.

\begin{theorem}[Informal version of Theorem~\ref{thm: main}]
  Suppose that the target function is given by $f_*(\x) = \sum_{k=1}^{P} \phi(\v_k^* \cdot \x)$ where 
  $\phi = \hat\phi_2 h_2 + \sum_{l=L}^\infty \hat\phi_l h_l$, with $L \ge 3$, $\hat\phi_2^2,\hat\phi_L^2>0$, 
  and $\{\v_k^*\}_{k=1}^P$ are orthonormal, and the input $\x$ follows the standard Gaussian distribution 
  $\Gaussian{0}{\Id_d}$. Then, we can use online SGD (followed by a ridge regression step) to train a two-layer network 
  of width $\tilde{O}(P)$ to learn (with high probability) this target function using $\tilde{O}(d P^{L-1})$ samples and steps. 
\end{theorem}

\paragraph{Organization}
The rest of the paper is organized as follows. First, we review the related works and summarize our contributions.
Then, we describe the detailed setting and state the formal version of the main theorem in Section~\ref{sec: setup}. 
In Section~\ref{sec: the gf analysis}, we discuss the easier case where the training algorithm is population gradient 
flow. Then, in Section~\ref{sec: gf to sgd}, we show how to convert the gradient flow analysis to an online SGD one. 
Finally, we conclude and discuss the limitations in Section~\ref{sec: conclusion}. 
The proofs, simulation results (Appendix~\ref{sec: simulation}), and a table of contents can be found in the appendix.

\subsection{Related work}

In this subsection, we discuss works that are directly related to ours or were not covered earlier in the introduction.

Along the line of information exponent, the paper most related to ours is \cite{oko_learning_2024}. They show that for 
near orthogonal multi-index models, the sample complexity of recovering all ground-truth directions using online SGD is
$\tilde{O}( P d^{\IE - 1} )$ when $\IE \ge 3$. Their results do not apply to the case $\IE = 2$ for the reason 
we have discussed earlier. They propose first removing the second-order terms using the technique in \cite{damian_neural_2022,dandi_how_2024},
which requires $d^2$ samples. Our result considers the situation where both the second and $L$-th order terms are present 
and show that in this case, the sample complexity of online SGD (without any preprocessing) is $\tilde{O}( d P^{L-1} )$. 

Another recent related work is \cite{arous_high-dimensional_2024}.
Our main results are not directly comparable since the settings are different. They run SGD on the Stiefel manifold, 
which automatically prevents the model from collapsing to a single direction, but allow the target model to have condition number larger than $1$. 
In addition, only the lowest degree is considered in their work. However, they also show (in their setting) that 
when the second order term is isotropic, the subspace and only the subspace can be recovered. A similar idea is also used 
in our analysis of Stage~1.1 (cf.~Section~\ref{subsec: gf: stage 1.1}). 

Another related line of research is learning two-layer networks in the teacher-student setting
(\cite{zhong_recovery_2017,li_convergence_2017,tian_analytical_2017,li_learning_2020,zhou_local_2021,ge_understanding_2021}).
Among them, the ones most relevant to this work are \cite{li_learning_2020} and the follow-up \cite{ge_understanding_2021}, 
both of which consider orthogonal models similar to ours and use similar ideas in the analysis of the population process. 
However, they do not assume a low-dimensional structure and only provide very crude $\poly(d)$-style sample complexity 
bounds. 

\subsection{Our contributions}

We summarize our contributions as follows: 

\begin{itemize}[leftmargin=15pt]
  \item We demonstrate that information/generative exponent alone is insufficient to characterize certain structures in the learning 
    task and show that for a specific orthogonal multi-index model, if we consider both the lower- and higher-order
    terms, the sample complexity of directional recovery using online SGD can be greatly improved over the vanilla 
    information exponent-based analysis. 
  \item As a by-product, we derive a collection of user-friendly technical lemmas to analyze the difference between noisy 
    one-dimensional processes and their deterministic counterparts, which may be of independent interests 
    (cf.~Section~\ref{subsec: technical lemmas for analyzing general noisy dynamics} and Appendix~\ref{sec: stochastic induction}). 
\end{itemize}

\section{Setup and main result}
\label{sec: setup}

In this section, we describe the setting of our learning task and the training algorithm, and then formally state our 
main result. We will also convert the problem to an orthogonal tensor decomposition task using the standard Hermite 
argument as in \cite{ge_learning_2018}. 

\paragraph{Notations}
We use $\norm{\cdot}_p$ to denote the $p$-norm of a vector. When $p = 2$, we often drop the subscript and simply write 
$\norm{\cdot}$. For $a, b, \delta \in \R$, $a = b \pm \delta$ means $|a - b| \le |\delta|$ and 
$a\vee b = \max\{a, b\}$ and $a \wedge b = \min\{a, b\}$. Beside the standard 
asymptotic (big $O$) notations, we also use the notation $f_d = O_\phi(g_d)$, which means there exists a constant 
$C_\phi > 0$ that can depend only on $\phi$ such that $f_d \le C_\phi g_d$ for all large enough $d$. Sometimes we
also write $f_d \lesssim_\phi g_d$ for $f_d = O_\phi(g_d)$. The actual value of $C_\phi$ can vary between lines.

\subsection{Input and target function}
We assume the input $\x$ follows the standard Gaussian distribution $\Gaussian{0}{\Id_d}$ and the target function has 
form $f_*(\x) = \sum_{k=1}^{P} \phi(\v_k^* \cdot \x), $
where $\log^C d \le P \le d$ for a large universal constant $C > 0$, $\{\v_k^*\}_{k=1}^P$ are orthonormal and 
$\phi: \R \to \R$ is the link function. In addition, we assume $\phi$ satisfies the following. 

\begin{assumption}[Assumptions on the link function]
  \label{assumption: link function}
  Let $h_k$ denote the degree-$k$ normalized Hermite polynomial and $\phi = \sum_{l=0}^\infty \hat\phi_k h_k$
  denote the Hermite expansion of $\phi \in L^2(\Gaussian{0}{\Id_d})$. 
  \begin{enumerate}[label=(\alph*),leftmargin=*,topsep=0mm]
    \item \tnbf{(IE structure)} For some constant $L > 2$,  $\phi(z) = \hat\phi_2 h_2(z) + \hat\phi_L h_L(z) 
    + \sum_{l > L} \hat\phi_l h_l(z)$. 
    \item \tnbf{(IE regularity)} $\hat\phi_2, \hat\phi_L = \Omega(1)$ and 
      $\norm{ \phi' }_{L^2}^2 = \sum_{l=1}^\infty l \hat\phi_l^2 \le C_{\phi}^2$ for some constant $C_\phi > 0$.
    \item \tnbf{(Polynomial growth)} There exists universal constants $C, q > 0$ such that $|\phi(x)| \vee |\phi'(x)|
    \le C (1 + x^2)^{q/2}$ for all $x \in \R$. 
  \end{enumerate}
\end{assumption}

Our target model and algorithm will all be invariant under rotation. Hence, we will assume w.l.o.g.~that $\v_k^* = \e_k$ where $\{\e_k\}_k$ is the standard basis of $\R^d$. 

\subsection{Learner model and the training algorithm}

Our learner model is a width-$m$ two-layer network
\[
  f(\x) 
  := f(\x; \a, \V) 
  := \sum_{i=1}^{m} a_i \phi(\v_i \cdot \x), 
\]
where $\a = (a_1, \dots, a_m) \in \R^m$ and $\V = (\v_1, \dots, \v_m) \in (\S^{d-1})^m$ are the trainable parameters. 
We call $\{\v_i\}_{i \in [m]}$ the first-layer neurons. We measure the difference between the learner and 
the target model using the correlation loss. Given a sample $(\x, f_*(\x))$, we define the per-sample and population MSE 
loss as 
\[
  l_{\MSE}(\x) 
  := l(\x; \a, \V) 
  := \frac{1}{2} \left( f_*(\x) - f(\x; \a, \V) \right)^2, \quad 
  \Loss_{\MSE}(\a, \V) := \E_{\x} l_{\MSE}(\x; \a, \V).
\]

Now, we describe the training algorithm. First, we initialize each output weight $a_i$ to be $1$. Then, we symmetrically 
initialize the first layer neurons. That is, for $i \in [m/2]$, we initialize $\v_i \sim \Unif(\S^{d-1})$ independently 
and set $\v_{m/2+i} = - \v_i$ for the other half of the neurons. 
After the initialization, we fix the output weights $\a$ and train the first-layer weight $\v_i$ using online 
(spherical) SGD with the correlation loss $l_{\corr}(\x) = -f_*(\x) f(\x)$ and step size $\eta > 0$ for $T$ iterations. 
Then, we fix the first-layer weights and use ridge regression to train the output weights $\a$. 

Let $\{ (\x_t, f_*(\x_t)) \}_{t \in \mbb{N}}$ be our samples where $\{\x_t\}$ are i.i.d.~standard Gaussian vectors,
and let $\tilde{\nabla}_{\v} = (\Id - \v\v\trans) \nabla_{\v}$ denote the spherical gradient.
Then, we can formally describe the training procedure as follows:
\begin{equation}
  \label{eq: training algorithm}
  \begin{aligned}
    & \text{Initialization:} \quad 
      && a_{0, i} = 1, \quad 
      \v_{0, i} \overset{\mrm{i.i.d.}}{\sim} \Unif(\S^{d-1}), \quad 
      \v_{0, m/2+i} = -\v_{0,i} &&& \forall i \in [m/2]; \\
    & \text{Stage 1:} 
      && 
      \left\{
      \begin{aligned}
        & \hat{\v}_{t+1, i} = \v_{t, i} + \eta f_*(\x_t) \nabla_{\v_i}\phi(\v_i\cdot\x), \\
        & \v_{t+1, i} = \hat{\v}_{t+1, i} / \norm{\hat{\v}_{t+1, i}},
      \end{aligned} 
      \right.
      &&& \forall i \in [m], t \in [T]; \\
    & \text{Stage 2:} 
      && \a = \argmin_{\a'} \frac{1}{2 N} \sum_{n=1}^{N} l(\x_{T+n}; \a', \V_T) + \lambda \norm{\a'}^2. 
  \end{aligned}
\end{equation}
Here, the hyperparameters are the network width $m > 0$, step size $\eta > 0$, time horizon $T > 0$, the number of 
samples $N$ in Stage~2, and the regularization strength $\lambda > 0$. 

We will show that after the first stage, for each ground truth direction
$\v_k^*$, there will be some neurons $\v_i$ that has converged to that direction. 
As a result, in the second stage, we can use ridge regression to pick out those neurons and use them to fit the target 
function. The analysis of this second stage is standard and has been done in \cite{damian_neural_2022,abbe_merged-staircase_2022,%
ba_high-dimensional_2022,lee_neural_2024,oko_learning_2024}. Hence, we will not further discuss this stage in the 
main text and defer the proofs of this stage to Appendix~\ref{sec: stage 2}. 

For the gradient update in Stage~1, we have the following lemma on its expectation and tail. 
The proof of this lemma is rather standard and can be found in, for example, \cite{ge_learning_2018,oko_learning_2024}.
We also provide a proof in Appendix~\ref{subsec: population and persample gradients} for completeness. 

\begin{restatable}[First-layer gradients]{lemma}{perSampleGradient}
  \label{lemma: persample gradient}
  Consider the setting described above. 
  Suppose that $\phi$ satisfies Assumption~\ref{assumption: link function} and $a_i = 1$ for all $i \in [m]$
  and $\{\v_k^*\}_k$ are orthonormal. Then, for each $i \in [m]$, we have 
  \begin{equation}
    \label{eq: population gradient}
    \E\left[ f_*(\x) \nabla_{\v_i}\phi(\v_i \cdot \x) \right]
    = 2 \hat\phi_2^2 \sum_{k=1}^{P} \inprod{\v_k^*}{\v_i} \v_k^*
      + \sum_{l\ge L} \sum_{k=1}^{P} l \hat\phi_l^2 \inprod{\v_k^*}{\v_i}^{l-1} \v_k^*. 
  \end{equation}
  Then, for each fixed neuron $(a, \v)$ and direction $\u \in \S^{d-1}$ that is 
  independent of $\x \sim \Gaussian{0}{\Id_d}$, we have 
  \begin{gather*}
    \E \inprod{ f_*(\x) \nabla_{\v_i}\phi(\v_i \cdot \x) }{\u}^2 
    \lesssim_\phi P, \\
    |\inprod{ f_*(\x) \nabla_{\v_i}\phi(\v_i \cdot \x) }{\u}| \lesssim_\phi P^{1/2} \log^{2(1+q)}\log(m/\delta_{\P})
    \quad\text{with probability at least $1 - \delta_{\P}$}.
  \end{gather*}
\end{restatable}
\begin{remark}
  We say a random variable $X$ is $(\sigma^2, \theta)$-subweibull \cite{vladimirova_sub-weibull_2020,kuchibhotla_moving_2022} if 
  \begin{equation}
    \label{eq: subweibull def}
    \P( |X| \ge M )
    \lesssim \exp\left( - (M / \sigma)^{1/\theta} \right), \quad 
    \forall M \ge 0.
  \end{equation}
  Hence, this lemma implies that $\inprod{ f_*(\x) \nabla_{\v_i}\phi(\v_i \cdot \x) }{\u}$ is 
  $(P, 1/(2(1+q)))$-subweibull. 
\end{remark}

\subsection{Main result}
The following is our main result. The proof of it can be found in Appendix~\ref{sec: proof of the main theorem}.
\begin{restatable}[Main Theorem]{theorem}{mainThm}
  \label{thm: main}
  Consider the setting and algorithm described above. Let $C > 0$ be a large universal constant. 
  Suppose that $\log^C d \le P \le d$ and $\{ \v_k^* \}_{k \in [P]}$ are orthonormal. Let $\delta_{\P} \in 
  ( \exp(-\log^C d), 1 )$ and $\eps_* > 0$ be given. Suppose that we choose $a_0, \eta, T, N$ satisfying
  \[
    m = \tilde\Theta( P ), \quad 
    N = \tilde\Theta\left( \frac{P^2}{\eps_*^2 \delta_{\P}^2} \right), \quad 
    \eta = \tilde{\Theta}_\phi\left( \frac{\eps_*^2 \delta_{\P} }{P d P^{L/2-1} } \right), \quad 
    T = \tilde{O}_\phi\left( \frac{P^{L/2-1}}{\eta \eps_*^4 \delta_{\P}} \right). 
  \]
  Then, there exists some $\lambda > 0$ such that at the end of training, we have $\Loss_{\MSE}(\a, \V) \le \eps_*$ with 
  probability at least $1 - O(\delta_{\P})$.
\end{restatable}
\begin{remark}
  Note that $N \ll T$. Hence, the total number of samples needed is $T = \tilde{O}_\phi( d P^{L-1} )$, which matches 
  the sample complexity of separately learning $P$ single-index models with the relevant subspace known 
  \textit{a priori} and the noise scales with the ambient dimension $d$.
\end{remark}

\section{The gradient flow analysis}
\label{sec: the gf analysis}

In this section, we consider the situation where the training algorithm in Stage~1 is gradient flow over the population
correlation loss instead of online SGD. The discussion here is non-rigorous and our formal proof does not rely on anything in 
this section. Nevertheless, this gradient flow analysis will provide valuable intuition on the behavior of online SGD.

For notational simplicity, we will assume w.l.o.g.~that $\v_k^* = \e_k$. 
In addition, we will assume $\phi = h_2 + h_L$ with $L > 2$ for ease of presentation.
Let $\v$ be an arbitrary first-layer neuron. By Lemma~\ref{lemma: persample gradient}, 
the dynamics of $\v$ are controlled by\footnote{In the main text, we use $\tau$ to index the time in this continuous-time process 
(as $t$ has been used to index the steps in the discrete-time process) and will often omit it when it is clear from the context.} 
\[
  \textstyle
  \dot{\v}_\tau 
  \approx 2 \sum_{k=1}^{P} v_k (\Id - \v\v\trans)\e_k + L \sum_{k=1}^{P} v_k^{L-1} (\Id - \v\v\trans) \e_k.  
\]
The second term on the RHS comes from the normalized/projection. For each $k \in [d]$, we have 
\begin{equation}
  \label{eq: gf: d vk2}
  \frac{\rd}{\rd \tau} v_k^2
  \approx 
    2 \indi\{ k \le P \} \left( 2 + L v_k^{L-2}  \right) v_k^2
    - 2 \left( 2 \norm{\v_{\le P}}^2 + L \norm{\v_{\le P}}_{L}^{L} \right) v_k^2. 
\end{equation}
We further split Stage~1 into two substages. In Stage~1.1, the second-order terms dominate and 
$\norm{ \v_{\le P} }^2 / \norm{ \v_{> P} }^2$ grows from $\Theta(P/d)$ to $\Theta(1)$. In Stage~1.2, 
$\v$ converges to one ground-truth direction relying on the signal from the higher-order terms.

The direction to which $\v$ will converge depends on the index of the largest $v_k^2$ at the beginning of Stage~1.2. 
With some standard concentration/anti-concentration argument, one can show that $\max_{k \in [P]} v_k^2$ is at least $1 + c$ times larger
than the second-largest $v_k^2$ for a small constant $c > 0$ with probability at least $1 / \poly(P)$ at the initialization
(of Stage~1.1). Hence, as long as this gap can be preserved throughout Stage~1, we can choose $m = \poly(P)$
to ensure all ground-truth directions can be found after Stage~1.2.

\subsection{Stage 1.1: learning the subspace and preservation of the gap}
\label{subsec: gf: stage 1.1}

In this substage, we track $\norm{\v_{\le P}}^2 / \norm{ \v_{>P} }^2$ and $v_p^2 / v_q^2$ where $p, q \in [P]$
are arbitrary. 
The goal is to show that $\norm{\v_{\le P}}^2 / \norm{ \v_{>P} }^2$ will grow to a constant while 
$v_p^2 / v_q^2$ stay close to its initial value. 

For the norm ratio, by \eqref{eq: gf: d vk2}, we have 
\begin{multline*}
  \frac{\rd}{\rd \tau} \frac{ \norm{\v_{\le P}}^2 }{ \norm{\v_{> P}}^2 }
  = \frac{ \frac{\rd}{\rd \tau} \norm{\v_{\le P}}^2 }{ \norm{\v_{> P}}^2 }
    - \frac{ \norm{\v_{\le P}}^2 }{ \norm{\v_{> P}}^2 }
      \frac{ \frac{\rd}{\rd \tau} \norm{\v_{> P}}^2 }{ \norm{\v_{> P}}^2 }   
  = \frac{ 4 \norm{\v_{\le P}}^2 }{ \norm{\v_{> P}}^2 } 
      + \frac{ 2  L \norm{\v}_L^L  }{ \norm{\v_{> P}}^2 } \\ 
      - \cancel{ \frac{ 2 \left( 2 \norm{\v_{\le P}}^2 + L \norm{\v_{\le P}}_{L}^{L} \right) \norm{\v_{\le P}}^2 }{ \norm{\v_{> P}}^2 } }  
        + \cancel{ 
          \frac{ \norm{\v_{\le P}}^2 }{ \norm{\v_{> P}}^2 }
          \frac{ 
            2 \left( 2 \norm{\v_{\le P}}^2 + L \norm{\v_{\le P}}_{L}^{L} \right) \norm{\v_{> P}}^2
          }{ \norm{\v_{> P}}^2 }  
        }. 
\end{multline*}
In particular, note that the terms coming from normalization cancel with each other. Moreover, 
this implies $\frac{\rd}{\rd \tau} \frac{ \norm{\v_{\le P}}^2 }{ \norm{\v_{> P}}^2 } \ge 4  \frac{ \norm{\v_{\le P}}^2 }{ \norm{\v_{> P}}^2 }$, 
and therefore, it takes only at most $\frac{1+o(1)}{4}\log(d/P) = \Theta(\log(d/P))$ amount of time for the ratio to grow from $\Theta(P/d)$ to $\Theta(1)$. 
If we choose a small step size $\eta$ so that online SGD closely tracks the gradient flow, then the number of 
steps one should expect is $O( \log(d/P) / \eta ) $. 

Meanwhile, for any $p, q \in [P]$, we have 
\begin{align*}
  \frac{\rd}{\rd \tau} \frac{v_p^2}{v_q^2}
  &= 2 \left( 2 + L v_p^{L-2}  \right) \frac{v_p^2}{v_q^2}
    - 2 \left( 2 \norm{\v_{\le P}}^2 + L \norm{\v_{\le P}}_{L}^{L} \right)
    \frac{ v_p^2 }{v_q^2} \\
    &\qquad
    - \frac{v_p^2}{v_q^2} \left(
        2 \left( 2 + L v_q^{L-2}  \right) 
        - 2 \left( 2 \norm{\v_{\le P}}^2 + L \norm{\v_{\le P}}_{L}^{L} \right) 
    \right) 
  =  2 L \left( v_p^{L-2} - v_q^{L-2} \right) \frac{v_p^2}{v_q^2}.
\end{align*}
Note that not only those terms coming from normalization cancel with each other, but also the second-order terms. 
In particular, this also implies that we cannot learn the directions using only the second-order terms. 
At initialization, with high probability $v_k^2 = \tilde{O}(1/d)$ for all $k \in [P]$.
Hence, if we assume the induction hypothesis $v_p^2 = \tilde{O}(1/P)$, then above will become 
$\frac{\rd}{\rd \tau} v_p^2 / v_q^2 \lesssim P\inv v_p^2/v_q^2$. As a result, 
$v_{t, p}^2 / v_{t, q}^2 \le (1 + o(1)) v_{0, p}^2 / v_{0, q}^2$ for any $t \le \Theta(\log(d/P))$,  
as long as $P \ge \poly\log d$. 

\subsection{Stage 1.2: learning the directions}
Let $\v$ be a first-layer neuron with $v_1^2 \ge (1 + c) \max_{2 \le k \le P} v_k^2$ for some small constant
$c > 0$ at initialization. By our previous discussion, we know at the end of Stage~1.1, the above bound still holds 
with a potentially smaller constant $c > 0$. In addition, since $\norm{\v_{\le P}}^2 = \Theta(1)$, we also have $v_1^2 
\ge \Omega(1/P)$ at the end of Stage~1.1. We claim that $\v$ will converge to $\e_1$. The argument here is similar to
the proofs in \cite{li_learning_2020} and \cite{ge_understanding_2021}. 

Again, by \eqref{eq: gf: d vk2}, we have 
\[
  \frac{\rd}{\rd \tau} v_1^2
  \approx 
    2 \left( 
      2 - 2 \norm{\v_{\le P}}^2
      + L v_1^{L-2} - L \norm{\v_{\le P}}_{L}^{L} 
    \right) v_1^2
  \ge   
    2 L \left( 
      v_1^{L-2} - \norm{\v_{\le P}}_{L}^{L} 
    \right) v_1^2
\]
Assume the induction hypothesis $v_1^2 \ge (1 + c) \max_{2 \le k \le P} v_k^2$ and write 
\[ 
  v_1^{L-2} - \norm{\v_{\le P}}_L^L 
  = v_1^{L-2} \left( 1 - v_1^2 \right) 
    - \left( \norm{\v_{\le P}}^2 - v_1^2 \right) \sum_{k=2}^{P} \frac{v_k^2}{ \norm{\v_{\le P}}^2 - v_1^2 } v_k^{L-2}
\] 
Note that the summation is a weighted average of $\{ v_k^{L-2} \}_{k \ge 2}$ and therefore can be upper bounded by
$\left( v_1^2 / (1 + c) \right)^{L/2-1} \le (1 - c_L) v_1^{L-2}$ for some constant $c_L > 0$ that can only depend on $L$. 
Thus, we have 
\[
  \frac{\rd}{\rd \tau} v_1^2
  \gtrsim   
    2 L \left( 
      v_1^{L-2} \left( 1 - v_1^2 \right) 
      - \left( \norm{\v_{\le P}}^2 - v_1^2 \right) (1 - c_L) v_1^{L-2}
    \right) v_1^2
  \ge 2 c_L L \left( 1 - v_1^2 \right) v_1^L.
\]
When $v_1^2 \le 3/4$, this implies $\frac{\rd}{\rd \tau} v_1^2 \gtrsim_L v_1^{L}$. As a result, it takes at most 
$O_L( P^{L/2-1} )$ amount of time for $v_1^2$ to grow from $\Omega(1/P)$ to $3/4$ under gradient flow.
It is important that $v_1^2 = \Omega(1/P)$ instead of $\Omega(1/d)$ at the start of Stage~1.2, since otherwise the time 
needed will be $O_L(d^{L-1})$. After $v_1^2$ reaches $3/4$, we have $\frac{\rd}{\rd \tau} (1 - v_1^2) \lesssim_L 
- \left( 1 - v_1^2 \right)$. Thus, $v_1^2$ will converge linearly to $1$ afterwards.

\section{From gradient flow to online SGD}
\label{sec: gf to sgd}

In this section, we discuss how to convert the previous gradient flow analysis to an online SGD one. Our actual proof 
will be based directly on the online SGD analysis, but the overall idea is still proving that the online SGD dynamics 
of certain important quantities closely track their population gradient descent (GD) counterparts. Our choice of 
learning rate $\eta$ will be much smaller than what needed for GD to track GF, so the bottleneck comes from the 
GD-to-SGD conversion, not the GF-to-GD one. Provided that SGD tracks GD well, the number of 
steps/samples it needs to finish each substage is roughly the amount of time GF needs, divided by the step size $\eta$. 

The rest of this section is organized as follows. In Section~\ref{subsec: technical lemmas for analyzing general noisy 
dynamics}, we collect a few useful lemmas for controlling the difference between noisy dynamics 
and their deterministic counterparts. The idea behind them has appeared in \cite{ben_arous_online_2021} and is also 
used in \cite{abbe_merged-staircase_2022}. Here, we simplify and slightly generalize their argument and provide 
a user-friendly interface. When used properly, it reduces the GD-to-SGD proof to routine calculus. Then, in 
Section~\ref{subsec: sample complexity of online sgd}, we discuss how to apply those general results to analyze the 
dynamics of online SGD in our setting. 

\subsection{Technical lemmas for analyzing general noisy dynamics}
\label{subsec: technical lemmas for analyzing general noisy dynamics}

We start with the lemma that will be used to analyze $\norm{\v_{\le P}}^2 / \norm{\v_{> P}}^2$. The formal proofs of it 
and all other lemmas in this subsection can be found in Section~\ref{sec: stochastic induction}.

\begin{restatable}[Stochastic Gronwall's lemma]{lemma}{stochasticInductionGronwall}
  \label{lemma: stochastic discrete gronwall}
  Suppose that $(X_t)_t$ satisfies
  \begin{equation}
    \label{eq: stochastic gronwall, Xt+1 =}
    X_{t+1} = (1 + \alpha) X_t + \xi_{t+1} + Z_{t+1},  
    \quad X_0 = x_0 > 0,
  \end{equation}
  where the signal growth rate $\alpha > 0$ and initialization $x_0 > 0$ are given, $(\xi_t)_t$ is an 
  adapted process, and $(Z_t)_t$ is a martingale difference sequence. Define $x_t = (1 + \alpha)^t x_0$.
  
  Let $T > 0$ and $\delta_{\P} \in (0, 1)$ be given. 
  Suppose that there exists some $\delta_{\P, \xi} \in (0, 1)$ and $\Xi, \sigma_Z > 0$ such that for every $t \ge 0$,
  if $X_t = (1 \pm 0.5) x_t$, then we have $|\xi_{t+1}| \le (1 + \alpha)^t \Xi$ with probability at least $1 - \delta_{\P, \xi}$ 
  and $Z_{t+1}$ is conditionally $((1 + \alpha)^t \sigma_Z^2, \theta)$-subweibull. 
  Then, if 
  \begin{equation}
    \label{eq: conditions of stochastic gronwall}
    \Xi \lesssim \frac{x_0}{T}
    \quad\text{and}\quad 
    \sigma_Z^2  
    \lesssim \frac{x_0^2}{T \log^{\theta+1}(T / \delta_{\P}) }, 
  \end{equation}
  we have $X_t = (1 \pm 0.5) x_t$ for all $t \in [T]$ with probability at least $1 - \delta_{\P} - T \delta_{\P, \xi}$. 
\end{restatable}
\begin{remark}[Condition \eqref{eq: conditions of stochastic gronwall}]
  One may interpret $Z_{t+1}$ as those terms coming from the difference between the population and mini-batch gradients,
  whose variance is typically quadratic in $\eta$, and $\xi_{t+1}$ as the higher-order error terms. $\alpha$ is usually small. 
  In our case, it is proportional to the step size $\eta$. $T$ is usually the time needed for $X_t$ to grow from a small 
  $x_0 > 0$ to $\Theta(1)$, which is roughly $\alpha\inv \log(1/x_0)$. Since the LHS' of \eqref{eq: conditions of stochastic gronwall}
  are $O(\eta^2)$ while the RHS' are $\Omega(\eta)$, \eqref{eq: conditions of stochastic gronwall} can be alternatively viewed 
  as a condition on $\eta$.
\end{remark}
\begin{remark}[Stochastic induction]
  One important feature of this lemma is that it only requires the bounds $|\xi_{t+1}| \le (1 + \alpha)^t \Xi$ and 
  $\E[Z_{t+1}^2 \mid \cF_t] \le (1 + \alpha)^t \sigma_Z^2$ to hold when $X_t = (1 \pm 0.5) x_t$. This can be viewed 
  as a form of induction and is particularly useful when considering the dynamics of, say, $v_k^2$. 
  Similar to how the RHS of $\frac{\rd}{\rd\tau} v_{\tau, k}^2 = 2 v_{\tau, k} \dot{v}_{\tau, k}$ depends on $v_{\tau, k}$, 
  the size of $\xi_{t+1}, Z_{t+1}$ will usually depend on $X_t$. Hence, we will not be able to bound them without 
  suitable induction hypotheses. 
\end{remark}
\begin{remark}[Remark on the subweibull condition]
  We assume the martingale difference terms $(Z_{t+1})_t$ are conditionally subweibull. This allows us to get 
  poly-logarithmic dependence on $\delta_{\P}$, which is important as we will eventually take union bound over 
  $\poly P$ events. One may replace this condition with the weaker condition
  $\E[ Z_{t+1}^2 \mid \cF_t ] \le (1 + \alpha)^t \sigma_Z^2$. This will lead to a linear dependence on $\delta_{\P}$.
\end{remark}

\begin{proof}[Proof sketch of Lemma~\ref{lemma: stochastic discrete gronwall}]
  For the ease of presentation, we assume that $|\xi_{t+1}| \le (1 + \alpha)^t \Xi$ with probability at least
  $1 - \delta_{\P, \xi}$ and $\E[Z_{t+1}^2 \mid \cF_t] \le (1 + \alpha)^t \sigma_Z^2$ always hold. This step can be 
  made formal using a stopping time argument. See Section~\ref{sec: stochastic induction} for details. Then, 
  Unroll \eqref{eq: stochastic gronwall, Xt+1 =} to obtain
  \(
    X_{t+1}
    = (1 + \alpha)^{t+1} x_0 
      + \sum_{s=1}^t (1 + \alpha)^{t-s} \xi_{s+1} 
      + \sum_{s=1}^t (1 + \alpha)^{t-s} Z_{s+1}.  
  \) 
  Divide both sides with $(1 + \alpha)^{t+1}$ and replace $t+1$ with $t$. Then, the above becomes 
  \[
    \textstyle
    X_t (1 + \alpha)^{-t} 
    = x_0 
      + \sum_{s=1}^t (1 + \alpha)^{-s} \xi_s 
      + \sum_{s=1}^t (1 + \alpha)^{-s} Z_s.  
  \]
  The second term is bounded by $T \Xi$ (uniformly over $t \le T$) with probability at least $1 - T \delta_{\P, \xi}$. 
  Note that $(1 + \alpha)^{-s} Z_s$ is still a martingale difference sequence. Hence, by Doob's $L^2$-submartingale 
  inequality, the third term is bounded by $x_0/4$ with probability at least $16 \sigma_Z^2 / (\alpha x_0^2)$. 
  Thus, when \eqref{eq: conditions of stochastic gronwall} holds, the RHS is $(1 \pm 0.5) x_0$ with probability at 
  least $1 - T \delta_{\P, \xi} - \delta_{\P}$. Multiply both sides with $(1 + \alpha)^t$, and we complete the proof.
  To improve the dependence on $\delta_{\P}$ from linear to poly-logarithmic, it suffices to replace Doob's 
  $L^2$-submartingale inequality with a variant of Freedman's inequality that works with subweibull variables
  (cf.~Appendix~\ref{sec: stochastic induction}). 
\end{proof}

Using the same strategy, one can prove a similar lemma that 
deals with the case $\alpha = 0$, which will be used to show the preservation of the gap in Stage~1.1. Another 
interesting case is where the growth is not linear but polynomial. This is the case of Stage~1.2 in our setting. 
For this case, we have the following lemma.\footnote{In an early version of this manuscript, we did not relax the 
conditions on the noises when $X_t$ grows as in Lemma~\ref{lemma: stochastic discrete gronwall}. We thank Eshaan Nichani
for pointing out that this would result in a suboptimal rate.}

\begin{restatable}{lemma}{stochasticInductionPoly}
  \label{lemma: stochastic gronwall (polynomial)}
  Let $(X_t)_t$ be a non-negative stochastic process satisfying
  \begin{equation}
    \label{eq: stochastic gronwall (polynomial)}
    X_{t+1}
    \ge X_t + \alpha X_t^p + Z_{t+1} + \xi_{t+1}, 
    \quad 
    X_0 = x_0 > 0,
  \end{equation}
  where $\alpha > 0$, $(Z_{t+1})_t$ is a martingale difference sequence, and $(\xi_t)_t$ is an adapted process. 
  Let $\hat{x}_t$ be the solution to the deterministic recurrence relationship
  \(
    \hat{x}_{t+1} = \hat{x}_t + \alpha \hat{x}_t^p, \hat{x}_0 = x_0 / 2.
  \)

  Let $\delta_{\P} \in (0, 1)$ be given and 
  \(
    T := \inf\braces{ t \lesssim \left( p \alpha (x_0/2)^{p-1}\right)\inv \,:\, X_t \ge 1 }.
  \)
  Suppose that there exists $\Xi, \sigma_Z > 0$ and $\delta_{\P, \xi} \in (0, 1)$ such that if 
  $X_t \ge \hat{x}_t$ and $t \le T$, we have $|\xi_t| \le \Xi X_t$ with probability at least $1 - \delta_{\P, \xi}$ and 
  $Z_{t+1}$ is conditionally $(\sigma_Z^2 X_t, \theta)$-subweibull.
  Then, if 
  \begin{equation}
    \label{eq: stochastic polynomial gronwall conditions}
    \alpha \lesssim x_0^{p-1}/p, \quad 
    \Xi \lesssim p \alpha x_0^{p-1}, \quad 
    \sigma_Z^2 \lesssim p \alpha x_0^p \poly\log(T/\delta_{\P}), 
  \end{equation}
  we have $X_t \ge \hat{x}_t$ for all $t \le T$ and $X_t \ge 1$ with 
  probability at least $1 - T \delta_{\P, \xi} - \delta_{\P}$. 
\end{restatable}
The proof of this lemma can be found in Section~\ref{sec: stochastic induction}. 
It is similar to the previous proof in spirit: we replace $(1 + \alpha)^t$ with 
$\prod_{s=0}^{t-1} (1 + \alpha X_s^{p-1})$ and unroll the recurrence. However, unlike the linear case, it is generally 
difficult to upper bound the difference between $X_t$ and $\hat{x}_t$, as this type of polynomial systems exhibit
sharp transitions and blow up in finite time. Consequently, $|\xi_t| \le \Xi X_t$ and $\E[ Z_{t+1}^2 \mid \cF_t ] \le \sigma_Z^2 X_t$
do not directly imply that $|\xi_t| \lesssim \Xi \hat{x}_t$ and $\E[ Z_{t+1}^2 \mid \cF_t ] \lesssim \sigma_Z^2 \hat{x}_t$,
and this makes the analysis tricky as the RHS' are not deterministic. To handle this issue, we use the following
recoupling argument: whenever $X_t \ge 4 \hat{x}_t$, we replace the current $\hat{x}_t$ with $X_t/2$. Clear that 
this can only increase $\hat{x}_t$, and it ensures $X_t \le 4 \hat{x}_t$ always holds. Moreover, after each 
recoupling, $\hat{x}_t$ will at least double. As a result, the conditions we need to absorb the noises will also 
become weaker.

\subsection{Sample complexity of online SGD}
\label{subsec: sample complexity of online sgd}

In this subsection, we demonstrate how to use the previous results to obtain results for online SGD and discuss why 
the sample complexity is $\tilde{O}(d P^{L-1})$ instead of $\tilde{O}(d^{L-1})$ even though we are relying on the 
$L$-th order terms to learn the directions.

\subsubsection{A simplified version of Stage~1.1}
As an example, we consider the dynamics of $P v_p^2 / (d v_q^2)$ where $p \le P$ and $q > P$ and assume both of $v_p$ and $v_q$ are 
small and $P v_p^2 / (d v_q^2) \le 1$. 
This can be viewed as a simplified version of the analysis of $\norm{\v_{\le P}}^2 / \norm{\v_{> P}}^2$ in Stage~1.1.
The analysis of other quantities/stages is essentially the same --- we rewrite the update rule to single out martingale  
difference terms and the higher-order error terms, and apply a suitable lemma from the previous 
subsection (or Section~\ref{sec: stochastic induction}) to complete the proof. 

For the ease of presentation, in this subsection, we ignore the higher-order terms. In particular, we assume the approximation
\(
  \hat{v}_{t+1,k}
  \approx v_{t, k} + 2 \eta \left( \indi\{k \le P\} - \norm{\v_{\le P}}^2 \right) + \eta Z_{t+1, k}, 
\) 
for all $k \in [d]$, 
where $Z_{t+1, k}$ represents the difference between the population and mini-batch gradients. Then, we compute 
\[
  \hat{v}_{t+1,k}^2
  \approx \left( 1 + 4 \eta \left( \indi\{k \le P\} - \norm{\v_{\le P}}^2 \right) \right) v_k^2 
    + 2 \eta v_k Z_k 
    \pm C_L \eta^2 (1 \vee Z_k^2). 
\]
Here, the last term is the higher-order term and will eventually be included in $\xi$. For simplicity,
we will also ignore them in the following discussion. 
The second term is the martingale difference term. Its (conditional) variance depend on $v_k$, and this necessitates 
the induction-style conditions in Lemma~\ref{lemma: stochastic discrete gronwall}. Note that 
$v_{t+1, p}^2 / v_{t+1, q}^2 = \hat{v}_{t+1, p}^2 / \hat{v}_{t+1, q}^2$. Hence, we have 
\[ 
  \frac{v_{t+1, p}^2}{v_{t+1, q}^2}
  \approx 
    \frac{
      \left( 1 + 4 \eta \left( 1 - \norm{\v_{\le P}}^2 \right) \right) v_p^2 
      + 2 \eta v_p Z_p 
    }{
      \left( 1 - 4 \eta \norm{\v_{\le P}}^2 \right) v_q^2 
      + 2 \eta v_q Z_q 
    }. 
\]
Repeatedly use the elementary identity 
$ \frac{1}{a + \delta} = \frac{1}{a} \left( 1 - \frac{\delta}{a} \left( 1 - \frac{\delta}{a + \delta} \right) \right)
\approx \frac{1}{a} \left( 
  1 
  - \frac{\delta}{a}  
\right)$ for any $a > 0$ and small $\delta > 0$, we can rewrite the above equation as 
\[
  \frac{P v_{t+1, p}^2}{d v_{t+1, q}^2}
  \approx 
    \left( 1 + 4 \eta \right) \frac{ P  v_p^2 }{d v_q^2 }
    - \frac{ P  v_p^2 }{d v_q^2 } \frac{2 \eta v_q Z_q }{ v_q^2 } 
    + \frac{
      2 P \eta v_p Z_p 
    }{d v_q^2 }. 
\]
Suppose that $v_p^2 \approx v_q^2$ at initialization and assume the induction hypothesis $P v_p^2/ (d v_q^2) 
= (1 \pm 0.5) (1 + 4 \eta)^t P v_{0, p}^2/ (d v_{0, q}^2) $. Then, by Lemma~\ref{lemma: persample gradient}, the conditional variance of the martingale 
difference terms (the last two terms) is bounded by $O_L( (1 + 4 \eta)^t \eta^2 P^2 / d  )$. Using the language of 
Lemma~\ref{lemma: stochastic discrete gronwall}, this means $\sigma_Z^2 \le O_L(\eta^2 P^2 / d )$. 
Meanwhile, by our gradient flow analysis, the number steps Stage~1.1 needs is roughly $\log d / \eta$.
Hence, in order for (the second condition of) \eqref{eq: conditions of stochastic gronwall} to hold, it suffices to choose 
$\eta = \tilde{O}(1/d)$. One can also show that for the higher-order terms to be small, it suffices to choose 
$\eta = \tilde{O}(1/(d P))$. As a result, for Stage~1.1, the sample complexity is $\tilde{O}( d P )$.

\subsubsection{The improved sample complexity for Stage 1.2}
To see why the existence of the second-order terms can reduce the sample complexity from $d^{\IE - 1}$ to $d \poly(P)$, 
first note that after Stage~1.1, $\max_{p \in [P]} v_p^2$ will be $\Omega(1/P)$. Also 
note that the conditions in Lemma~\ref{lemma: stochastic gronwall (polynomial)} depend on the initial value. 
With the initial value being $\Omega(1/P)$ instead of $\tilde{O}(1/d)$, the largest possible step size we can choose 
will be $O(1) / (P d P^{L/2-1})$, which is much larger than the usual $O( 1 / (P d^{L/2}) )$ requirement from the vanilla 
information exponent argument. Meanwhile, by our gradient flow analysis, we know the number of iterations needed 
is $O(P^{L/2-1} / \eta)$. Combine these and we obtain the $\tilde{O}( d P^{L-1}  )$ sample complexity.

\section{Conclusion and limitations}
\label{sec: conclusion}

In this work, we study the task of learning multi-index models of form $f_*(\x) = \sum_{k=1}^{P} \phi(\v_k^* \cdot \x)$
with $P \ll d$, $\{\v_k^*\}_k$ be orthogonal and $\phi = \hat{\phi}_2 h_2 + \sum_{l=L}^\infty \hat\phi_l h_l$. By considering both the lower- and 
higher-order terms, we prove an $\tilde{O}(d \poly(P))$ bound on the sample complex for strong recovery of 
directions using online SGD, which improves the results one can obtain using vanilla information exponent-based analysis.

The main limitation of this work is the orthogonality condition. This can potentially be relaxed to near-orthogonality
as in \cite{oko_learning_2024}. Extending this result beyond near-orthogonal teacher neurons is an interesting but 
challenging future direction, as when the teacher neurons are not near-orthogonal, this task is hard in general. 
However, we conjecture that when the target model has a hierarchical structure across different orders, online SGD can 
gradually learn the directions using those terms of different order sequentially. 

Another limitation of this work is the assumption that the signal strengths are isotropic. When this is not true, 
training with the second-order terms and correlation loss will make all neurons collapse to the largest direction 
or require $d^2$ samples if we perform only one gradient step \cite{damian_neural_2022,dandi_how_2024}. That being 
said, it is still reasonable to expect the overall sample complexity to be improved if 
we can leverage the second-order terms properly. Formally establishing this is also a potential future direction.

\section*{Acknowledgements}

JDL acknowledges support of  NSF CCF 2002272, NSF IIS 2107304,  NSF CIF 2212262, ONR Young Investigator Award, 
and NSF CAREER Award 2144994. We thank Eshaan Nichani for pointing out an earlier version of 
Lemma~\ref{lemma: stochastic gronwall (polynomial)} is suboptimal, and anonymous reviewer~4V3d for helpful discussions
on relaxing Assumption~\ref{assumption: link function}.

\bibliography{reference}

\newcommand{\etalchar}[1]{$^{#1}$}
\begin{thebibliography}{DPVLB24}

\bibitem[AAM22]{abbe_merged-staircase_2022}
Emmanuel Abbe, Enric~Boix Adsera, and Theodor Misiakiewicz.
\newblock The merged-staircase property: a necessary and nearly sufficient condition for {SGD} learning of sparse functions on two-layer neural networks.
\newblock In {\em Proceedings of {Thirty} {Fifth} {Conference} on {Learning} {Theory}}, pages 4782--4887. PMLR, June 2022.
\newblock ISSN: 2640-3498.

\bibitem[AAM23]{abbe_sgd_2023}
Emmanuel Abbe, Enric~Boix Adserà, and Theodor Misiakiewicz.
\newblock {SGD} learning on neural networks: leap complexity and saddle-to-saddle dynamics.
\newblock In {\em Proceedings of {Thirty} {Sixth} {Conference} on {Learning} {Theory}}, pages 2552--2623. PMLR, July 2023.
\newblock ISSN: 2640-3498.

\bibitem[ADK{\etalchar{+}}24]{arnaboldi_repetita_2024}
Luca Arnaboldi, Yatin Dandi, Florent Krzakala, Luca Pesce, and Ludovic Stephan.
\newblock Repetita {Iuvant}: {Data} {Repetition} {Allows} {SGD} to {Learn} {High}-{Dimensional} {Multi}-{Index} {Functions}.
\newblock June 2024.

\bibitem[BAGJ21]{ben_arous_online_2021}
Gerard Ben~Arous, Reza Gheissari, and Aukosh Jagannath.
\newblock Online stochastic gradient descent on non-convex losses from high-dimensional inference.
\newblock {\em Journal of Machine Learning Research}, 22(106):1--51, 2021.

\bibitem[BAGP24]{arous_high-dimensional_2024}
Gérard Ben~Arous, Cédric Gerbelot, and Vanessa Piccolo.
\newblock High-dimensional optimization for multi-spiked tensor {PCA}, August 2024.
\newblock arXiv:2408.06401 [cs, math, stat].

\bibitem[BBPV23]{bietti_learning_2023}
Alberto Bietti, Joan Bruna, and Loucas Pillaud-Vivien.
\newblock On {Learning} {Gaussian} {Multi}-index {Models} with {Gradient} {Flow}, November 2023.
\newblock arXiv:2310.19793.

\bibitem[BBSS22]{bietti_learning_2022}
Alberto Bietti, Joan Bruna, Clayton Sanford, and Min~Jae Song.
\newblock Learning single-index models with shallow neural networks.
\newblock In Alice~H. Oh, Alekh Agarwal, Danielle Belgrave, and Kyunghyun Cho, editors, {\em Advances in {Neural} {Information} {Processing} {Systems}}, 2022.

\bibitem[BES{\etalchar{+}}22]{ba_high-dimensional_2022}
Jimmy Ba, Murat~A. Erdogdu, Taiji Suzuki, Zhichao Wang, Denny Wu, and Greg Yang.
\newblock High-dimensional {Asymptotics} of {Feature} {Learning}: {How} {One} {Gradient} {Step} {Improves} the {Representation}.
\newblock {\em Advances in Neural Information Processing Systems}, 35:37932--37946, December 2022.

\bibitem[BKM{\etalchar{+}}19]{barbier_optimal_2019}
Jean Barbier, Florent Krzakala, Nicolas Macris, Léo Miolane, and Lenka Zdeborová.
\newblock Optimal errors and phase transitions in high-dimensional generalized linear models.
\newblock {\em Proceedings of the National Academy of Sciences}, 116(12):5451--5460, March 2019.
\newblock Publisher: Proceedings of the National Academy of Sciences.

\bibitem[CM20]{chen_learning_2020}
Sitan Chen and Raghu Meka.
\newblock Learning {Polynomials} in {Few} {Relevant} {Dimensions}.
\newblock In {\em Proceedings of {Thirty} {Third} {Conference} on {Learning} {Theory}}, pages 1161--1227. PMLR, July 2020.
\newblock ISSN: 2640-3498.

\bibitem[DKPS24]{dandi_how_2024}
Yatin Dandi, Florent Krzakala, Luca Pesce, and Ludovic Stephan.
\newblock How two-layer neural networks learn, one (giant) step at a time.
\newblock {\em J. Mach. Learn. Res.}, 25(1), January 2024.
\newblock Publisher: JMLR.org.

\bibitem[DLS22]{damian_neural_2022}
Alexandru Damian, Jason Lee, and Mahdi Soltanolkotabi.
\newblock Neural {Networks} can {Learn} {Representations} with {Gradient} {Descent}.
\newblock In {\em Proceedings of {Thirty} {Fifth} {Conference} on {Learning} {Theory}}, pages 5413--5452. PMLR, June 2022.
\newblock ISSN: 2640-3498.

\bibitem[DNGL23]{damian_smoothing_2023}
Alex Damian, Eshaan Nichani, Rong Ge, and Jason~D. Lee.
\newblock Smoothing the {Landscape} {Boosts} the {Signal} for {SGD}: {Optimal} {Sample} {Complexity} for {Learning} {Single} {Index} {Models}.
\newblock In {\em Advances in {Neural} {Information} {Processing} {Systems}}, November 2023.

\bibitem[DPVLB24]{damian_computational-statistical_2024}
Alex Damian, Loucas Pillaud-Vivien, Jason~D. Lee, and Joan Bruna.
\newblock Computational-{Statistical} {Gaps} in {Gaussian} {Single}-{Index} {Models}, March 2024.
\newblock arXiv:2403.05529 [cs, stat].

\bibitem[DTA{\etalchar{+}}24]{dandi_benefits_2024}
Yatin Dandi, Emanuele Troiani, Luca Arnaboldi, Luca Pesce, Lenka Zdeborova, and Florent Krzakala.
\newblock The {Benefits} of {Reusing} {Batches} for {Gradient} {Descent} in {Two}-{Layer} {Networks}: {Breaking} the {Curse} of {Information} and {Leap} {Exponents}.
\newblock In {\em Proceedings of the 41st {International} {Conference} on {Machine} {Learning}}, pages 9991--10016. PMLR, July 2024.
\newblock ISSN: 2640-3498.

\bibitem[GLM18]{ge_learning_2018}
Rong Ge, Jason~D. Lee, and Tengyu Ma.
\newblock Learning {One}-hidden-layer {Neural} {Networks} with {Landscape} {Design}.
\newblock In {\em International {Conference} on {Learning} {Representations}}, 2018.

\bibitem[GRWZ21]{ge_understanding_2021}
Rong Ge, Yunwei Ren, Xiang Wang, and Mo~Zhou.
\newblock Understanding {Deflation} {Process} in {Over}-parametrized {Tensor} {Decomposition}, October 2021.
\newblock arXiv:2106.06573 [cs, stat].

\bibitem[Ich93]{ichimura_semiparametric_1993}
Hidehiko Ichimura.
\newblock Semiparametric least squares ({SLS}) and weighted {SLS} estimation of single-index models.
\newblock {\em Journal of Econometrics}, 58(1):71--120, July 1993.

\bibitem[KC22]{kuchibhotla_moving_2022}
Arun~Kumar Kuchibhotla and Abhishek Chakrabortty.
\newblock Moving beyond sub-{Gaussianity} in high-dimensional statistics: applications in covariance estimation and linear regression.
\newblock {\em Information and Inference: A Journal of the IMA}, 11(4):1389--1456, December 2022.

\bibitem[LMZ20]{li_learning_2020}
Yuanzhi Li, Tengyu Ma, and Hongyang~R. Zhang.
\newblock Learning {Over}-{Parametrized} {Two}-{Layer} {Neural} {Networks} beyond {NTK}.
\newblock In Jacob Abernethy and Shivani Agarwal, editors, {\em Proceedings of {Thirty} {Third} {Conference} on {Learning} {Theory}}, volume 125 of {\em Proceedings of {Machine} {Learning} {Research}}, pages 2613--2682. PMLR, July 2020.

\bibitem[LOSW24]{lee_neural_2024}
Jason~D. Lee, Kazusato Oko, Taiji Suzuki, and Denny Wu.
\newblock Neural network learns low-dimensional polynomials with {SGD} near the information-theoretic limit, June 2024.
\newblock arXiv:2406.01581 [cs, stat] version: 1.

\bibitem[LY17]{li_convergence_2017}
Yuanzhi Li and Yang Yuan.
\newblock Convergence {Analysis} of {Two}-layer {Neural} {Networks} with {ReLU} {Activation}.
\newblock In {\em Advances in {Neural} {Information} {Processing} {Systems}}, volume~30. Curran Associates, Inc., 2017.

\bibitem[MZ03]{meir_generalization_2003}
Ron Meir and Tong Zhang.
\newblock Generalization {Error} {Bounds} for {Bayesian} {Mixture} {Algorithms}.
\newblock {\em Journal of Machine Learning Research}, 4(Oct):839--860, 2003.

\bibitem[O'D14]{odonnell_analysis_2014}
Ryan O'Donnell.
\newblock {\em Analysis of {Boolean} {Functions}}.
\newblock Cambridge University Press, 1 edition, June 2014.

\bibitem[OSSW24]{oko_learning_2024}
Kazusato Oko, Yujin Song, Taiji Suzuki, and Denny Wu.
\newblock Learning sum of diverse features: computational hardness and efficient gradient-based training for ridge combinations.
\newblock In {\em Proceedings of {Thirty} {Seventh} {Conference} on {Learning} {Theory}}, pages 4009--4081. PMLR, June 2024.
\newblock ISSN: 2640-3498.

\bibitem[Pin20]{pinelis_concentration_2020}
Iosif Pinelis.
\newblock Concentration and anti-concentration of gap between largest and second largest value in gaussian iid sample.
\newblock MathOverflow, 2020.
\newblock URL:https://mathoverflow.net/q/379688 (version: 2020-12-25).

\bibitem[TDD{\etalchar{+}}24]{troiani_fundamental_2024}
Emanuele Troiani, Yatin Dandi, Leonardo Defilippis, Lenka Zdeborová, Bruno Loureiro, and Florent Krzakala.
\newblock Fundamental computational limits of weak learnability in high-dimensional multi-index models, October 2024.
\newblock arXiv:2405.15480.

\bibitem[Tia17]{tian_analytical_2017}
Yuandong Tian.
\newblock An {Analytical} {Formula} of {Population} {Gradient} for two-layered {ReLU} network and its {Applications} in {Convergence} and {Critical} {Point} {Analysis}.
\newblock In {\em Proceedings of the 34th {International} {Conference} on {Machine} {Learning}}, pages 3404--3413. PMLR, July 2017.
\newblock ISSN: 2640-3498.

\bibitem[Ver18]{vershynin_high-dimensional_2018}
Roman Vershynin.
\newblock {\em High-{Dimensional} {Probability}: {An} {Introduction} with {Applications} in {Data} {Science}}.
\newblock Cambridge University Press, 1 edition, September 2018.

\bibitem[VGNA20]{vladimirova_sub-weibull_2020}
Mariia Vladimirova, Stéphane Girard, Hien Nguyen, and Julyan Arbel.
\newblock Sub-{Weibull} distributions: {Generalizing} sub-{Gaussian} and sub-{Exponential} properties to heavier tailed distributions.
\newblock {\em Stat}, 9(1):e318, January 2020.

\bibitem[Wai19]{wainwright_high-dimensional_2019}
Martin~J. Wainwright.
\newblock {\em High-{Dimensional} {Statistics}: {A} {Non}-{Asymptotic} {Viewpoint}}.
\newblock Cambridge University Press, 1 edition, February 2019.

\bibitem[ZGJ21]{zhou_local_2021}
Mo~Zhou, Rong Ge, and Chi Jin.
\newblock A {Local} {Convergence} {Theory} for {Mildly} {Over}-{Parameterized} {Two}-{Layer} {Neural} {Network}.
\newblock In {\em Proceedings of {Thirty} {Fourth} {Conference} on {Learning} {Theory}}, pages 4577--4632. PMLR, July 2021.
\newblock ISSN: 2640-3498.

\bibitem[ZSJ{\etalchar{+}}17]{zhong_recovery_2017}
Kai Zhong, Zhao Song, Prateek Jain, Peter~L. Bartlett, and Inderjit~S. Dhillon.
\newblock Recovery {Guarantees} for {One}-hidden-layer {Neural} {Networks}.
\newblock In {\em Proceedings of the 34th {International} {Conference} on {Machine} {Learning}}, pages 4140--4149. PMLR, July 2017.
\newblock ISSN: 2640-3498.

\end{thebibliography}

\newpage
\appendix
{
  \hypersetup{linkcolor=black}
  \renewcommand{\baselinestretch}{0.75}\normalsize
  \tableofcontents
  \renewcommand{\baselinestretch}{1.0}\normalsize
}

\section{Preliminaries}
\label{sec: preliminaries}

\subsection{Population and per-sample gradients}
\label{subsec: population and persample gradients}

In this subsection, we show that the task of learning the multi-index target function $f_*(\x) = \sum_{k=1}^{P} \phi(\v_k^* \cdot \x)$
can be reduced to tensor decomposition and prove the tail bounds in Lemma~\ref{lemma: persample gradient}.

For the first goal, we will need the following classical result on Hermite polynomials 
(cf.~Chapter~11.2 of \cite{odonnell_analysis_2014}) and correlated Gaussian variables. 

\begin{lemma}[Proposition~11.31 of \cite{odonnell_analysis_2014}]
  \label{lemma: correlated gaussians and hermite polynomials}
  For $k \in \mbb{N}_{\ge 0}$ denote the normalized Hermite polynomials. 
  Let $\rho \in [-1, 1]$ and $z, z'$ be $\rho$-correlated standard Gaussian variables. Then, we have 
  \[
    \E_{z,z'} \left[ h_k(z) h_j(z') \right]
    = \indi\{ k = j \} \rho^k. 
  \]
\end{lemma}

\begin{lemma}
  \label{lemma: population gradient}
  Under the setting described in Section~\ref{sec: setup}, we have 
  \[
    \E_{\x}\left[ f_*(\x) \nabla_{\v} \phi(\v\cdot\x) \right]
    = \sum_{k=1}^{P} \sum_{l=1}^{\infty} l \hat\phi_l^2 \inprod{\v_k^*}{\v}^{l-1} \v_k^*.
  \]
\end{lemma}
\begin{proof}
  Let $\phi = \sum_{k=0}^\infty \hat\phi_k h_k$ be the Hermite expansion of $\phi$ where the convergence is in $L^2$ sense. 
  For any $\rho \in [-1, 1]$ and $\rho$-correlated standard Gaussian variables $z, z'$, we have 
  \[ 
    \E_{z,z'} \braces{ \phi(z) \phi(z') }
    = \sum_{k,l=0}^\infty \hat\phi_k \hat\phi_l \E_{z, z'}\braces{ h_k(z) h_l(z') }
    = \sum_{k=0}^\infty \hat\phi_k^2 \rho^k, 
  \] 
  where the first equality comes from the Dominated Convergence Theorem and the second from Lemma~\ref{lemma: correlated 
  gaussians and hermite polynomials}. Therefore, we have 
  \[
    \E_{\x}\left[ f_*(\x) \phi(\v\cdot\x) \right]
    = \sum_{k=1}^{P} \E_{\x}\left[ \phi(\v_k^*\cdot\x) \phi(\v\cdot\x) \right]
    = \sum_{k=1}^{P} \sum_{l=1}^{\infty} \hat\phi_l^2 \inprod{\v_k^*}{\v}^l.
  \]
  Then, we compute 
  \[
    \E_{\x}\left[ f_*(\x) \nabla_{\v} \phi(\v\cdot\x) \right]
    = \sum_{k=1}^{P} \sum_{l=1}^{\infty} \hat\phi_l^2 \nabla_{\v} \inprod{\v_k^*}{\v}^l
    = \sum_{k=1}^{P} \sum_{l=1}^{\infty} l \hat\phi_l^2 \inprod{\v_k^*}{\v}^{l-1} \v_k^*.
  \]
\end{proof}

Now, we consider the per-sample gradient. The goal here is to prove variance and tail bounds for 
$\inprod{ f_*(\x) \nabla_{\v}\phi(\v\cdot\x)  }{\u}$, where $\u \in \S^{d-1}$ is an arbitrary direction that is 
independent of $\x$. 
First, we upper bound $f_*$. To this end, we will use need the following concentration inequality for 
the sum of independent subweibull variables. It is a consequence of Theorem~3.2 of \cite{kuchibhotla_moving_2022}
and the discussion after Definition~2.3 of the same paper. 

\begin{lemma}[\cite{kuchibhotla_moving_2022}]
  \label{lemma: concentration of sum of subweibull}
  Let $\psi_\alpha(x) = \exp(x^\alpha) - 1$ and $\norm{\cdot}_{\psi_\alpha}$ denote the corresponding Orlicz norm. 
  Let $\alpha \le 1$ and $X_1, \dots, X_n$ be i.i.d.~mean zero random variables with variance $\sigma^2$ and 
  $\norm{ X_1 }_{\psi_\alpha} < \infty$. 
  Then, for any $\delta_{\P} \in (0, 1)$, we have 
  \[
    \abs{ \sum_{i=1}^{n} X_i } 
    \lesssim_\alpha 
      \sqrt{n} \sigma \log^{1/2}(1/\delta_{\P})  
      + \sqrt{n} \norm{X_1}_{\psi_\alpha}
        \log^{1/\alpha}(n/\delta_{\P}), 
    \quad\text{with probability at least $1 - \delta_{\P}$.}
  \]
\end{lemma}

\begin{lemma}
  Suppose that Assumption~\ref{assumption: link function} holds and $\{\v_k^*\}_k$ are orthonormal.
  Then, for any $\delta_{\P} \in (0, 1)$, we have 
  \[
    |f_*(\x)|
    \lesssim  \sqrt{P} \log^q(P/\delta_{\P}), 
    \quad\text{with probability at least $1 - \delta_{\P}$.}
  \]
\end{lemma}
\begin{proof}
  Write $Y_k := \phi(\v_k^* \cdot \x)$. By the orthonormality of $\{\v_k^*\}_k$, $\{ Y_k \}_k$ are independent 
  variables. For any $p \ge 1$, we have 
  \begin{align*}
    \norm{ Y_1 }_{L^p}
    \le \left( \E_{z \sim \Gaussian{0}{1}} |\phi|^p(z) \right)^{1/p}
    &\le C \left( \E_{z \sim \Gaussian{0}{1}} \left[ (1 + z^2)^{p q/2} \right] \right)^{1/p} \\
    &\le C p^q \sqrt{ \E_{z \sim \Gaussian{0}{1}} \left[  (1 + z^2)^q \right] }
    \lesssim p^q, 
  \end{align*}
  where the first inequality in the second line comes from the Gaussian hypercontractivity. This implies that 
  $\norm{ Y_1 }_{\psi_{1/q}} \lesssim 1$ and $\Var Y_1 \lesssim 1$. Thus, by Lemma~\ref{lemma: concentration of sum of 
  subweibull}, we have with probability at least $1 - \delta_{\P}$ that 
  \[
    |f_*(\x)|
    \lesssim \sqrt{P} \log^{1/2}(1/\delta_{\P})  
      + \sqrt{P} \log^q(P/\delta_{\P})
    \lesssim  \sqrt{P} \log^q(P/\delta_{\P}).
  \]
\end{proof}

\begin{lemma}
  \label{lemma: variance and tail bounds}
  Suppose that Assumption~\ref{assumption: link function} holds and $\{\v_k^*\}_k$ are orthonormal. Then, we have 
  \begin{gather*}
    \E \inprod{ f_*(\x) \nabla_{\v}\phi(\v \cdot \x) }{\u}^2 
    \lesssim_\phi P, \\
    |\inprod{ f_*(\x) \nabla_{\v}\phi(\v \cdot \x) }{\u}| \lesssim_\phi P^{1/2} \log^{2(1+q)}\log(m/\delta_{\P})
    \quad\text{with probability at least $1 - \delta_{\P}$}, 
  \end{gather*}
  where $q$ is the degree of $\phi$ if it is a polynomial and $Q = 0$ if $\phi$ is Lipschitz.
\end{lemma}
\begin{proof}
  Note that $\inprod{\nabla_{\v} \phi(\v \cdot \x)}{\u} = \phi'(\v\cdot\x) \inprod{\x}{\u}$. 
  First, for the variance, we have 
  \[
    \E( f_*(\x) \phi'(\v\cdot \x) \inprod{\u}{\x} )^2 
    \lesssim \E f_*^6 + \E (\phi'(\v\cdot \x))^6 + \E \inprod{\u}{\x}^6
    \lesssim P, 
  \]
  where the second inequality comes from Assumption~\ref{assumption: link function} and the hypercontractivity of Gaussian.
  This implies $\E \inprod{\nabla f(\x)}{\u}^2 \lesssim P$.
  For the tail bound, first recall from the previous lemma that $|f_*(\x)| \lesssim  \sqrt{P} \log^q(P/\delta_{\P})$ 
  with probability at least $1 - \delta_{\P}$. The proof of it also implies $| \phi'(\v \cdot \x) | 
  \lesssim \log^q(P/\delta_{\P})$ with probability at least $1 - \delta_{\P}$. Finally, since $\inprod{\x}{\u}$ 
  is $1$-subgaussian, we have $\abs{\x \cdot \u} \lesssim \log^{1/2}(1/\delta_{\P})$ 
  with probability at least $1 - \delta_{\P}$. Combine these bounds, take the union bound over $m$ learner neurons, 
  and we complete the proof. 
\end{proof}

\subsection{Typical structure at initialization}

In this subsection, we use the results in Section~\ref{subsec: concentration and anti-concentration of Gaussian} to analyze 
the structure of $\v_1, \dots, \v_m$ at initialization. Recall that we initialize $\v_i$ with $\Unif(\S^{d-1})$ independently. Meanwhile, 
note that for $\v \sim \Unif(\S^{d-1})$, we have $\v \overset{d}{=} \Z / \norm{\Z}$ where 
$\Z \sim \Gaussian{0}{\Id_d}$. 

We start with a lemma on the largest coordinate. 
\begin{lemma}[Largest coordinate]
  \label{lemma: uniform on sphere: upper bound on max vi}
  Let $\v \sim \Unif(\S^{d-1})$. For any $K \ge 1$, we have 
  \[
    \max_{i \in [d]} |v_i| 
    \le \frac{ 4 \sqrt{2 K \log d}}{\sqrt{d}}
    \quad\text{with probability at least }
    1 - \frac{4}{d^K}. 
  \]
  As a corollary, for any $\delta_{\P} \in (0, 1)$, at initialization, we have 
  \[
    \max_{i \in [m]} \norm{\v_i}_\infty 
    \le \frac{ 4 \sqrt{2 \log( 4 m / \delta_{\P} ) }}{\sqrt{d}}
    \quad\text{with probability at least $1 - \delta_{\P}$}. 
  \]
\end{lemma}
\begin{proof}
  Let $\Z \sim \Gaussian{0}{\Id_d}$. Recall that $\Z / \norm{\Z}$ follows the uniform distribution over the sphere.
  By Lemma~\ref{lemma: fluctuation of the Gaussian norm}, we have $\norm{\Z} \ge \sqrt{d}/2$ with 
  probability at least $1 - 2 \exp( -d/18 )$. Then, by Lemma~\ref{lemma: gaussian, upper tail for the max}, 
  with probability at least $1 - 2 e^{-d/18} - 2 e^{-s^2/2}$, we have 
  \[ 
    \frac{\max_{i \in [d]} |Z_i|}{\norm{\Z}} 
    \le \frac{\sqrt{2 \log d} + s}{\sqrt{d}/2}
    = \frac{2 \sqrt{2 \log d } }{\sqrt{d}}
      + \frac{ 2 s}{\sqrt{d}}. 
  \] 
  Let $K \ge 1$ be arbitrary. Choose $s = \sqrt{2 K \log d}$ and the above becomes 
  \[
    \frac{\max_{i \in [d]} |Z_i|}{\norm{\Z}} 
    \le \frac{ 4 \sqrt{2 K \log d}}{\sqrt{d}}
    \quad\text{with probability at least }
    1 - \frac{4}{d^K}. 
  \]
  For the corollary, use union bound and choose $K = \log( 4 m / \delta_{\P} ) / \log d$, we have 
  \[
    \max_{i \in [m]} \norm{\v_i}_\infty
    \le \frac{ 4 \sqrt{2 \log( 4 m / \delta_{\P} ) }}{\sqrt{d}}
    \quad\text{with probability at least }
    1 - \frac{4 m}{d^K}
    = 1 - \delta_{\P}. 
  \]

\end{proof}

Suppose that we only have higher-order terms. Then, for a neuron $\v \in \S^{d-1}$ to converge to a ground-truth 
direction $\e_k$ in a reasonable amount of time, we need $v_k^2$ to be the largest among all $v_i^2$ and there is gap between
it and the second largest $v_i^2$. The following lemma ensures that when $m$ is large, for every ground-truth direction
$\{ \e_k \}_{k \in [P]}$, there will be at least one neuron satisfying the above property. Note that in our case, we 
only need to ensure $v_k^2$ is the largest among all $\{v_i^2\}_{i \in [P]}$ instead of $\{v_i^2\}_{i \in [d]}$, as 
the second-order term will help us identify the correct subspace.

\begin{lemma}[Existence of good neurons]
  \label{lemma: large m => existence of good neurons}
  Let $\delta_{\P} \in (e^{-\log^C d}, 1)$ be given. Suppose that $m \ge 2 P \log(P / \delta_{\P}) 
  = \tilde{\Theta}(P)$.
  Then, at initialization, with probability at least $1 - \delta_{\P}$, we have 
  \[
    \forall p \in [P]\, \exists i \in [m] \quad\text{such that}\quad 
    \frac{v_{i, p}^2 }{ \max_{q \in [P]\setminus \{p\}} v_{i, q}^2 } 
    \ge 1 + \frac{1}{200 \log(48 P)}
    = 1 + \tilde{\Theta}(1).
  \]
\end{lemma}
\begin{proof}
  Let $\delta_0$ be a parameter to chosen later and let $\delta_{\P, 0}$ be the probability that 
  $\max_{q \in [P]} v_q^2$ is smaller than $1 + \delta_0$ times the second largest $v_q^2$. 
  For each $p \in [P]$, let $B_p$ be the event $\braces{ \forall k \in [m], v_{k,p}^2 \le (1 + \delta_0) \max_{q \in [P]\setminus\{p\}} v_{k, q}^2 }$.
  To bound $\P[B_p]$, we write 
  \begin{align*}
    \P[B_p]
    &= \left(
        \P_{\v\sim\Unif(\S^{d-1})}\left[ v_{p}^2 \le (1 + c) \max_{q \in [P]\setminus\{p\}} v_{q}^2 \right]
      \right)^m \\
    &= \left(
        \P\left[ v_p^2 \ne \max_{q \in [P]} v_q^2  \right]
        + 
        \P\left[ 
          v_p^2 \le (1 + c) \max_{q \in [P]\setminus\{p\}} v_{q}^2 
          \;\bigg|\; v_p^2 = \max_{q \in [P]} v_q^2
        \right]
        \P\left[ v_p^2 = \max_{q \in [P]} v_q^2 \right]
      \right)^m \\
    &= \left(
        1 - \frac{1}{P} + \frac{\delta_{\P, 0}}{P}
      \right)^m. 
  \end{align*}
  By Corollary~\ref{cor: ratio between the largest and second largest vp2}, if we choose $\delta_{\P, 0} = 1/2$, then we 
  can choose 
  \[
    \delta_0 = \frac{1}{200 \log(48 P)}.
  \]
  With the above choices of parameters, we have 
  \[
    \P\left[\bigcup_{p \in [P]} B_p \right]
    \le P \left( 1 - \frac{1}{2 P} \right)^m
    \le P \exp\left( -\frac{m}{2 P} \right). 
  \]
  For the last term to be bounded by $\delta_{\P}$, it suffices to choose $m \ge 2 P \log(P / \delta_{\P})$.
\end{proof}

\begin{lemma}[Typical structure at initialization]
  \label{lemma: Typical structure at initialization}
  Let $\delta_{\P} \in (e^{-\log^C d}, 1)$ be given and $c_g > 0$ be a small constant. 
  Suppose that $\{\v_k\}_{k=1}^m \sim \Unif(\S^{d-1})$ independently with 
  \[
    m = \Theta\left( P \log(P/\delta_{\P}) \right).
  \]
  Then, with probability at least $1 - 3 \delta_{\P}$, we have 
  \begin{gather}
    \forall p \in [P]\, \exists i \in [m] \quad\text{such that}\quad 
    \frac{v_{i, p}}{ \max_{q \in [P]\setminus \{p\}} |v_{i, q}| } 
    \ge 1 + \frac{\Theta(1)}{\log P}, \label{eq: good neuron} \\
    \forall i \in [m],\,\quad  
    \norm{\v_i}_\infty 
    \le \frac{ 
        20 \sqrt{ \log( P / \delta_{\P} ) }
      }{\sqrt{d}}, \nonumber \\
    \forall i \in [m],\,\quad 
    \frac{\sqrt{P}}{3 \sqrt{d}} \le \frac{\norm{\v_{\le P}}}{\norm{\v}}  \le \frac{3 \sqrt{P}}{\sqrt{d}}. \nonumber
  \end{gather}
\end{lemma}
\begin{proof}
  The first two bounds comes directly from Lemma~\ref{lemma: uniform on sphere: upper bound on max vi} and 
  Lemma~\ref{lemma: large m => existence of good neurons}
  and the fact that we use symmetric initialization. By Lemma~\ref{lemma: fluctuation of the Gaussian norm},
  we have 
  \begin{align*}
    \P\left( \left| \norm{\Z} - \E \norm{\Z} \right| \ge \sqrt{d}/2 \right) 
    \le 2 e^{-d/8}, \\
    \P\left( \left| \norm{\Z_{\le P}} - \E \norm{\Z_{\le P}} \right| \ge \sqrt{P}/2 \right) 
    \le 2 e^{-P/8}. 
  \end{align*}
  As a result, for any $\v \sim \Unif(\S^{d-1})$, we have with probability at least $1 - 4 e^{-P/8}$ that 
  \[
    \frac{\norm{\v_{\le P}}}{\norm{\v}} 
    \overset{d}{=} \frac{\norm{\Z_{\le P}}}{\norm{\Z}}
    = \frac{\E \norm{\Z_{\le P}} \pm \sqrt{P}/2}{\E \norm{\Z} \pm \sqrt{d}/2}
    = [1/3, 3] \times \sqrt{\frac{P}{d}}. 
  \]
  Since we assume $P \ge \log^{C'} d$ for a large $C'$, we have $4 e^{-P/8} \le \delta_{\P} / m$. This gives the third
  bound.  
\end{proof}

\subsection{Concentration and anti-concentration of Gaussian}
\label{subsec: concentration and anti-concentration of Gaussian}

\begin{lemma}
  \label{lemma: ratio between the largest and second largest abs Gaussian}
  Let $Z_1, \dots, Z_d$ be independent $\Gaussian{0}{1}$ variables. Let $Y_1, Y_2$ be the largest and second largest
  of $|Z_1|, \dots, |Z_d|$. For any $\delta_{\P} \in (0, 1)$, we have 
  \[
    \frac{Y_1}{Y_2} \ge 1 + \frac{\delta_{\P}}{12 \log\left( 12 d / \delta_{\P} \right) }
    \quad\text{with probability at least $1 - \delta_{\P}$.}
  \]
\end{lemma}
\begin{proof}
  \def\currentprefix{proof: ratio between largest and the second largest. daokkmd}
  The following proof is adapted from this MathOverflow answer \cite{pinelis_concentration_2020}.
  Let $f$ and $F$ denote the PDF and CDF of $|Z|$ with $Z \sim \Gaussian{0}{1}$, respectively. We have 
  $f(y) = 2 \phi(y) \indi\{ y \ge 0\}$ and $F(y) = \erf(y/\sqrt{2}) \indi\{ y \ge 0 \} $, where $\phi$ is the PDF 
  of $\Gaussian{0}{1}$ and $\erf$ is the error function. 
  We will use the following formula for the joint PDF of two order statistics: 
  \[
    f_{Y_1, Y_2}(y_1, y_2) 
    = d (d - 1)  F^{d-2}(y_2) f(y_2) f(y_1) \indi\{ 0 < y_2 < y_1 \}. 
  \]
  Consider small $s > 0$. We compute 
  \begin{align*}
    \P\left( \frac{Y_1}{Y_2} \ge 1 + s \right)
    &= \int_0^\infty \int_0^\infty \indi\{ y_1 \ge (1 + s) y_2 \} f_{Y_1, Y_2}(y_1, y_2) \,\rd y_2 \rd y_1 \\
    &= \int_0^\infty \int_0^\infty  f_{Y_1, Y_2}((1 + s)y_2 + r, y_2) \,\rd y_2 \rd r \\
    &= d (d - 1)  
      \int_0^\infty F^{d-2}(y_2) f(y_2) \left( \int_0^\infty f((1 + s)y_2 + r) \,\rd r \right) \,\rd y_2  \\
    &= d (d - 1)  
      \int_0^\infty F^{d-2}(y_2) f(y_2) \left( 1 - F((1 + s) y_2) \right) \,\rd y_2. 
  \end{align*}
  Let $G = F\inv$. With the change-of-variables $u = F(y_2)$, $y_2 = G(u)$, we can rewrite the above as 
  \begin{align*}
    \P\left( \frac{Y_1}{Y_2} \ge 1 + s \right)
    &= d (d - 1) \int_0^1 u^{d-2} \left( 1 - F((1 + s) G(u)) \right) f(G(u)) \,\rd G(u) \\
    &= d (d - 1) \int_0^1 u^{d-2} \left( 1 - F((1 + s) G(u)) \right) \cancel{f(G(u)) \frac{1}{F'(G(u))}} \,\rd u \\
    &= d (d - 1) \int_0^1 u^{d-2} \left( 1 - F((1 + s) G(u)) \right) \,\rd u. 
  \end{align*}

  Now, we analyze the last integral. We will use the following expansion of the (complementary) error function:
  \[
    1 - F(y)
    = 1 - \erf\left(y/\sqrt{2}\right)
    = \frac{e^{-y^2/2}}{y \sqrt{\pi/2}} 
      \left(
        1 - \frac{1}{\sqrt{\pi}} \int_y^\infty r^2 e^{-r^2}\,\rd r
      \right). 
  \]
  For notational simplicity, put $w = G(u)$. Then, we have 
  \begin{align*}
    1 - F( (1 + s) w )
    &= \frac{e^{-(1 + s)^2 w^2/2}}{(1 + s) w \sqrt{\pi/2}} 
      \left( 1 - \frac{1}{\sqrt{\pi}} \int_{(1+s)w}^\infty r^2 e^{-r^2}\,\rd r \right) \\
    &= \frac{\exp\left( -(s + s^2/2) w^2 \right)}{1 + s}
      \frac{e^{-w^2/2}}{w \sqrt{\pi/2}} 
      \left( 1 - \frac{1}{\sqrt{\pi}} \int_{(1+s)w}^\infty r^2 e^{-r^2}\,\rd r \right) \\
    &= \frac{\exp\left( -(s + s^2/2) w^2 \right)}{1 + s}
      \left( 1 - F( w ) \right)
      \frac{ 1 - \frac{1}{\sqrt{\pi}} \int_{(1+s)w}^\infty r^2 e^{-r^2}\,\rd r }{
        1 - \frac{1}{\sqrt{\pi}} \int_w^\infty r^2 e^{-r^2}\,\rd r
      }. 
  \end{align*}
  Note that the last factor is at least $1$ and $F(w) = F(G(u)) = F(F\inv(u)) = u$. Therefore, 
  \[
    1 - F( (1 + s) w )
    \ge \frac{\exp\left( -(s + s^2/2) w^2 \right)}{1 + s} (1 - u). 
  \]
  As a result, we have 
  \begin{align*}
    \P\left( \frac{Y_1}{Y_2} \ge 1 + s \right)
    &\ge d (d - 1) \int_0^1 u^{d-2} (1 - u) \frac{\exp\left( -(s + s^2/2) G^2(u) \right)}{1 + s} \,\rd u \\
    &\ge d (d - 1) \int_0^{1-\eps} u^{d-2} (1 - u) \frac{\exp\left( -(s + s^2/2) G^2(u) \right)}{1 + s} \,\rd u, 
  \end{align*}
  where $\eps > 0$ is a parameter to be chosen later. By the next lemma, when $u \le 1 - \eps$, we have 
  $G^2(u) \le 2 \log(2/\eps)$. Therefore, 
  \begin{align*}
    \P\left( \frac{Y_1}{Y_2} \ge 1 + s \right)
    &\ge d (d - 1) \int_0^{1-\eps} u^{d-2} (1 - u) \,\rd u
      \frac{\exp\left( -(2 s + s^2) \log(2/\eps) \right)}{1 + s}  \\
    &\ge d (d - 1) \int_0^{1-\eps} u^{d-2} (1 - u) \,\rd u
      \frac{1}{1 + s} \left( \frac{\eps}{2} \right)^{4s}  \\
    &= (1 - \eps)^d \left( 1 + \frac{d \eps}{1 - \eps}  \right)
      \frac{1}{1 + s} \left( \frac{\eps}{2} \right)^{4s}. 
  \end{align*}
  Let $\delta_{\P} \in (0, 1)$ be our target failure probability. We choose 
  \[
    (1 - \eps)^d \left( 1 + \frac{d \eps}{1 - \eps}  \right) 
    \ge 1 - \frac{\delta_{\P}}{3}
    \quad\Leftarrow\quad 
    \eps 
    = \frac{\delta_{\P}}{6 d}. 
  \]
  With this choice of $\eps$, we compute 
  \[
    \left( \frac{\eps}{2} \right)^{4s} \ge 1 - \frac{\delta_{\P}}{3}
    \quad\Leftarrow\quad 
    \left( \frac{\delta_{\P}}{12 d} \right)^{4s} \ge 1 - \frac{\delta_{\P}}{3}
    \quad\Leftarrow\quad 
    4s \log\left( \frac{\delta_{\P}}{12 d} \right) \ge  - \frac{\delta_{\P}}{3}
    \quad\Leftarrow\quad 
    s \le \frac{\delta_{\P}}{12 \log\left( \frac{12 d}{\delta_{\P}} \right) }. 
  \]
  Note that for $s$ satisfying this condition, we automatically have $1/(1 + s) \ge 1 - \delta_{\P}/3$. 
  Thus, we have 
  \[
    \frac{Y_1}{Y_2} \ge 1 + \frac{\delta_{\P}}{12 \log\left( \frac{8 d}{\delta_{\P}} \right) }
    \quad\text{with probability at least $1 - \delta_{\P}$.}
  \]
\end{proof}

\begin{lemma}
  Let $F(y) = \erf(y/\sqrt{2}) \indi\{ y \ge 0 \}$ be the CDF of $|\Gaussian{0}{1}|$, respectively. 
  Let $G = F\inv$. If $u \le 1 - 1/(d \log d)$, then $G(u) \le \sqrt{2 \log(2/\eps)}$.
\end{lemma}
\begin{proof}
  Note that 
  $G(u) \le M$ iff $u \le F(M)$ iff $1 - u \ge 1 - F(M)$ iff $1 - u \ge \P(|Z| \ge M)$. In other words, 
  our goal here is to find the smallest $M$ such that $\P(|Z| \ge M) \le \eps$. By the standard Gaussian
  concentration, we have $\P(|Z| \ge M) \le 2 \exp( -M^2/2 )$. For the RHS to be upper bounded by 
  $\eps$, it suffices to choose $M \ge \sqrt{2 \log(2/\eps)}$. 
\end{proof}

\begin{lemma}
  \label{lemma: fluctuation of the Gaussian norm}
  Let $\z_1, \dots, \z_m$ be independent $\Gaussian{0}{\Id_d}$ random vectors. Then, for any $\eps > 0$, we have 
  \[
    \P\left( \forall k \in [m],  \left| \frac{\norm{\z_k}}{\E \norm{\z_1}} - 1 \right| \le \eps \right) 
    \ge 1 - 2 m  e^{- \eps^2 d / 3}. 
  \]
\end{lemma}
\begin{proof}
  It is well-known that any $1$-Lipschitz function of $\Gaussian{0}{\Id_d}$ is $1$-subgaussian (see, for example, 
  Theorem~5.2.2 of \cite{vershynin_high-dimensional_2018}). Hence, for any $s > 0$, we have 
  \[
    \P\left( \max_{k \in [m]} \left| \norm{\z_k} - \E \norm{\z_k} \right| \ge s \right) 
    \le \sum_{k=1}^{m} \P\left( \left| \norm{\z_k} - \E \norm{\z_k} \right| \ge s \right) 
    \le 2 m  e^{-s^2/2}. 
  \]
  Set $s = \eps \E \norm{\z_k}$ and the above becomes 
  \[
    \P\left( \forall k \in [m],  \left| \frac{\norm{\z_k}}{\E \norm{\z_1}} - 1 \right| \le \eps \right) 
    \ge 1 - 2 m  e^{- \eps^2 ( \E \norm{\z_1} )^2 /2}. 
  \]
  To complete the proof, it suffices to note that $\norm{\z_1}$ follows the $\chi_d$ distribution, 
  and therefore we have $\E \norm{\z_1} \ge \sqrt{d}(1 - 2/d)$. 
\end{proof}

\begin{corollary}
  \label{cor: ratio between the largest and second largest vp2}
  Let $\v \sim \mrm{Unif}(\S^{d-1})$ and $w_1$ and $w_2$ denote the largest and second largest of $|v_1|, \dots, |v_P|$. 
  Suppose that $\frac{d}{\log d} \gtrsim \frac{ \log^2(P / \delta_{\P})}{\delta_{\P}^2}$ and 
  $\delta_{\P} \in (e^{-\log^C d}, 1)$. Then, we have 
  \[
    \frac{w_1}{w_2}
    \ge 1 + \frac{\delta_{\P}}{100 \log(24 P / \delta_{\P})}
    \quad\text{with probability at least $1 - \delta_{\P}$}.
  \]
\end{corollary}
\begin{proof}
  Let $\z \sim \Gaussian{0}{\Id_d}$ vectors. Note that $\v \overset{d}{=} \z / \norm{\z}$. 
  By Lemma~\ref{lemma: fluctuation of the Gaussian norm}, we have, for any $\delta_{\P} \in (0, 1)$, that 
  \[
    \left| \frac{\norm{\z_k}}{\E \norm{\z_1}} - 1 \right| 
    \le \sqrt{ \frac{3 \log(2/\delta_{\P})}{d} }
    \quad\text{with probability at least $1 - \delta_{\P}$}.
  \]
  Suppose that $|z_{k_1}|$ and $|z_{k_2}|$ are the largest and second largest of 
  $|z_1|, \dots, |z_P|$. By Lemma~\ref{lemma: ratio between the largest and second largest abs Gaussian} (with $d$
  replaced by $P$), we have 
  \[
    \frac{|z_{k_1}|}{|z_{k_2}|}
    \ge 1 + \frac{\delta_{\P}}{12 \log(12 P / \delta_{\P})}
    \quad\text{with probability at least $1 - \delta_{\P}$}.
  \]
  Write 
  \[
    \frac{w_1}{w_2}
    \ge \left( 1 + \frac{\delta_{\P}}{12 \log(12 P / \delta_{\P})} \right) 
      \frac{1 - \sqrt{ \frac{3 \log(2 P/\delta_{\P})}{d} }}{1 + \sqrt{ \frac{3 \log(2d/\delta_{\P})}{d} }} 
    \ge 1 
      + \frac{\delta_{\P}}{12 \log(12 P / \delta_{\P})} 
      - 3 \sqrt{ \frac{3 \log( 2d/\delta_{\P})}{d} }.
  \]
  In order to merge the last term into the second last term, it suffices to require
  \[
    \frac{\delta_{\P}}{12 \log(12 P / \delta_{\P})} 
    \ge 6 \sqrt{ \frac{3 \log(2d/\delta_{\P})}{d} }
    \quad\Leftarrow\quad 
    \frac{d}{\log d} 
    \gtrsim \frac{ \log^2(P / \delta_{\P})}{\delta_{\P}^2}. 
  \]
  Then, with probability at least $1 - 2 \delta_{\P}$, we have 
  \[
    \frac{w_1}{w_2}
    \ge 1 + \frac{\delta_{\P}}{24 \log(12 P / \delta_{\P})} 
  \]
  Replace $\delta_{\P}$ with $\delta_{\P}/2$ and we complete the proof. 
\end{proof}

\begin{lemma}[Upper tail for the maximum]
  \label{lemma: gaussian, upper tail for the max}
  Let $Z_1, \dots, Z_d \sim \Gaussian{0}{1}$ be independent. We have the upper tail 
  \[
    \P\left( \max_{i \in [d]} |Z_i| \ge \sqrt{2 \log d} + s \right)
    \le 2 e^{-s^2/2 }, 
    \quad \forall  s \ge 0.
  \]
\end{lemma}
\begin{proof}
  For notational simplicity, put $Z^* = \max_{i \in [d]} Z_i$. 
  By union bound and the Chernoff bound, we have for each $s, \theta > 0$,
  \[
    \P( Z^* \ge s )
    = \P\left( \bigvee_{i=1}^d Z_i \ge s \right)
    \le d \P(Z_1 \ge s)
    \le d \frac{\E e^{\theta Z_1}}{e^{\theta s}}  
    = d e^{\theta^2 / 2 - \theta s}. 
  \]
  Choose $\theta = s$ to minimize the RHS, and we obtain
  $
    \P( Z^* \ge s )
    \le e^{\log d - s^2 / 2 }. 
  $
  Replace $s$ with $\sqrt{2 \log d + s^2}$ and this becomes 
  \[    
    \P\left( Z^* \ge \sqrt{2 \log d} + s \right)
    \le \P\left( Z^* \ge \sqrt{2 \log d + s^2} \right)
    \le e^{-s^2/2 }. 
  \]
  Use the fact $-\min_{i \in [d]} Z_i \overset{d}{=} \max_{i \in [d]} Z_i$ and we complete the proof. 
\end{proof}

\section{Stage 1: recovery of the subspace and directions}

In this section, we consider the stage where the second layer is fixed to be a small value and the first layer is 
trained using online spherical SGD. Let $\v$ be a first-layer neuron that is good in the sense of~\eqref{eq: good neuron}.
Assume w.l.o.g.~that $v_1$ is the largest. Our goal in this section is to show $\v$ will converge to close to $\e_1$
with probability at least $1 - \delta_{\P}$ at the end of Stage~1.

For notational simplicity, let $l_{\corr}$ and $\Loss_{\corr}$ denote the per-sample and population correlation loss,
respectively. By Lemma~\ref{lemma: persample gradient}, we can write its update rule as 
\[
  \hat{\v}_{t+1}
  = \v_t + \eta \tnabla \Loss_{\corr}  + \eta \Z_{t+1}, \quad 
  \v_{t+1} = \frac{\hat\v_{t+1}}{ \norm{ \hat\v_{t+1} } }, 
\]
where $\Z_{t+1} =  (\Id - \v\v\trans)( \nabla_{\v} l_{\corr}(\x) - \nabla_{\v} \Loss_{\corr} )$ and, by Lemma~\ref{lemma: 
persample gradient},
\begin{align*}
  - \tnabla_{\v} \Loss_{\corr} 
  &= - (\Id - \v\v\trans) \nabla_{\v} \Loss \\
  &= 2 \hat\phi_2^2 \sum_{k=1}^{P} v_k (\Id - \v\v\trans) \e_k 
    + \sum_{l\ge L} \sum_{k=1}^{P} l \hat\phi_l^2 v_k^{l-1} (\Id - \v\v\trans) \e_k.
\end{align*}
In particular, for each $k \in [d]$, we have\footnote{We will often drop the subscript $t$ when it is clear from 
the context.}
\begin{align*}
  \hat{v}_{t+1, k}
  &= v_{t, k}
    + 2 \eta \hat\phi_2^2 \left( \indi\{k \le P\} - \norm{\v_{\le P}}^2 \right) v_k
    + L \eta \hat\phi_L^2 \left( \indi\{k \le P\} v_k^{L-2} - \norm{\v_{\le P}}_L^L \right) v_k \\
    &\qquad
    + \eta \sum_{l>L} l \hat\phi_l^2 \left( \indi\{k \le P\} v_k^{l-2} - \norm{\v_{\le P}}_l^l  \right) v_k 
    + \eta Z_{t+1, k} \\
  &= v_{t, k}
    + \eta \indi\{k \le P\}  \left( 
      2 \hat\phi_2^2 
      + L \hat\phi_L^2 v_k^{L-2} 
      + \sum_{l>L} l \hat\phi_l^2 v_k^{l-2}  
    \right) v_k \\
    &\qquad
    - \eta \left( 
      2 \hat\phi_2^2 \norm{\v_{\le P}}^2 
      + L \hat\phi_L^2 \norm{\v_{\le P}}_L^L 
      + \sum_{l>L} l \hat\phi_l^2 \norm{\v_{\le P}}_l^l  
    \right) v_k 
    + \eta Z_{t+1, k}.
\end{align*}
For notation simplicity, we define 
\begin{equation}
  \label{eq: definition of rho}
  \rho 
  := 2 \hat\phi_2^2 \norm{\v_{\le P}}^2 
    + L  \hat\phi_L^2 \norm{\v_{\le P}}_L^L 
    + \sum_{l>L} l \hat\phi_l^2 \norm{\v_{\le P}}_l^l.   
\end{equation}
Note that $\rho$ is independent of the coordinate $k$, and we can write 
\begin{equation}
  \label{eq: hat vt+1 k = ...}
  \begin{aligned}
    \hat{v}_{t+1, k}
    &= v_{t, k}
      + \eta \indi\{k \le P\}  \left( 
        2 \hat\phi_2^2 
        + L \hat\phi_L^2 v_k^{L-2} 
        + \sum_{l>L} l \hat\phi_l^2 v_k^{l-2}  
      \right) v_k 
      - \eta \rho v_k  
      + \eta Z_{t+1, k}.
  \end{aligned}
\end{equation}
For the martingale difference term $Z$, note that by Lemma~\ref{lemma: persample gradient}, for any $\u \in \S^{d-1}$, 
$\inprod{\Z_{t+1}}{\u}$ is a $(M_Z^2, \theta)$-subweibull variable with $M_Z = P^{1/2}$ and $1/\theta = 2(1 + Q)$. 
In particular, this implies
\begin{equation}
  \left|\inprod{\Z_{t+1}}{\u}\right| 
  \lesssim_\phi M_Z \log^{2(1+Q)}\log(d/\delta_{\P})
  =: \hat{M}_Z, 
  \quad\text{with probability at least $1 - \delta_{\P}/d^C$}C \label{eq: Z high probability bound}
\end{equation}
where $C > 0$ is any fixed constant.

In addition, we have the following lemma on the dynamics of $v_k^2$. The proof is routine calculation and is deferred
to the end of this section (cf.~Section~\ref{subsec: stage 1: deferred proofs}). 

\begin{lemma}[Dynamics of $v_k^2$]
  \label{lemma: dynamics of vk2}
  For any first-layer neuron $\v$ and $k \in [d]$, we have 
  \[
    \hat{v}_{t+1, k}^2
    = v_{t, k}^2 
      + 2 \eta \gamma_{t, k} v_{t, k}^2 
      + 2 \eta v_{t, k} Z_{t+1, k}
      + \xi_{t+1, k},
  \]
  where $\gamma_{k, t} 
  := \indi\{k \le P\}  \left( 2 \hat\phi_2^2 + L \hat\phi_L^2 v_k^{L-2} + \sum_{l>L} l \hat\phi_l^2 v_k^{l-2} \right) 
  - \rho$ is a $\cF_t$-measurable random variable with $|\gamma_{t, k}| \le 2 C_\phi^2$ 
  and $(\xi_{t+1})_{t \in [T]}$ is (uniformly) bounded by $O_\phi( \eta^2 \hat{M}_Z^2 )$ with probability at least 
  $1 - \delta_{\P}$. 
\end{lemma}

To proceed, we split Stage 1 into two substages. In Stage 1.1, we rely on the second-order terms to learn the relevant 
subspace. We will also show that the gap between largest and second-largest coordinates, which can be guaranteed
with certain probability at initialization, is preserved throughout Stage 1.1. These give Stage 1.2 a nice starting 
point. Then, we show that in Stage~1.2, online spherical SGD can recover the directions using the $L$-th order terms.

\subsection{Stage 1.1: recovery of the subspace and preservation of the gap}

In this subsection, first we show that the ratio $\norm{\v_{\le P}}^2 / \norm{\v_{> P}}^2$ will grow from $\Omega(P/d)$ to 
$\Theta(1)$ within $\tilde{O}(dP)$ iterations. We will rely on the second-order terms and bound 
the influence of higher-order terms. This leads to the desired complexity. The next goal to show the initial randomness 
can be preserved. In our case, we only need the gap between the largest and the second-largest coordinate to be 
preserved, which will ensure that the neurons will not collapse to one single direction. Formally, we have the following 
lemma.

\begin{lemma}[Stage 1.1]
  \label{lemma: main lemma of stage 1.1}
  Let $\v \in \S^{d-1}$ be an arbitrary first-layer neuron satisfying $\norm{ \v }_\infty \le \log^2 d / d$, 
  $\norm{\v_{\le P}}^2 / \norm{\v_{> P}}^2 \gtrsim P / d$, and $v_p^2 = (1 + \delta_0) \argmax_{q \in [P] \setminus \{p\}} \v_q^2$
  at initialization. Let $\delta_{\P}$ be given. Suppose that 
  \[
    P \gg_\phi \log^2 d 
    \quad\text{and}\quad 
    \eta 
    \lesssim_\phi \frac{\delta_0^2}{d P \log d} \left(
        \frac{P}{ \hat{M}_Z^2  } 
        \wedge 
        \frac{1}{ M_Z^2 \log^{\theta+1}(d / \delta_{\P}) }
      \right)
    = \Theta_\phi\left( \frac{\delta_c^2}{d P} \right).
  \]
  Then, with probability at least $1 - O(\delta_{\P})$, we have 
  \[ 
    \frac{ \norm{\v_{\le P}}^2 }{\norm{\v_{< P}}^2} \ge 1 
    \quad\text{within $T = \frac{1 + o(1)}{ 4 \hat{\phi}_2^2 \eta } \log\left(\frac{d}{P}\right) = \tilde{\Theta}( d P) $ iterations}. 
  \] 
  In addition, at the end of Stage~1.1, we have $v_p^2 = (1 + \delta_0/2) \argmax_{q \in [P] \setminus \{p\}} \v_q^2$.
\end{lemma}
\begin{proof}
  It suffices to combine Lemma~\ref{lemma: learning the subspace} and Lemma~\ref{lemma: preservation of the gap}.
\end{proof}

To prove this lemma, we will use stochastic induction (cf.~Section~\ref{sec: stochastic induction}), in particular,
Lemma~\ref{lemma: stochastic discrete gronwall} and Lemma~\ref{lemma: stochastic gronwall when alpha = 0}. 
For example, to analyze the dynamics of
$\norm{\v_{\le P}}^2 / \norm{\v_{> P}}^2$, it suffices to write down the update rule of 
$\norm{\v_{\le P}}^2 / \norm{\v_{> P}}^2$ and decompose it into a signal growth term, a higher-order error term, and 
a martingale difference term as in Lemma~\ref{lemma: stochastic discrete gronwall}. Then, we bound the higher-order 
error terms, and estimate the covariance of the martingale difference terms, assuming the induction hypotheses.

The induction hypotheses we will maintain in this substage are the following:
\begin{equation}
  \label{eq: stage 1.1 induction hypotheses}
  \frac{ \norm{\v_{t, \le P}}^2 }{ \norm{\v_{t, > P}}^2 }
  = \Theta(1) (1 + 4 \hat\phi_2^2 \eta)^t \frac{ \norm{\v_{0, \le P}}^2 }{ \norm{\v_{0, > P}}^2 }, \quad
  v_p^2 \le \frac{\log^2 d}{P}. 
\end{equation}
They are established in Lemma~\ref{lemma: learning the subspace} and Lemma~\ref{lemma: upper bound on vp2}. 

\subsubsection{Learning the subspace}

Now, we derive formulas for the dynamics of the ratio $\norm{\v_{\le P}}^2 / \norm{\v_{> P}}^2$. As we have mentioned 
earlier, the goal here is separate the signal terms, martingale difference terms, and higher-order error terms. 

\begin{lemma}[Dynamics of the norm ratio]
  \label{lemma: dynamics of the norm ratio}
  Assume the induction hypotheses~\eqref{eq: stage 1.1 induction hypotheses} at time $t \le T$. 
  Suppose that $\eta  \le \left( d \hat{M}_Z^2 \right)\inv$. 
  Let $\v$ be an arbitrary first-layer neuron. Then, at time $t$, we have 
  \[
    \frac{ \norm{\v_{t+1, \le P}}^2  }{ \norm{\v_{t+1, > P}}^2  }
    = \frac{ \norm{\v_{\le P}}^2 }{ \norm{\v_{>P}}^2 }
      \left( 1 + 4 \eta \hat\phi_2^2 + 2 \eta \eps_v \right) 
      + H_{t+1}
      + \xi_{t+1},
  \]
  where $\eps_v := \sum_{l\ge L} l \hat\phi_l^2 \norm{\v_{\le P}}_l^l / \norm{\v_{\le P}}^2$, where 
  $(H_{t+1})_t$ is a martingale difference sequence that is conditionally
  $\left( O_\phi( (1 + 4 \hat\phi_2^2 \eta)^t \frac{P}{d} ), \theta \right)$-subweibull, and 
  $(\xi_t)_t$ is an adapted process with 
  $|\xi_{t+1}| \lesssim_\phi (1 + 4 \hat\phi_2^2 \eta)^t \eta^2 \hat{M}_Z^2  P$ for all $t \in [T]$
  with probability at least $1 - \delta_{\P}$.
\end{lemma}
\begin{proof}
  First, recall from Lemma~\ref{lemma: dynamics of vk2} that 
  \[
    \hat{v}_{t+1, k}^2
    = v_{t, k}^2 
      + 2 \eta \gamma_{t, k} v_{t, k}^2 
      + 2 \eta v_{t, k} Z_{t+1, k}
      + \eta^2 \gamma_{t, k}^2 v_{t, k}^2
      + \eta^2 Z_{t+1, k}^2 
      + 2 \eta^2 \gamma_{t, k}^2 v_{t, k} Z_{t+1, k},
  \]
  where $\gamma_{k, t} 
  := \indi\{k \le P\} \left( 2 \hat\phi_2^2 + L \hat\phi_L^2 v_k^{L-2} + \sum_{l>L} l \hat\phi_l^2 v_k^{l-2} \right) - \rho$ 
  is a $\cF_t$-measurable random variable with $|\gamma_{t, k}| \le 2 C_\phi^2$.
  First, for $\norm{\hat{\v}_{\le P}}^2$, we have 
  \begin{align*}
    \norm{\hat{\v}_{t+1, \le P}}^2
    &= \left( 1 + 4 \eta \hat\phi_2^2 - 2 \eta \rho \right) \norm{\v_{\le P}}^2
      + 2 \eta \sum_{l\ge L} l \hat\phi_l^2 \norm{\v_{\le P}}^2 
      + 2 \eta \inprod{\v_{\le P}}{\Z_{\le P}} 
      \\
      &\qquad
      \underbrace{
        \pm 8 C_\phi^4 \eta^2 \norm{\v_{\le P}}^2
        \pm 2 \eta^2 \norm{\Z_{\le P}}^2 
      }_{=:\; \xi_{\le P, t+1}}.
  \end{align*}
  Similarly, for $\norm{\hat{\v}_{> P}}$, we have 
  \begin{align*}
    \norm{\hat{\v}_{t+1, > P}}^2
    &= \norm{\v_{> P}}^2
      - 2 \eta \rho \norm{\v_{> P}}^2
      + 2 \eta \inprod{\v_{>P}}{\Z_{>P}} 
      \underbrace{
        \pm 8 C_\phi^4 \eta^2 \norm{\v_{> P}}^2
        \pm 2 \eta^2 \norm{\Z_{>P}}^2 
      }_{=:\;\xi_{>P, t+1}}. 
  \end{align*}
  For notational simplicity, we also write $\eps_v := \sum_{l\ge L} l \hat\phi_l^2 \norm{\v_{\le P}}_l^l / \norm{\v_{\le P}}^2$. 
  Note that by Assumption~\ref{assumption: link function} and the fact that 
  $\norm{\v_{\le P}}_l^l / \norm{\v_{\le P}}^2 \le 1$, $\eps_v \le C_\phi^2$. 
  Since $\norm{\v_{\le P}} / \norm{\v_{> P}} = \norm{\hat\v_{\le P}} / \norm{\hat\v_{> P}}$, we have 
  \begin{align*}
    \frac{ \norm{\v_{t+1, \le P}}^2  }{ \norm{\v_{t+1, > P}}^2  }
    &= \frac{
        \left( 1 + 4 \eta \hat\phi_2^2  - 2 \eta \rho + 2 \eta  \eps_v \right) \norm{\v_{\le P}}^2
      }{
        \left( 1  - 2 \eta \rho \right) \norm{\v_{>P}}^2
      }
      \left(
        1 
        - \frac{ 2 \eta \inprod{\v_{>P}}{\Z_{>P}} }{ \norm{\hat\v_{t+1, > P}}^2 }
        - \frac{ \xi_{> P} }{ \norm{\hat\v_{t+1, > P}}^2 }
      \right)
      \\
      &\qquad
      + \frac{  2 \eta \inprod{\v_{\le P}}{\Z_{\le P}} }{ \norm{\hat\v_{t+1, > P}}^2 } 
      + \frac{ \xi_{\le P} }{ \norm{\hat\v_{t+1, > P}}^2 }  \\
    &= 
      \frac{
        \left( 1 + 4 \eta \hat\phi_2^2  - 2 \eta \rho + 2 \eta \eps_v \right) \norm{\v_{\le P}}^2
      }{
        \left( 1  - 2 \eta \rho \right) \norm{\v_{>P}}^2
      } \\
      &\qquad
      - \frac{
        \left( 1 + 4 \eta \hat\phi_2^2  - 2 \eta \rho + 2 \eta \eps_v \right) \norm{\v_{\le P}}^2
      }{
        \left( 1  - 2 \eta \rho \right) \norm{\v_{>P}}^2
      } \frac{ 2 \eta \inprod{\v_{>P}}{\Z_{>P}} }{ \norm{\hat\v_{t+1, > P}}^2 } 
      + \frac{  2 \eta \inprod{\v_{\le P}}{\Z_{\le P}} }{ \norm{\hat\v_{t+1, > P}}^2 }  \\
      &\qquad
      - \frac{
        \left( 1 + 4 \eta \hat\phi_2^2  - 2 \eta \rho + 2 \eta \eps_v \right) \norm{\v_{\le P}}^2
      }{
        \left( 1  - 2 \eta \rho \right) \norm{\v_{>P}}^2
      } \frac{ \xi_{> P} }{ \norm{\hat\v_{t+1, > P}}^2 }
      + \frac{ \xi_{\le P} }{ \norm{\hat\v_{t+1, > P}}^2 }  \\
    =:\;& \Term_1\left( \frac{ \norm{\v_{t+1, \le P}}^2  }{ \norm{\v_{t+1, > P}}^2  } \right)
      + \Term_2\left( \frac{ \norm{\v_{t+1, \le P}}^2  }{ \norm{\v_{t+1, > P}}^2  } \right)
      + \Term_3\left( \frac{ \norm{\v_{t+1, \le P}}^2  }{ \norm{\v_{t+1, > P}}^2  } \right), 
  \end{align*}
  where each $\Term_i$ represents one line. Note that, up to some higher order terms, $\Term_1$ contains the 
  signal terms and $\Term_2$ contains the martingale difference terms. Now, our goal is to factor out those 
  higher order terms. 

  For $\Term_1$, recall that $|\rho| \le C_\phi^2$ and use the fact that  
  \begin{equation}
    \label{eq: 1/(1 + z) approx 1 - z}
    \frac{1}{1 + z} = 1 - z \pm 2 z^2, \quad \forall |z| \le 1/2, 
  \end{equation}
  to obtain
  \begin{align*}
    \Term_1\left( \frac{ \norm{\v_{t+1, \le P}}^2  }{ \norm{\v_{t+1, > P}}^2  } \right)
    &= \frac{ \norm{\v_{\le P}}^2 }{ \norm{\v_{>P}}^2 }
      \left( 1 + 4 \eta \hat\phi_2^2  - 2 \eta \rho + 2 \eta \eps_v \right) 
      \left( 1 + 2 \eta \rho \pm 4 \eta^2 \rho \right)  \\
    &= \frac{ \norm{\v_{\le P}}^2 }{ \norm{\v_{>P}}^2 }
      \left( 1 + 4 \eta \hat\phi_2^2 + 2 \eta \eps_v \pm 30 C_\phi^4 \eta^2 \right). 
  \end{align*}
  
  Now, we consider 
  \[
    \Term_2 
    = 
    - \frac{ \norm{\v_{\le P}}^2 }{\norm{\v_{>P}}^2 }
    \frac{ 1 + 4 \eta \hat\phi_2^2  - 2 \eta \rho + 2 \eta \eps_v }{ 1  - 2 \eta \rho  } 
    \frac{ 2 \eta \inprod{\v_{>P}}{\Z_{>P}} }{ \norm{\hat\v_{t+1, > P}}^2 } 
    + \frac{  2 \eta \inprod{\v_{\le P}}{\Z_{\le P}} }{ \norm{\hat\v_{t+1, > P}}^2 }.
  \]
  First, we estimate the $1 / \norm{\hat\v_{t+1, > P}}^2$. We write 
  \begin{align*}
    \frac{1}{\norm{\hat\v_{t+1, > P}}^2}
    &= \frac{1}{
        (1 - 2 \eta \rho) \norm{\v_{> P}}^2
        + 2 \eta \inprod{\v_{>P}}{\Z_{>P}} 
        + \xi_{>P, t+1}
      } \\
    &= \frac{1}{ (1 - 2 \eta \rho) \norm{\v_{> P}}^2 }
      \left(
        1
        - \frac{ 
          2 \eta \inprod{\v_{>P}}{\Z_{>P}} 
          + \xi_{>P, t+1} 
        }{
          \norm{\hat\v_{t+1, >P}}^2
        }
      \right).
  \end{align*}
  By \eqref{eq: Z high probability bound}, we have with probability at least $1 - \delta_{\P}$ that 
  \[
    \left| \overline{\v_{>P}} \cdot \Z_{>P} \right|
    \wedge \left| \overline{\v_{\le P}} \cdot \Z_{\le P} \right|
    \wedge \max_{k \in [d]} |Z_k|
    \lesssim_\phi \hat{M}_Z.
  \]
  Note that the above conditions also imply 
  \begin{align*}
    | \xi_{\le P} |
    &\le 8 C_\phi^4 \eta^2 \norm{\v_{\le P}}^2
      + 2 \eta^2 P \hat{M}_Z^2
    \lesssim_\phi \eta^2 P \hat{M}_Z^2 , \\
    | \xi_{> P} |
    &\le 
      8 C_\phi^4 \eta^2 \norm{\v_{> P}}^2
      + 2 \eta^2 d \hat{M}_Z^2
    \lesssim_\phi \eta^2 d \hat{M}_Z^2.
  \end{align*}
  By our definition of Stage~1.1, we have $\norm{\hat\v_{t+1, > P}}^2 \ge 1/2$. 
  Then, we have 
  \begin{align*}
    \frac{1}{\norm{\hat\v_{t+1, > P}}^2}
    &= \frac{1}{ (1 - 2 \eta \rho) \norm{\v_{> P}}^2 }
      \left(
        1
        \pm 4 \eta \hat{M}_Z
        \pm 32 C_\phi^4 \eta^2 d \hat{M}_Z^2
      \right) \\
    &= \frac{1}{ (1 - 2 \eta \rho) \norm{\v_{> P}}^2 }
      \left( 1 \pm O_\phi( \eta \hat{M}_Z ) \right).
  \end{align*}
  Using the above two estimations of $1/\norm{\hat{\v}_{t+1, >P}}^2$, we can rewrite $\Term_2$ as 
  \begin{align*}
    \Term_2 
    &= - \frac{ \norm{\v_{\le P}}^2 }{\norm{\v_{>P}}^2 }
      \frac{ 1 + 4 \eta \hat\phi_2^2  - 2 \eta \rho + 2 \eta \eps_v }{ 1  - 2 \eta \rho  } 
      \frac{ 2 \eta \inprod{\v_{>P}}{\Z_{>P}} }{ (1 - 2 \eta \rho) \norm{\v_{> P}}^2 } 
      \left( 1 \pm O_\phi( \eta \hat{M}_Z ) \right) \\
      &\qquad
      + \frac{  2 \eta \inprod{\v_{\le P}}{\Z_{\le P}} }{ (1 - 2 \eta \rho) \norm{\v_{> P}}^2 } \left( 1 \pm O_\phi( \eta \hat{M}_Z ) \right) \\
    &= - \frac{ \norm{\v_{\le P}}^2 }{\norm{\v_{>P}}^2 }
      \frac{ 1 + 4 \eta \hat\phi_2^2  - 2 \eta \rho + 2 \eta \eps_v }{ 1  - 2 \eta \rho  } 
      \frac{ 2 \eta \inprod{\v_{>P}}{\Z_{>P}} }{ (1 - 2 \eta \rho) \norm{\v_{> P}}^2 } 
      + \frac{  2 \eta \inprod{\v_{\le P}}{\Z_{\le P}} }{ (1 - 2 \eta \rho) \norm{\v_{> P}}^2 }
      \\
      &\qquad
      \pm 20 \frac{ \norm{\v_{\le P}}^3 }{\norm{\v_{>P}}^3 } \eta \hat{M}_Z  O_\phi( \eta \hat{M}_Z )
      \pm 4 \frac{ \norm{\v_{\le P}} }{\norm{\v_{>P}} } \eta \hat{M}_Z O_\phi( \eta \hat{M}_Z ) \\
    &= - \frac{ \norm{\v_{\le P}}^2 }{\norm{\v_{>P}}^2 }
      \frac{ 1 + 4 \eta \hat\phi_2^2  - 2 \eta \rho+  2 \eta \eps_v }{ 1  - 2 \eta \rho  } 
      \frac{ 2 \eta \inprod{\v_{>P}}{\Z_{>P}} }{ (1 - 2 \eta \rho) \norm{\v_{> P}}^2 } 
      + \frac{  2 \eta \inprod{\v_{\le P}}{\Z_{\le P}} }{ (1 - 2 \eta \rho) \norm{\v_{> P}}^2 } \\
      &\qquad
      \pm O_\phi\left( \frac{ \norm{\v_{\le P}} }{\norm{\v_{>P}} } \eta^2 \hat{M}_Z^2  \right).
  \end{align*}
  Finally, consider the third term 
  \[
    \Term_3 
    :=
      - \frac{
        \left( 1 + 4 \eta \hat\phi_2^2  - 2 \eta \rho + 2 \eta \eps_v \right) \norm{\v_{\le P}}^2
      }{
        \left( 1  - 2 \eta \rho \right) \norm{\v_{>P}}^2
      } \frac{ \xi_{> P} }{ \norm{\hat\v_{t+1, > P}}^2 }
      + \frac{ \xi_{\le P} }{ \norm{\hat\v_{t+1, > P}}^2 } .
  \]
  By our previous bounds on $\xi$, we have 
  \[
    |\Term_3| 
    \le 64 \frac{ \norm{\v_{\le P}}^2 }{ \norm{\v_{>P}}^2 } \frac{ C_\phi^4 \eta^2  d \hat{M}_Z^2 }{ \norm{\hat\v_{t+1, > P}}^2 }
      + 32 C_\phi^4 \eta^2  P \hat{M}_Z^2  
    \le 100 C_\phi^4 \eta^2  \hat{M}_Z^2  
      \left(
        P \vee \frac{ \norm{\v_{\le P}}^2 }{ \norm{\v_{>P}}^2 } d
      \right).
  \]
  Combine the above bounds, and we get 
  \begin{align*}
    \frac{ \norm{\v_{t+1, \le P}}^2  }{ \norm{\v_{t+1, > P}}^2  }
    &= \frac{ \norm{\v_{\le P}}^2 }{ \norm{\v_{>P}}^2 }
      \left( 1 + 4 \eta \hat\phi_2^2 + 2 \eta \eps_v \right) \\
      &\qquad
      - \frac{ \norm{\v_{\le P}}^2 }{\norm{\v_{>P}}^2 }
        \frac{ 1 + 4 \eta \hat\phi_2^2  - 2 \eta \rho+  2 \eta \eps_v }{ 1  - 2 \eta \rho  } 
        \frac{ 2 \eta \inprod{\v_{>P}}{\Z_{>P}} }{ (1 - 2 \eta \rho) \norm{\v_{> P}}^2 } 
      + \frac{  2 \eta \inprod{\v_{\le P}}{\Z_{\le P}} }{ (1 - 2 \eta \rho) \norm{\v_{> P}}^2 } \\
      &\qquad
      \pm \frac{ \norm{\v_{\le P}}^2 }{ \norm{\v_{>P}}^2 } 30 C_\phi^4 \eta^2  
      \pm 1000 C_\phi^4 
      \frac{ \norm{\v_{\le P}} }{\norm{\v_{>P}} } \eta^2 \hat{M}_Z^2 
      \left(
        1
        \vee \eta d \hat{M}_Z 
      \right) \\
      &\qquad
      \pm 100 C_\phi^4 \eta^2  \hat{M}_Z^2  
      \left(
        P \vee \frac{ \norm{\v_{\le P}}^2 }{ \norm{\v_{>P}}^2 } d
      \right).
  \end{align*}
  Let $H_{t+1}$ denote the second line and $\xi_{t+1}$ denote the last two lines. Recall our induction
  hypothesis $\norm{\v_{t, \le P}}^2 / \norm{\v_{t, > P}}^2 = \Theta(1) (1 + 4 \hat\phi_2^2 \eta)^t 
  \norm{\v_{0, \le P}}^2 / \norm{\v_{0, > P}}^2 = \Theta( (1 + 4 \hat\phi_2^2 \eta)^t P / d )$. Meanwhile,
  note that $\frac{ \norm{\v_{\le P}} }{\norm{\v_{>P}} } \le \frac{ \norm{\v_{\le P}}^2 }{\norm{\v_{>P}}^2 } \sqrt{\frac{d}{P}}$
  Then, we compute 
  \begin{align*}
    |\xi_{t+1}|
    &\lesssim_\phi 
      (1 + 4 \hat\phi_2^2 \eta)^t \frac{P}{d} \eta^2  
      + 
      (1 + 4 \hat\phi_2^2 \eta)^t \frac{P}{d} \sqrt{\frac{d}{P}} \eta^2 \hat{M}_Z^2 
      \left(
        1
        \vee \eta d \hat{M}_Z 
      \right) 
      + \eta^2  \hat{M}_Z^2  (1 + 4 \hat\phi_2^2 \eta)^t P \\
    &\lesssim_\phi (1 + 4 \hat\phi_2^2 \eta)^t \eta^2 \hat{M}_Z^2 
      \left( \sqrt{Pd} \eta \hat{M}_Z \vee  P \right) \\
    &\lesssim_\phi (1 + 4 \hat\phi_2^2 \eta)^t \eta^2 \hat{M}_Z^2 P
  \end{align*}
  Then, consider $H_{t+1}$. We have 
  \begin{align*}
    |H_{t+1}|
    &\le \left|
        \frac{ \norm{\v_{\le P}}^2 }{\norm{\v_{>P}}^2 }
        \frac{ 1 + 4 \eta \hat\phi_2^2  - 2 \eta \rho+  2 \eta \eps_v }{ 1  - 2 \eta \rho  } 
        \frac{ 2 \eta \inprod{\v_{>P}}{\Z_{>P}} }{ (1 - 2 \eta \rho) \norm{\v_{> P}}^2 } 
      \right|
      + \left| \frac{  2 \eta \inprod{\v_{\le P}}{\Z_{\le P}} }{ (1 - 2 \eta \rho) \norm{\v_{> P}}^2 } \right| \\
    &\lesssim 
      (1 + 4 \hat\phi_2^2 \eta)^t \frac{P}{d} \eta
      \left| \inprod{\overline{\v_{>P}}}{\Z_{>P}} \right|
      + \left( (1 + 4 \hat\phi_2^2 \eta)^t \frac{P}{d} \right)^{1/2} \eta 
        \left| \inprod{\overline{\v_{\le P}}}{\Z_{\le P}} \right|. 
  \end{align*}
  Since both $\left| \inprod{\overline{\v_{>P}}}{\Z_{>P}} \right|$ and $\left| \inprod{\overline{\v_{\le P}}}{\Z_{\le P}} \right|$
  are conditionally $(P, \theta)$-subweibull, $H_{t+1}$ is conditionally 
  $\left( O_\phi( (1 + 4 \hat\phi_2^2 \eta)^t \frac{P}{d} ), \theta \right)$-subweibull.
\end{proof}

With the above formula, we can now use Lemma~\ref{lemma: stochastic discrete gronwall} to analyze the dynamics
of the ratio of the norms. 

\begin{lemma}[Learning the subspace]
  \label{lemma: learning the subspace}
  Suppose that 
  \[
    P \gg_\phi \log^2 d 
    \quad\text{and}\quad 
    \eta 
    \lesssim_\phi \frac{1}{d P \log d} \left(
        \frac{P^2}{ \hat{M}_Z^2  } 
        \wedge 
        \frac{P}{ M_Z^2 \log^{\theta+1}(d / \delta_{\P}) }
      \right)
    = \tilde{\Theta}_\phi\left( \frac{1}{d P} \right).
  \]
  Then, throughout Stage~1.1, we have 
  \[
    \frac{ (1 + 4 \hat\phi_2^2 \eta)^t }{2} \frac{\norm{\v_{0, \le P}}^2}{\norm{\v_{0, > P}}^2}
    \le \frac{\norm{\v_{\le P}}^2}{\norm{\v_{> P}}^2}
    \le \frac{ 3 (1 + 4 \hat\phi_2^2 \eta)^t }{2} \frac{\norm{\v_{0, \le P}}^2}{\norm{\v_{0, > P}}^2},
  \]
  and Stage~1.1 takes at most $(1 + o(1)) (4 \hat{\phi}_2^2 \eta)\inv \log\left( d / P \right) = \tilde{O}_\phi\left( dP \right)$ iterations. 
  For this result to hold for the $P$ good neurons (satisfying \eqref{eq: good neuron}), it suffices to replace 
  $\delta_{\P}$ with $\delta_{\P}/P$. 
\end{lemma}
\begin{proof}
  First, by Lemma~\ref{lemma: dynamics of the norm ratio}, we have for any $t \le T$, 
  \[
    \frac{ \norm{\v_{t+1, \le P}}^2  }{ \norm{\v_{t+1, > P}}^2  }
    = \frac{ \norm{\v_{\le P}}^2 }{ \norm{\v_{>P}}^2 }
      \left( 1 + 4 \eta \hat\phi_2^2 + 2 \eta \eps_v \right) 
      + H_{t+1}
      + \xi_{t+1},
  \]
  where $\eps_v := \sum_{l\ge L} l \hat\phi_l^2 \norm{\v_{\le P}}_l^l / \norm{\v_{\le P}}^2$, where 
  $(H_{t+1})_t$ is a martingale difference sequence that is conditionally
  $\left( O_\phi( (1 + 4 \hat\phi_2^2 \eta)^t \frac{P}{d} ), \theta \right)$-subweibull, and 
  $(\xi_t)_t$ is an adapted process with 
  $|\xi_{t+1}| \lesssim_\phi (1 + 4 \hat\phi_2^2 \eta)^t \eta^2 \hat{M}_Z^2  P$ for all $t \in [T]$
  with probability at least $1 - \delta_{\P}$.
  By our induction hypothesis $v_p^2 \le \log^2 d / P$, we have 
  \[
    0 
    \le \eps_v 
    := \frac{1}{\norm{\v_{\le P}}^2} \sum_{l\ge L} l \hat\phi_l^2 \sum_{k=1}^{P} v_k^l 
    \le \sum_{l\ge L} l \hat\phi_l^2 \norm{\v_{\le P}}_\infty^{l-2} 
    \le \frac{C_\phi^2 \log^{L-2} d}{P^{L/2-1}} 
    =: \delta_v. 
  \]
  In particular, note that $\delta_v$ does not depend on $t$ and is $o(1)$. For notational simplicity, let 
  $X_t := \norm{\v_{\le P}}^2 / \norm{\v_{>P}}^2$, $x_t^- = (1 + 4 \eta)^t x_0$ and 
  $x_t^+ = (1 + 4 \eta (1 + \delta_v))^t x_0$. 
  $x^\pm$ will serve as the 
  lower and upper bounds for the deterministic counterpart of $X$, since 
  \begin{align*}
      \left( 1 + 4 \hat\phi_2^2 \eta \right) X_t 
      + \xi_{t+1} 
      + H_{t+1} 
    \le X_{t+1}
    \le \left( 1 + 4 \hat\phi_2^2 \eta (1 + \delta_v) \right) X_t 
      + \xi_{t+1} 
      + H_{t+1}.  
  \end{align*}
  Moreover, note that for any $t \le T$, we have 
  \begin{align*}
    \frac{x^+_t}{x^-_t}
    = \left( \frac{1 + 4 \hat\phi_2^2 \eta (1 + \delta_v)}{1 + 4 \hat\phi_2^2 \eta} \right)^t 
    &= \left( 
        \left( 1 + 4   \eta (1 + \delta_v) \right)\left( 1 - 4 \hat\phi_2^2 \eta \pm O_\phi(\eta^2) \right)
      \right)^t  \\
    &\le \left( 
        1 + 4 \hat\phi_2^2 \eta \delta_v
        \pm O_\phi(\eta^2)
      \right)^t \\
    &\le \exp\left( O_\phi(1) \eta T \left( \delta_v + \eta \right) \right). 
  \end{align*}
  Since $T \lesssim_\phi \log d / \eta$, the above implies 
  \[
    1 
    \le \frac{x^+_t}{x^-_t}
    \le  \exp\left( O_\phi(1) \log d \left( \delta_v + \eta \right) \right)
    \le 1 + O_\phi(1) \log d \left( \delta_v + \eta \right)
    = 1 + o(1), 
  \]
  where the last (approximate) identity holds whenever 
  \[
    \delta_v \ll \frac{1}{\log d} 
    \quad\Leftarrow\quad 
    \frac{C_\phi^2 \log^{L-2} d}{P^{L/2-1}}  \ll \frac{1}{\log d}
    \quad\Leftarrow\quad 
    P \gg_\phi \log^2 d. 
  \]
  In particular, this implies that the (multiplicative) difference between $x^+_t$ and $x^-_t$ is small. 
  Now, we apply Lemma~\ref{lemma: stochastic discrete gronwall} to $X_t$. In our case, we have 
  \[
    \Xi 
    \lesssim_\phi 
    \eta^2 \hat{M}_Z^2  P,
    \quad 
    \sigma_Z^2 
    \lesssim_\phi \eta^2 \frac{P}{d} M_Z^2, 
  \]
  $\alpha = 4 (1 + o(1)) \hat\phi_2^2 \eta$ and $X_0 = \Theta(P/d)$. Recall that $T \lesssim_\phi \log d / \eta$. 
  Hence, to meet the conditions of Lemma~\ref{lemma: stochastic discrete gronwall}, it suffices to choose
  \begin{align*}
    \eta^2 P \hat{M}_Z^2 
    \lesssim_\phi \frac{X_0 }{T} 
    &\quad\Leftarrow\quad 
    \eta 
    \lesssim_\phi \frac{1}{ d \hat{M}_Z^2  \log d }, 
    \\ 
    \eta^2 \frac{P}{d} M_Z^2
    \lesssim_\phi \frac{x_0^2}{T \log^{\theta+1}(T / \delta_{\P}) }
    &\quad\Leftarrow\quad 
    \eta 
    \lesssim_\phi \frac{P }{d M_Z^2 \log d \log^{\theta+1}(T / \delta_{\P}) }.
  \end{align*}
  Then, by Lemma~\ref{lemma: stochastic discrete gronwall},
  we have, with probability at least $1 - \Theta(\delta_{\P})$, $0.5 x^-_t \le X_t \le 1.5 x^+_t$. Since $x^+_t =
  (1 + o(1)) x^-_t$, this implies $0.5 x_t \le X_t \le 2 x_t$. To complete the proof, it suffices to note that 
  for $x_t$ to grow from $\Theta(P/d)$ to $1$, the number of iterations needed is bounded by 
  $ (1 + o(1)) (4 \phi_2^2 \eta)\inv \log\left( d / P \right)$. 
\end{proof}

\subsubsection{Preservation of the gap}
\label{subsubsec: stage 1.1: preservation of the gap}

Now, we show that the gap between the largest coordinate and the second-largest coordinate can be preserved
in Stage~1.1. Let $p = \argmax_{i \in [P]} v_i^2(0)$ and consider the ratio $v_q^2 / v_p^2$, where $q \in [P]$ is 
arbitrary. The proof is conceptually very similar to the previous one, except that we will use 
Lemma~\ref{lemma: stochastic gronwall when alpha = 0} instead of Lemma~\ref{lemma: stochastic discrete gronwall}.

\begin{lemma}
  \label{lemma: dynamics of vq2/vp2}
  For $p = \argmax_{i \in [P]} v_i^2(0)$ and any $q \in [P]$,
  we have 
  \[
    \frac{v_{t+1, q}^2}{v_{t+1, p}^2}
    \le \frac{v_{t, q}^2}{v_{t, p}^2}
      + H_{t+1} 
      + \xi_{t+1}, 
  \]
  where $(H_{t+1})_t$ is a martingale difference sequence that is conditionally
  $\left( O_\phi( \eta^2 d M_Z^2 ), \theta \right)$-subweibull, 
  and $(\xi_t)_t$ is an adapted process that is uniformly bounded by $O_\phi(\eta^2 d \hat{M}_Z^2)$ with probability
  at least $1 - \delta_{\P}$.
\end{lemma}
\begin{proof}
  For notational simplicity, define 
  \(
    \rho_{t, q/p} := v_{t, q}^2 / v_{t, p}^2 , 
  \)
  Our goal is to upper bound $\rho_{t, q/p}$. 
  Recall from Lemma~\ref{lemma: dynamics of vk2} that for any $k \le P$, we have 
  \[
    \hat{v}_{t+1, k}^2
    = v_{t, k}^2 
      + 2 \eta \gamma_{t, k} v_{t, k}^2 
      + 2 \eta v_{t, k} Z_{t+1, k}
      + \underbrace{
        \eta^2 \gamma_{t, k}^2 v_{t, k}^2
        + \eta^2 Z_{t+1, k}^2 
        + 2 \eta^2 \gamma_{t, k}^2 v_{t, k} Z_{t+1, k} 
      }_{=:\;\xi_{t+1, k}},
  \]
  where $\gamma_{k, t} := 2 \hat\phi_2^2 + L \hat\phi_L^2 v_k^{L-2} + \sum_{l>L} l \hat\phi_l^2 v_k^{l-2} - \rho$.
  Then, we compute 
  \begin{align*}
    \rho_{t+1, q/p} 
    &\le \frac{ 
        \left(1 + 2 \eta \gamma_{t, q}   \right) v_{t, q}^2 
        + 2 \eta v_{t, q} Z_{t+1, q}
        + \xi_{t+1, q}
      }{ 
        \left(1 + 2 \eta \gamma_{t, p} \right) v_{t, p}^2 
        + 2 \eta v_{t, p} Z_{t+1, p}
        + \xi_{t+1, p}
      }   \\
    &= \frac{ 
        \left(1 + 2 \eta \gamma_{t, q}   \right) v_{t, q}^2 
      }{
        \left(1 + 2 \eta \gamma_{t, p} \right) v_{t, p}^2 
      }
      - \frac{ 
        \left(1 + 2 \eta \gamma_{t, q}   \right) v_{t, q}^2 
      }{
        \left(1 + 2 \eta \gamma_{t, p} \right) v_{t, p}^2 
      }
      \frac{ 2 \eta v_{t, p} Z_{t+1, p}  }{  \hat{v}_{t+1, p}^2 } 
      + \frac{ 
        2 \eta v_{t, q} Z_{t+1, q}
      }{
        \left(1 + 2 \eta \gamma_{t, p} \right) v_{t, p}^2 
      } \\
      &\qquad
      + \frac{ \xi_{t+1, q} }{ \hat{v}_{t+1, p}^2 } 
      - \frac{ \left(1 + 2 \eta \gamma_{t, q}   \right) v_{t, q}^2 }{ \left(1 + 2 \eta \gamma_{t, p} \right) v_{t, p}^2 }
        \frac{ \xi_{t+1, p} }{ \hat{v}_{t+1, p}^2 } 
      - \frac{ 
        2 \eta v_{t, q} Z_{t+1, q}
      }{
        \left(1 + 2 \eta \gamma_{t, p} \right) v_{t, p}^2 
      }
      \frac{ 2 \eta v_{t, p} Z_{t+1, p} + \xi_{t+1, p} }{  \hat{v}_{t+1, p}^2 }  \\
    &=: \Term_1(\rho_{t+1, q/p} ) + \Term_2(\rho_{t+1, q/p} ) + \Term_3(\rho_{t+1, q/p} ),
  \end{align*}
  where $\Term_1$ contains the first term (signal term), $\Term_2$ contains the next two terms (approximate martingale difference terms), 
  and $\Term_3$ contains the last line (higher order error terms). 

  First, for the first term, we compute 
  \begin{align*}
    \Term_1 
    &= \rho_{t, q/p}
      \left( 1 + 2 \eta \gamma_{t, q} \right)
      \left(
        1 
        - 2 \eta \gamma_{t, p} 
        \left( 1 - \frac{2 \eta \gamma_{t, p}}{ 1 + 2 \eta \gamma_{t, p} } \right)
      \right) \\
    &= \rho_{t, q/p}
      \left( 1 + 2 \eta \gamma_{t, q} \right)
      \left(
        1 
        - 2 \eta \gamma_{t, p} 
        \pm  5 \eta^2 \gamma_{t, p}^2 
      \right) \\
    &= \rho_{t, q/p}
      \left(
        1 
        + 2 \eta \left( \gamma_{t, q} - \gamma_{t, p}  \right)
        \pm 20 \eta^2 \left( \gamma_{t, p}^2 \vee \gamma_{t, q}^2 \right)
      \right).
  \end{align*}
  Recall that $|\gamma_{t, k}| \le 2 C_\phi^2$ and note that 
  \(
    \gamma_{t, q} - \gamma_{t, p}
    = \sum_{l\ge L} l \hat\phi_l^2 v_q^{l-2} 
      - \sum_{l\ge L} l \hat\phi_l^2 v_p^{l-2} 
    \le 0.
  \)
  Hence, 
  \[
    \Term_1 \le \rho_{t, q/p} \left( 1 + 80 C_\phi^4 \eta^2 \right).
  \]
  Now, consider the martingale difference term 
  \[
    \Term_2 
    := - \rho_{t, q/p} \frac{ 1 + 2 \eta \gamma_{t, q} }{ 1 + 2 \eta \gamma_{t, p} }
      \frac{ 2 \eta v_{t, p} Z_{t+1, p}  }{  \hat{v}_{t+1, p}^2 } 
      + \frac{ 2 \eta v_{t, q} Z_{t+1, q} }{ \left(1 + 2 \eta \gamma_{t, p} \right) v_{t, p}^2 }.
  \]
  We rewrite the denominator as 
  \begin{align*}
    \frac{1}{\hat{v}_{t+1, p}^2}
    &= \frac{1}{
        v_{t, p}^2 
        + 2 \eta \gamma_{t, p} v_{t, p}^2 
        + 2 \eta v_{t, p} Z_{t+1, p}
        + \xi_{t+1, p}
      } \\
    &= \frac{1}{ v_{t, p}^2 + 2 \eta \gamma_{t, p} v_{t, p}^2  }
    - \frac{1}{ v_{t, p}^2 + 2 \eta \gamma_{t, p} v_{t, p}^2  }
      \frac{ 2 \eta v_{t, p} Z_{t+1, p} + \xi_{t+1, p} }{ \hat{v}_{t+1, p}^2 }.
  \end{align*}
  By \eqref{eq: Z high probability bound}, with probability at least $1 - \delta_{\P}/d^C$, we have 
  \[
    |\xi_{p, t+1}|  
    \lesssim_\phi \eta^2 v_{t, p}^2
      + \eta^2 \hat{M}_Z^2
      + \eta^2 |v_{t, p}| \hat{M}_Z
    \lesssim_\phi \eta^2 \hat{M}_Z^2.
  \]
  Therefore, 
  \[
    \frac{1}{\hat{v}_{t+1, p}^2}
    = \frac{1}{ v_{t, p}^2 + 2 \eta \gamma_{t, p} v_{t, p}^2  }
    \pm O_\phi\left( \frac{1}{ v_{t, p}^2 } \frac{ \eta \hat{M}_Z }{ v_{t, p} }  \right).
  \]
  Then, we can rewrite $\Term_2$ as 
  \begin{align*}
    \Term_2 
    &= - \rho_{t, q/p} \frac{ 1 + 2 \eta \gamma_{t, q} }{ 1 + 2 \eta \gamma_{t, p} } \frac{2 \eta v_{t, p} Z_{t+1, p}  }{ v_{t, p}^2 + 2 \eta \gamma_{t, p} v_{t, p}^2  } 
       + \frac{ 2 \eta v_{t, q} Z_{t+1, q} }{ \left(1 + 2 \eta \gamma_{t, p} \right) v_{t, p}^2 } 
      \pm O_\phi\left( \eta^2 \frac{\hat{M}_Z^2 }{ v_{t, p}^2 }  \right) \\
    &=: H_{t+1}
     \pm O_\phi\left( \eta^2 d \hat{M}_Z^2 \right).
  \end{align*}
  Note that $H_{t+1}$ is a martingale difference term with 
  \[
    |H_{t+1}|
    \lesssim_\phi  \frac{\eta |Z_{t+1, p}|  }{ v_{t, p}} + \frac{ \eta v_{t, q} |Z_{t+1, q}| }{ v_{t, p}^2 }
    \lesssim_\phi  \eta \sqrt{d} \left( |Z_{t+1, p}| + |Z_{t+1, q}| \right).
  \]
  Since $Z_{t+1,p}$ and $Z_{t+1, q}$ are both conditionally $(M_Z^2, \theta)$-subweibull, $H_{t+1}$ is 
  conditionally $\left( O_\phi( \eta^2 d M_Z^2 ), \theta \right)$-subweibull.
  Finally, consider 
  \[
    \Term_3
    := \frac{ \xi_{t+1, q} }{ \hat{v}_{t+1, p}^2 } 
      - \frac{ \left(1 + 2 \eta \gamma_{t, q}   \right) v_{t, q}^2 }{ \left(1 + 2 \eta \gamma_{t, p} \right) v_{t, p}^2 }
        \frac{ \xi_{t+1, p} }{ \hat{v}_{t+1, p}^2 } 
      - \frac{ 
        2 \eta v_{t, q} Z_{t+1, q}
      }{
        \left(1 + 2 \eta \gamma_{t, p} \right) v_{t, p}^2 
      }
      \frac{ 2 \eta v_{t, p} Z_{t+1, p} + \xi_{t+1, p} }{  \hat{v}_{t+1, p}^2 }.
  \]
  Since $|\xi_{q, t+1}| \vee |\xi_{p, t+1}| \le \eta^2 \hat{M}_Z^2$ and 
  $|Z_{t+1,p}| \vee |Z_{t+1,q}| \le \hat{M}_Z$, we have 
  \[
    |\Term_3|
    \lesssim_\phi 
      \eta^2 d \hat{M}_Z^2 .
  \]
  Combining the above bounds, we get 
  \[
    \rho_{t+1, q/p}
    \le \rho_{t, q/p} \left( 1 + 80 C_\phi^4 \eta^2 \right)
      + H_{t+1} + O_\phi\left( \eta^2 d \hat{M}_Z^2 \right) 
    = \rho_{t, q/p} + H_{t+1} + O_\phi\left( \eta^2 d \hat{M}_Z^2 \right)  .
  \]
\end{proof}

\begin{lemma}[Preservation of the gap]
  \label{lemma: preservation of the gap}
  Consider $\delta_c \in (0, 1)$, $p = \argmax_{i \in [P]} v_i^2(0)$ and any $q \in [P]$. 
  Suppose that 
  \[
    \eta 
    \lesssim_\phi
      \frac{\delta_c^2}{d P \log d}
      \left( \frac{P}{\hat{M}_Z^2 } \wedge \frac{P}{M_Z^2 \log^{\theta+1}(T/\delta_{\P}) } \right)
    = \tilde{\Theta}_\phi\left( \frac{\delta_c}{d P} \right).
  \]
  Then, we have 
  \[
    \sup_{t \le T} \left( \frac{v_{t, q}^2}{v_{t, p}^2} - \frac{v_{0, q}^2}{v_{0, p}^2}  \right)
    \le \delta_c 
    \quad\text{with probability at least $1 - \delta_{\P}$.}
  \]
\end{lemma}
\begin{proof}
  By Lemma~\ref{lemma: dynamics of vq2/vp2}, we have 
  \[
    \frac{v_{t+1, q}^2}{v_{t+1, p}^2}
    \le \frac{v_{t, q}^2}{v_{t, p}^2}
      + H_{t+1} 
      + \xi_{t+1}, 
  \]
  where $(H_{t+1})_t$ is a martingale difference sequence that is conditionally
  $\left( O_\phi( \eta^2 d M_Z^2 ), \theta \right)$-subweibull, 
  and $(\xi_t)_t$ is an adapted process that is uniformly bounded by $O_\phi(\eta^2 d \hat{M}_Z^2)$ with probability
  at least $1 - \delta_{\P}$. Hence, by Lemma~\ref{lemma: stochastic gronwall when alpha = 0}, we have 
  \begin{align*}
    \sup_{t \le T} \left( \frac{v_{t, q}^2}{v_{t, p}^2} - \frac{v_{0, q}^2}{v_{0, p}^2}  \right)
    &\lesssim_\phi T \eta^2 d \hat{M}_Z^2 
      + \sqrt{ \eta^2 d M_Z^2 T \log^{\theta+1}(T/\delta_{\P}) } \\
    &\lesssim_\phi 
      \eta d \hat{M}_Z^2 \log d
      + \sqrt{ \eta d M_Z^2 \log^{\theta+1}(T/\delta_{\P}) \log d }.
  \end{align*}
  For the RHS to be bounded by $\delta_c \in (0, 1)$, it suffices to require
  \begin{align*}
    \eta d \hat{M}_Z^2 \log d
    \lesssim_\phi \delta_c 
    &\quad\Leftarrow\quad 
    \eta \lesssim_\phi \frac{\delta_c}{d \hat{M}_Z^2 \log d}, \\
    \sqrt{ \eta d M_Z^2 \log^{\theta+1}(T/\delta_{\P}) \log d }.
    \lesssim_\phi \delta_c 
    &\quad\Leftarrow\quad 
    \eta \lesssim_\phi \frac{\delta_c^2}{d M_Z^2 \log^{\theta+1}(T/\delta_{\P}) \log d}. 
  \end{align*}
\end{proof}

\subsubsection{Other induction hypotheses}

In this subsection, we verify the induction hypothesis: $v_p^2 \lesssim \log^2 d / P$ for all $p \in [P]$. This condition is used 
to ensure the influence of the higher-order term is small compared to the influence of the second-order terms.

\begin{lemma}[Upper bound on $v_p^2$]
  \label{lemma: upper bound on vp2}
  Suppose that 
  \[
    \eta 
    \lesssim_\phi \frac{\log d}{d P} \left(
        \frac{P}{\hat{M}_Z^2} \wedge \frac{P}{M_Z^2  \log^{\theta+1}(d / \delta_{\P}) }
      \right).
  \]
  Then, throughout Stage~1, we have $v_p^2 \lesssim \log^2 d / P$.
\end{lemma}
\begin{proof}
  First, by Lemma~\ref{lemma: dynamics of vk2}, for any $p \le P$, we have 
  \begin{align*}
    \hat{v}_{t+1, p}^2
    &\le v_{t, p}^2 
      + 2 \eta \left(
        2 \hat\phi_2^2 + \sum_{l\ge L} l \hat\phi_l^2 v_k^{l-2}  
      \right) v_{t, p}^2 
      + 2 \eta v_{t, p} Z_{t+1, p}
      \\
      &\qquad
      + \underbrace{
        \eta^2 \gamma_{t, p}^2 v_{t, p}^2
        + \eta^2 Z_{t+1, p}^2 
        + 2 \eta^2 \gamma_{t, p}^2 v_{t, p} Z_{t+1, p} 
      }_{=: \xi_{t+1}} \\
    &\le 
      v_{t, p}^2 
      + 4 \hat\phi_2^2 \eta \bigg(
        1 
        + \underbrace{  \frac{C_\phi^2}{2 \hat\phi_2^2}  \frac{\log^{L-2}d}{P^{L/2-1}}  }_{=:\; \delta_v}
      \bigg) v_{t, p}^2 
      + 2 \eta v_{t, p} Z_{t+1, p}
      + \xi_{t+1}
  \end{align*}
  where $\gamma_{p, t} := 2 \hat\phi_2^2 + L \hat\phi_L^2 v_k^{L-2} + \sum_{l>L} l \hat\phi_l^2 v_k^{l-2}  - \rho$ is a 
  $\cF_t$-measurable random variable with $|\gamma_{t, p}| \le 2 C_\phi^2$. 
  By \eqref{eq: Z high probability bound}, with probability at least $1 - \delta_{\P}/T$, we have 
  \[
    |\xi_{t+1}|  
    \lesssim_\phi \eta^2 v_{t, p}^2
      + \eta^2 \hat{M}_Z^2
      + \eta^2 |v_{t, p}| \hat{M}_Z
    \lesssim_\phi 
      \eta^2 \hat{M}_Z^2 
  \]
  We maintain the induction hypothesis $v_{t, p}^2 \le 2 (1 + 4 \hat\phi_2^2 \eta (1 + \delta_v) )^t \log^2 d / d$.
  Under this induction hypothesis, we have 
  \[
    \left| 2 \eta v_{t, p} Z_{t+1, p} \right|
    \lesssim \eta \sqrt{ (1 + 4 \bar\phi_2^2 \eta (1 + \delta_v) ) \log^2 d / d } |Z_{t+1, p}|, 
  \]
  and therefore, is $\left( O_\phi\left( \eta^2 (1 + 4 \bar\phi_2^2 \eta (1 + \delta_v) )^t v_{0, p}^2 M_Z^2, \theta \right) \right)$-subweibull.
  Using the language of Lemma~\ref{lemma: stochastic discrete gronwall}, we have 
  \[
    \Xi \lesssim_\phi \eta^2 \hat{M}_Z^2 \quad\text{and}\quad 
    \sigma_Z^2 \lesssim_\phi \eta^2 (1 + 4 \bar\phi_2^2 \eta (1 + \delta_v) )^t \frac{\log^2 d}{d} M_Z^2.
  \]
  Therefore, as long as 
  \begin{align*}
    \eta^2 \hat{M}_Z^2 \lesssim_\phi \frac{\log^2 d / d}{T}
    &\quad\Leftarrow\quad
    \eta \lesssim_\phi \frac{\log d}{d \hat{M}_Z^2}, \\
    \eta^2 \frac{\log^2 d}{d} M_Z^2 
    \lesssim_\phi \frac{x_0^2}{T \log^{\theta+1}(T / \delta_{\P}) }
    &\quad\Leftarrow\quad
    \eta \lesssim_\phi \frac{\log d}{ d M_Z^2  \log^{\theta+1}(T / \delta_{\P}) },
  \end{align*}
  we have $v_{t, p}^2 \le 2 (1 + 4 \hat\phi_2^2 \eta (1 + \delta_v) )^t \log^2 d / d$ throughout Stage~1. In 
  particular, by Lemma~\ref{lemma: learning the subspace}, this implies
  \[
    v_{t, p}^2 
    \lesssim \exp^{1+\delta_v}\left( 4 \hat\phi_2^2 \eta T \right)  \frac{\log^2 d}{d}
    \lesssim \frac{\log^2 d}{P}.
  \]
\end{proof}

\subsection{Stage 1.2: recovery of the directions}

Let $\v$ be an arbitrary first-layer neuron. Assume w.l.o.g.~that $v_1^2$ is the largest at initialization and 
$v_{0, 1}^2 / \max_{2 \le k \le P} v_{0, k}^2 \ge 1 + \delta_0$. By Lemma~\ref{lemma: main lemma of stage 1.1}, we know 
this gap can be approximately preserved in the sense that $v_{0, 1}^2 / \max_{2 \le k \le P} v_{0, k}^2 \ge 1 + \delta_0/2$ 
holds. For notational simplicity, we will drop the factor $1/2$ in the sequel. 
Moreover, since we use symmetric initialization, we can further assume that $v_1 > 0$. 
In this subsection, we show that $v_1^2$ will grow from $\Omega(1/P)$ to $3/4$ and then to close to $1$. 
Formally, we prove the following lemma. 

\begin{lemma}[Stage 1.2]
  \label{lemma: main lemma of stage 1.2}
  Let $\v \in \S^{d-1}$ be an arbitrary first-layer neuron satisfying $v_{T_1, 1}^2 \ge c / P$ and 
  $v_{T_1, 1}^2 / \max_{2 \le k \le P} v_{T_1, k}^2 \ge 1 + c$ for some small universal constant $c > 0$.
  Let $\delta_{\P} \in (0, 1)$ and $\eps_v > 0$ be given. Suppose that we choose 
  \[
    \eta 
    \lesssim_\phi 
    \frac{\delta_0}{d P^{L/2}} \left(
      \frac{P}{\hat{M}_z^2 } \wedge \frac{d}{M_Z^2} \frac{1}{ \log^{\theta+1}( d / \delta_{\P} ) }
    \right)
    \wedge 
    \frac{\eps_*}{d P}
      \left(
        \frac{P}{\hat{M}_Z^2  \log(1/\eps_*)}
        \wedge
        \frac{\eps_* d P}{ M_Z^2 \log d \log^{\theta+1}(d / \delta_{\P}) }
      \right)
  \]
  Then, with probability at least $1 - O(\delta_{\P})$, we have $v_1^2 \ge 1 - \eps_v$ within
  $O_\phi\left( \left( P^{L/2-1} + \log(1/\eps_v) \right) / \eta \right)$ iterations.
\end{lemma}
\begin{proof}
  It suffices to combine Lemma~\ref{lemma: stage 1.2: weak recovery} and Lemma~\ref{lemma: stage 1.2: strong recovery}.
\end{proof}

\begin{lemma}[Dynamics of $v_1^2$]
  \label{lemma: dynamics of v12}
  We have 
  \[
    v_{t+1, 1}^2 
    \ge v_{t, 1}^2 
      \left(
        1
        + 2 \eta \sum_{l\ge L} l \hat\phi_l^2 v_1^{l-2}
        - 2 \eta \sum_{l\ge L} l \hat\phi_l^2 \norm{\v_{t, \le P}}_l^l
      \right)
      + H_{t+1}
      + \tilde{\xi}_{t+1},
  \]
  where $H_{t+1}$ is a martingale difference term that is conditionally $(O_\phi(\eta^2 v_{t, 1}^2 M_Z^2), \theta)$-subweibull
  and $\xi_{t+1}$ is bounded by $O_\phi(\eta^2 d \hat{M}_Z^2 v_{t, 1}^2)$ uniformly over $t \in [T]$ 
  with probability at least $1 - \delta_{\P}$.
\end{lemma}
\begin{proof}
  Recall from Lemma~\ref{lemma: dynamics of vk2} that 
  \[
    \hat{v}_{t+1, k}^2
    = v_{t, k}^2 
      + 2 \eta \gamma_{t, k} v_{t, k}^2 
      + 2 \eta v_{t, k} Z_{t+1, k}
      + \xi_{t+1, k},
  \]
  where $\gamma_{k, t} := 
  \indi\{k \le P\} \left( 2 \hat\phi_2^2 + L \hat\phi_L^2 v_k^{L-2} + \sum_{l>L} l \hat\phi_l^2 v_k^{l-2} \right) - \rho$ 
  is a $\cF_t$-measurable random variable with $|\gamma_{t, k}| \le 2 C_\phi^2$ and $(\xi_{t+1})_{t \in [T]}$ is 
  (uniformly) bounded by $O_\phi( \eta^2 \hat{M}_Z^2 )$ with probability at least $1 - \delta_{\P}$. 
  Sum over $k \in [d]$ and we get 
  \begin{align*}
    \norm{ \hat{\v}_{t+1} }^2 
    &= 1 + 2 \eta \sum_{k=1}^{d} \left(
        \indi\{k \le P\} \left( 2 \hat\phi_2^2 + L \hat\phi_L^2 v_k^{L-2} + \sum_{l>L} l \hat\phi_l^2 v_k^{l-2} \right) - \rho
      \right) v_{t, k}^2 
      + 2 \eta \inprod{\v_t}{\Z_{t+1}} 
      + \xi_{t+1}' \\
    &= 1 + 2 \eta 
      \left(
        2 \hat\phi_2^2 \norm{\v_{t, \le P}}^2
        + \sum_{l\ge L} l \hat\phi_l^2 \norm{\v_{t, \le P}}_l^l
        - \rho \norm{\v_t}^2
      \right) 
      + 2 \eta \inprod{\v_t}{\Z_{t+1}} 
      + \xi_{t+1}' \\
    &\le 
      \underbrace{ 
        1 + 2 \eta \left( 2 \hat\phi_2^2  - \rho \right) 
        + 2 \eta \sum_{l\ge L} l \hat\phi_l^2 \norm{\v_{t, \le P}}_l^l
      }_{=: N_v^2}
      + 2 \eta \inprod{\v_t}{\Z_{t+1}} 
      + \xi_{t+1}', 
  \end{align*}
  where $\xi_{t+1}$ is bounded by $O_\phi(\eta^2 d \hat{M}_Z^2)$. Recall from \eqref{eq: Z high probability bound}
  that $|\inprod{\v_t}{\Z_{t+1}}| \lesssim_\phi \hat{M}_Z$ and choose $\eta \le (d \hat{M}_Z^2)\inv$. As a result, 
  \begin{align*}
    \frac{1}{\norm{\hat\v_{t+1}}^2}
    &\ge \frac{1}{N_v^2}
      \left(
        1 
        - \frac{
          2 \eta \inprod{\v_t}{\Z_{t+1}} + \xi_{t+1}'
        }{N_v^2}
        \left(
          1 
          - \frac{2 \eta \inprod{\v_t}{\Z_{t+1}} + \xi_{t+1}'}{ 
            N_v^2
            + 2 \eta \inprod{\v_t}{\Z_{t+1}} + \xi_{t+1}'
          }
        \right)
      \right) \\
    &\ge 
      \frac{1}{N_v^2}
      - \frac{1}{N_v^2}
        \frac{
          2 \eta \inprod{\v_t}{\Z_{t+1}} 
        }{N_v^2}
      \pm  O_\phi(\eta^2 d \hat{M}_Z^2).
  \end{align*}
  Meanwhile, we have 
  \[
    \hat{v}_{t+1, 1}^2 
    = v_{t, 1}^2 
      + 2 \eta \left( 2 \hat\phi_2^2 - \rho \right) v_{t, 1}^2 
      + 2 \eta \sum_{l\ge L} l \hat\phi_l^2 v_1^l 
      + 2 \eta v_{t, 1} Z_{t+1, 1}
      + \xi_{t+1, 1}, 
  \]
  where $|\xi_{t+1, 1}| \lesssim_\phi \eta^2 \hat{M}_Z^2$. 
  Therefore, 
  \begin{align*}
    v_{t+1,1}^2 
    &\ge \hat{v}_{t+1, 1}^2 \left(
        \frac{1}{N_v^2}
        - \frac{1}{N_v^2}
          \frac{
            2 \eta \inprod{\v_t}{\Z_{t+1}} 
          }{N_v^2}
        \pm  O_\phi(\eta^2 d \hat{M}_Z^2)
      \right) \\
    &= \frac{\hat{v}_{t+1, 1}^2}{N_v^2}
        - \frac{\hat{v}_{t+1, 1}^2}{N_v^2}
          \frac{ 2 \eta \inprod{\v_t}{\Z_{t+1}} }{N_v^2}
        \pm  O_\phi( \eta^2 d \hat{M}_Z^2 v_{t, 1}^2) \\
    &= \frac{
        v_{t, 1}^2 
        + 2 \eta \left( 2 \hat\phi_2^2 - \rho \right) v_{t, 1}^2 
        + 2 \eta \sum_{l\ge L} l \hat\phi_l^2 v_1^l 
      }{N_v^2} \\
      &\qquad
      + \frac{ 2 \eta v_{t, 1} Z_{t+1, 1} }{N_v^2} 
      - \frac{
        v_{t, 1}^2 
        + 2 \eta \left( 2 \hat\phi_2^2 - \rho \right) v_{t, 1}^2 
        + 2 \eta \sum_{l\ge L} l \hat\phi_l^2 v_1^l 
      }{N_v^2}
      \frac{ 2 \eta \inprod{\v_t}{\Z_{t+1}} }{N_v^2} \\
      &\qquad
      \pm O_\phi\left( \eta^2 d \hat{M}_Z^2 v_{t, 1}^2 \right) \\
    &=: \Term_1\left( v_{t+1,1}^2 \right) + \Term_2\left( v_{t+1,1}^2 \right) 
      \pm O_\phi\left( \eta^2 d \hat{M}_Z^2 v_{t, 1}^2 \right), 
  \end{align*}
  where we have used the fact that $v_{t, 1}^2 \gtrsim 1/P$ to merge error terms of form $\eta^2 \hat{M}_Z^2$ 
  into $O_\phi\left( \eta^2 d \hat{M}_Z^2 v_{t, 1}^2 \right)$. Meanwhile, for $\Term_2$, we have 
  $|\Term_2| \lesssim_\phi \left| \eta v_{t, 1} Z_{t+1, 1} \right| + \left| \eta v_{t, 1}^2 \inprod{\v_t}{\Z_{t+1}} \right|$.
  Therefore, it is a $(O_\phi( \eta^2 v_{t, 1}^2 M_Z^2 ), \theta)$-subweibull. 
  For the signal term $\Term_1$, by \eqref{eq: 1/(1 + z) approx 1 - z}, we have 
  \begin{align*}
    \Term_1 
    &= \frac{
        v_{t, 1}^2 
        + 2 \eta \left( 2 \hat\phi_2^2 - \rho \right) v_{t, 1}^2 
        + 2 \eta \sum_{l\ge L} l \hat\phi_l^2 v_1^l 
      }{
        1 + 2 \eta \left( 2 \hat\phi_2^2  - \rho \right) 
        + 2 \eta \sum_{l\ge L} l \hat\phi_l^2 \norm{\v_{t, \le P}}_l^l
      } \\
    &= v_{t, 1}^2 
      \left(
        1
        + 2 \eta \left( 2 \hat\phi_2^2 - \rho \right) 
        + 2 \eta \sum_{l\ge L} l \hat\phi_l^2 v_1^{l-2}
      \right)
      \left(
        1 
        - 2 \eta \left( 2 \hat\phi_2^2  - \rho \right) 
        - 2 \eta \sum_{l\ge L} l \hat\phi_l^2 \norm{\v_{t, \le P}}_l^l
        \pm O_\phi\left( \eta^2 \right)
      \right) \\
    &= v_{t, 1}^2 
      \left(
        1
        + 2 \eta \sum_{l\ge L} l \hat\phi_l^2 v_1^{l-2}
        - 2 \eta \sum_{l\ge L} l \hat\phi_l^2 \norm{\v_{t, \le P}}_l^l
        \pm O_\phi\left( \eta^2 \right)
      \right).
  \end{align*}
  Combine the above estimations, set $H_{t+1} = \Term_2$, and we complete the proof. 
\end{proof}

\begin{lemma}[Weak reocovery of directions]
  \label{lemma: stage 1.2: weak recovery}
  Suppose that we choose 
  \[
    \eta 
    \lesssim_\phi 
    \frac{\delta_0}{d P^{L/2}} \left(
      \frac{P}{\hat{M}_z^2 } \wedge \frac{d}{M_Z^2} \frac{1}{ \log^{\theta+1}( d / \delta_{\P} ) }
    \right).
  \]
  Then with probability at least $1 - O(\delta_{\P})$, we will have $v_1^2 \ge 3/4$ within the following number of 
  iterations:
  \[
    O_\phi\left( \frac{P^{L/2-1}}{\eta} \right)
    = O_\phi\left( 
      P d P^{L-2}
      \left(
        \frac{P}{\hat{M}_z^2 } \wedge \frac{d}{M_Z^2} \frac{1}{ \log^{\theta+1}( d / \delta_{\P} ) }
      \right)\inv
      \right).
  \]
\end{lemma}
\begin{remark}
  Note that when $\hat{M}_Z^2, M_Z^2 = \tilde{O}_\phi(P)$, then the above is roughly $P \times (d P^{L-2})$.
  The $d P^{L-2}$ is the usual bound for online SGD when the noise has order $d$ instead of $P$. The first 
  $P$ comes from the fact that there are $P$ directions.
\end{remark}
\begin{proof}
  By Lemma~\ref{lemma: dynamics of v12}, we have 
  \[
    v_{t+1, 1}^2 
    \ge v_{t, 1}^2 
      \left(
        1
        + 2 \eta \sum_{l\ge L} l \hat\phi_l^2 v_1^{l-2}
        - 2 \eta \sum_{l\ge L} l \hat\phi_l^2 \norm{\v_{t, \le P}}_l^l
      \right)
      + H_{t+1}
      + \tilde{\xi}_{t+1},
  \]
  where $H_{t+1}$ is a martingale difference term that is conditionally $(O_\phi(\eta^2 M_Z^2 v_{t, 1}^2), \theta)$-subweibull,
  and $\xi_{t+1}$ is bounded by $O_\phi(\eta^2 d \hat{M}_Z^2 v_{t, 1}^2)$ for all $t \in [T]$ with probability at least $1 - \delta_{\P}$.
  For the signal term, we write 
  \[
    v_1^{l-2} - \norm{\v}_l^l
    = v_1^{l-2} - v_1^l - \sum_{k=2}^P v_k^l  
    = v_1^{l-2} (1 - v_1^2) 
      - \left( \norm{\v_{\le P}}^2 - v_1^2 \right) 
        \sum_{k=2}^{P} \frac{ v_k^2 }{ \norm{\v_{\le P}}^2 - v_1^2 } v_k^{l-2}. 
  \]
  Note that the last term is a weighted average of $v_k^{l-2}$. 
  Similar to the proof in Section~\ref{subsubsec: stage 1.1: preservation of the gap}, one can show that the induction
  hypothesis $v_1^2 / \max_{2 \le k \le P} v_k^2 \ge 1 + \delta_0 / 2$ remains true,\footnote{
    The only difference is that now the $L$-th order terms cannot be simply ignored as we no longer have 
    the induction hypothesis $v_p^2 \le \log^2 d / P$. To handle them, it suffices to note that if $v_1^2 \ge v_q^2$,
    then those $L$-th order terms of $v_1^2$ are also larger, which will even lead to an amplification of the gap. 
    In fact, this is why we can recover the directions using them.
  }
  which gives 
  \[
    \sum_{k=2}^{P} \frac{ v_k^2 }{ \norm{\v_{\le P}}^2 - v_1^2 } v_k^{l-2}
    \le \left( \max_{2 \le k \le P} v_k^2 \right)^{L/2-1}
    \le \left( \frac{v_1^2}{1 + \delta_0/2} \right)^{L/2-1}
    \le \frac{v_1^{L-2}}{1 + \delta_0/2}. 
  \]
  Therefore, 
  \begin{align*}
    v_1^{l-2} - \norm{\v}_l^l
    &\ge v_1^{l-2} (1 - v_1^2) 
      - \left( \norm{\v_{\le P}}^2 - v_1^2 \right) \frac{v_1^{L-2}}{1 + \delta_0/2} \\
    &= \frac{v_1^{l-2}}{1 + \delta_0/2} \left(
        1 
        + \delta_0 (1 - v_1^2) 
        - \norm{\v_{\le P}}^2 
      \right) 
    \ge \frac{\delta_0}{2} v_1^{l-2} \left( 1 - v_1^2 \right). 
  \end{align*}
  As a result, for the signal term, we have 
  \[
    v_1^2
    \left(
      1
      + 2 \eta \sum_{l\ge L} l \hat\phi_l^2 v_1^{l-2}    
      - 2 \eta \sum_{l\ge L} l \hat\phi_l^2 \norm{\v_{\le P}}_l^l
    \right) 
    \ge 
      v_1^2
      + L \hat\phi_L^2 \delta_0 \left( 1 - v_1^2 \right) \eta v_1^L.  
  \]
  In particular, when $v_1^2 \le 3/4$, we have 
  \[
    v_{t+1, 1}^2 
    \ge v_1^2
      + \frac{L \hat\phi_L^2}{4} \delta_0  \eta v_1^L
      + H_{t+1}
      + \xi_{t+1}. 
  \]
  Thus, using the notations of Lemma~\ref{lemma: appendix: stochastic gronwall (polynomial)}, we have 
  \[
    \alpha = \frac{L \hat\phi_L^2}{4} \delta_0  \eta, \quad 
    \Xi \lesssim_\phi \eta^2 d \hat{M}_z^2, \quad 
    \sigma_Z^2 \lesssim_\phi \eta^2 M_Z^2, \quad 
    p = L/2, \quad 
    x_0 = \Omega(1/P).
  \]
  To meet the conditions of Lemma~\ref{lemma: appendix: stochastic gronwall (polynomial)}, it suffices to choose 
  \begin{align*}
    \alpha \lesssim x_0^{p-1} 
    &\quad\Leftarrow\quad 
    \alpha \lesssim \delta_0\inv x_0^{p-1} , \\
    \Xi \lesssim \alpha x_0^{p-1} 
    &\quad\Leftarrow\quad  
    \eta \lesssim_\phi \frac{\delta_0}{d \hat{M}_z^2 P^{L/2-1}}, \\
    \sigma_Z^2 \lesssim_\theta \frac{ \alpha x_0^p }{ \log^{\theta+1}\left( \log(1/x_0) / ( \alpha x_0^{p-1} \delta_{\P} ) \right)}
    &\quad\Leftarrow\quad 
    \eta 
    \lesssim_\phi \frac{\delta_0}{M_Z^2 P^{L/2}} \frac{1}{ \log^{\theta+1}( d / \delta_{\P} ) }.
  \end{align*}
  Combine the above and we get the condition
  \[
    \eta 
    \lesssim_\phi 
    \frac{\delta_0}{d P^{L/2}} \left(
      \frac{P}{\hat{M}_z^2 } \wedge \frac{d}{M_Z^2} \frac{1}{ \log^{\theta+1}( d / \delta_{\P} ) }
    \right).
  \]
  Finally, we apply Lemma~\ref{lemma: appendix: stochastic gronwall (polynomial)} to complete the proof. 
\end{proof}

\begin{lemma}[Strong recovery of directions]
  \label{lemma: stage 1.2: strong recovery}
  Let $\v \in \S^{d-1}$ be an arbitrary first-layer neuron. Let $\delta_{\P}$ and $\eps_*$ be given. Suppose that 
  we choose 
  \[
    \eta 
    \lesssim_\phi 
      \frac{\eps_*}{d P}
      \left(
        \frac{P}{\hat{M}_Z^2  \log(1/\eps_*)}
        \wedge
        \frac{\eps_* d P}{ M_Z^2 \log d \log^{\theta+1}(d / \delta_{\P}) }
      \right).
  \]
  Then, with probability at least $1 - O(\delta_{\P})$, we have $v_1^2 \ge 1 - \eps_*$ within $O_\phi(\log(1/\eps_*)/\eta)$
  iterations.
\end{lemma}
\begin{proof}
  Again, By Lemma~\ref{lemma: dynamics of v12}, we have 
  \[
    v_{t+1, 1}^2 
    \ge v_{t, 1}^2 
      \left(
        1
        + 2 \eta \sum_{l\ge L} l \hat\phi_l^2 v_1^{l-2}
        - 2 \eta \sum_{l\ge L} l \hat\phi_l^2 \norm{\v_{t, \le P}}_l^l
      \right)
      + H_{t+1}
      + \tilde{\xi}_{t+1},
  \]
  where $H_{t+1}$ is a martingale difference term that is conditionally $(O_\phi(\eta^2 M_Z^2 v_{t, 1}^2), \theta)$-subweibull,
  and $\xi_{t+1}$ is bounded by $O_\phi( \eta^2 d \hat{M}_Z^2 v_{t, 1}^2)$ for all $t \in [T]$ with probability at least $1 - \delta_{\P}$.
  Meanwhile, by the proof of the previous lemma, we have 
  \begin{align*}
    v_1^2 \left(
        1
        + 2 \eta \sum_{l\ge L} l \hat\phi_l^2 v_1^{l-2}    
        - 2 \eta \sum_{l\ge L} l \hat\phi_l^2 \norm{\v_{\le P}}_l^l
      \right) 
    &\ge v_1^2
      + 2 L \hat\phi_L^2 \frac{c_{g, L}}{1 + c_{g, L}}  \left( 1 - v_1^2 \right) \eta v_1^L \\
    &\ge v_1^2
      + \eta  \hat\phi_L^2 \frac{2 L c_{g, L}}{1 + c_{g, L}} \left( \frac{3}{4} \right)^L  
      \left( 1 - v_1^2 \right) \\
    &=: v_1^2 + \eta c_{g, \phi} \left( 1 - v_1^2 \right), 
  \end{align*}
  for some constant $c_{g, \phi} > 0$ that depends only on $c_{g, L}$ and $\phi$. Thus, 
  \[
    1 - v_{t+1, 1}^2 
    \ge \left( 1 - v_1^2 \right) - \eta c_{g, L, \phi} \left( 1 - v_1^2 \right)
      - H_{t+1}
      - \xi_{t+1}. 
  \]
  In the language of Lemma~\ref{lemma: stochastic discrete gronwall},\footnote{When $\alpha$ is negative, it 
  suffices to replace $x_0$ with our target $\eps_*$.} we have 
  \[
    \alpha = - \eta c_{g, L, \phi}, \quad  
    \eta T \lesssim_\phi \log(1/\eps_*), \quad 
    \sigma_Z^2 \lesssim_\phi \eta^2 M_Z^2, \quad 
    \Xi \lesssim_\phi \eta^2 d \hat{M}_Z^2. 
  \]
  To meet the conditions of Lemma~\ref{lemma: stochastic discrete gronwall}, it suffices to choose 
  \begin{align*}
    \Xi \lesssim \frac{\eps_*}{T} 
    &\quad\Leftarrow\quad
    \eta \lesssim_\phi \frac{\eps_*}{d \hat{M}_Z^2  \log(1/\eps_*)}, \\
    \sigma_Z^2  
    \lesssim \frac{x_0^2}{T \log^{\theta+1}(T / \delta_{\P}) }
    &\quad\Leftarrow\quad
    \eta 
    \lesssim \frac{\eps_*^2}{ M_Z^2 \log d \log^{\theta+1}(d / \delta_{\P}) }.
  \end{align*}
  Under the above conditions, by Lemma~\ref{lemma: stochastic discrete gronwall}, we have $v_1^2 \ge 1 - \eps_*$ within 
  $T = O_\phi( \log(1/\eps_*) / \eta )$ iterations with probability at least $1 - O(\delta_{\P})$. 
\end{proof}

\subsection{Deferred proofs in this section}
\label{subsec: stage 1: deferred proofs}

\begin{proof}[Proof of Lemma~\ref{lemma: dynamics of vk2}]
  First, recall from \eqref{eq: hat vt+1 k = ...} that 
  \[
    \hat{v}_{t+1, k}
    = v_{t, k}
      + \eta \indi\{k \le P\}  \left( 
        2 \hat\phi_2^2 
        + L \hat\phi_L^2 v_k^{L-2} 
        + \sum_{l>L} l \hat\phi_l^2 v_k^{l-2}  
      \right) v_k 
      - \eta \rho v_k  
      + \eta Z_{t+1, k}.
  \]
  Therefore, 
  \begin{align*}
    \hat{v}_{t+1, k}^2
    &= v_{t, k}^2
      + 2 \eta v_{t, k} \left( 
        \indi\{k \le P\}  \left( 
          2 \hat\phi_2^2 
          + L \hat\phi_L^2 v_k^{L-2} 
          + \sum_{l>L} l \hat\phi_l^2 v_k^{l-2}  
        \right) v_k 
        - \rho v_k  
        + Z_{t+1, k}
      \right) \\
      &\qquad
      + \eta^2 \left( 
        \indi\{k \le P\}  \left( 
          2 \hat\phi_2^2 
          + L \hat\phi_L^2 v_k^{L-2} 
          + \sum_{l>L} l \hat\phi_l^2 v_k^{l-2}  
        \right) v_k 
        - \rho v_k  
        + Z_{t+1, k}
      \right)^2 \\
    &=: v_{t, k}^2 + \Term_1\left( \hat{v}_{t+1, k}^2 \right)
      + \Term_2\left( \hat{v}_{t+1, k}^2 \right).
  \end{align*}
  For the first term, we rewrite it as 
  \[
    \Term_1 
    = 2 \eta v_{t, k}^2 \left( 
      \indi\{k \le P\}  \left( 
        2 \hat\phi_2^2 
        + L \hat\phi_L^2 v_k^{L-2} 
        + \sum_{l>L} l \hat\phi_l^2 v_k^{l-2}  
      \right) 
      - \rho 
    \right)
    + 2 \eta v_{t, k} Z_{t+1, k}.
  \]
  Consider the second term. For notation simplicity, put 
  \[
    \gamma_{k, t} 
    := \indi\{k \le P\}  \left( 2 \hat\phi_2^2 + L \hat\phi_L^2 v_k^{L-2} + \sum_{l>L} l \hat\phi_l^2 v_k^{l-2} \right) 
        - \rho.
  \]
  Note that $\gamma_{k, t}$ is $\cF_t$-measurable and by Assumption~\ref{assumption: link function}, we have 
  \begin{gather*}
    2 \hat\phi_2^2 + L \hat\phi_L^2 v_k^{L-2} + \sum_{l>L} l \hat\phi_l^2 v_k^{l-2}  
    \le 2 \hat\phi_2^2 + L \hat\phi_L^2 + \sum_{l>L} l \hat\phi_l^2 
    \le C_\phi^2, \\
    \rho 
    := 2 \hat\phi_2^2 \norm{\v_{\le P}}^2 
      + L  \hat\phi_L^2 \norm{\v_{\le P}}_L^L 
      + \sum_{l>L} l \hat\phi_l^2 \norm{\v_{\le P}}_l^l
    \le 2 \hat\phi_2^2 
      + L  \hat\phi_L^2 
      + \sum_{l>L} l \hat\phi_l^2  
    \le C_\phi^2, 
  \end{gather*}
  and therefore $|\zeta_{k, t}| \le 2 C_\phi^2$. Then, we compute
  \[
    \frac{\Term_2}{\eta^2}
    = \left( 
        \gamma_{t, k} v_k  
        + Z_{t+1, k}
      \right)^2 
    = \gamma_{t, k}^2 v_{t, k}^2
      + Z_{t+1, k}^2 
      + 2 \gamma_{t, k}^2 v_{t, k} Z_{t+1, k}. 
  \]
  Combine the above two bounds and we get 
  \begin{align*}
    \hat{v}_{t+1, k}^2
    &= v_{t, k}^2 
      + 2 \eta \gamma_{t, k} v_{t, k}^2 
      + 2 \eta v_{t, k} Z_{t+1, k}
      + \eta^2 \gamma_{t, k}^2 v_{t, k}^2
      + \eta^2 Z_{t+1, k}^2 
      + 2 \eta^2 \gamma_{t, k}^2 v_{t, k} Z_{t+1, k}.
  \end{align*}
  Now, consider the last three terms. By \eqref{eq: Z high probability bound}, we have 
  $|Z_{t+1, k}| \lesssim_\phi \hat{M}_Z$ with probability at least $1 - \delta_{\P}$ for all $t \in [T]$. Thus, 
  \[
    \left|
      \eta^2 \gamma_{t, k}^2 v_{t, k}^2
      + \eta^2 Z_{t+1, k}^2 
      + 2 \eta^2 \gamma_{t, k}^2 v_{t, k} Z_{t+1, k} 
    \right|
    \lesssim_\phi \eta^2 \hat{M}_Z^2.
  \]
\end{proof}

\section{Stage 2: training the second layer}
\label{sec: stage 2}

\begin{lemma}
  \label{lemma: stage 2: existence of good a}
  Suppose that for each $p \in [P]$, there exists a first-layer neuron $\v_{i_p}$ with 
  $v_{i_p, p} \ge \sqrt{1 - \eps_v}$ for some small positive $\eps_v = O(1/P)$, then we can choose $\a_* \in \R^m$ 
  with $\norm{\a_*} = \sqrt{P}$ such that 
  \[
    \Loss(\a_*, \V)
    := \E \left( f_*(\x) - f(\x; \a_*, \V)  \right)^2
    \le 20 C_\phi^2 P^2 \eps_v. 
  \]
\end{lemma}
\begin{proof}
  Choose one $\v_{i_p}$ for each $p \in [P]$. 
  Then, we set the $i_p$-th entries of $\a_*$ to be $1$ and all other entries $0$. Then, we write 
  \begin{align*}
    \left( f_*(\x) - f(\x; \a_*, \V)  \right)^2 
    &= \left( \sum_{k=1}^{P} \left( \phi(x_k) - \phi(\v_{i_k} \cdot \x) \right) \right)^2 \\
    &= \sum_{k, l=1}^{P} \left( \phi(x_k) - \phi(\v_{i_k} \cdot \x) \right) 
      \left( \phi(x_l) - \phi(\v_{i_l} \cdot \x) \right).  
  \end{align*}
  By expanding $\phi$ in the Hermite basis, for any $\v, \v' \in \S^{d-1}$, we have 
  \[
    \E_{\x \sim \Gaussian{0}{\Id}}[ \phi(\v \cdot \x) \phi(\v' \cdot \x) ]
    = \sum_{i=0}^{\infty} \hat\phi_i^2  \inprod{\v}{\v'}^i 
    = \sum_{i=2}^{\infty} \hat\phi_i^2  \inprod{\v}{\v'}^i. 
  \]
  Hence, for $k = l$, we have 
  \begin{align*}
    \E \left( \phi(x_k) - \phi(\v_{i_k} \cdot \x) \right)^2
    &= \E \phi^2(x_k) + \E \phi^2(\v_{i_k} \cdot \x) 
      - 2 \E \phi(x_k) \phi(\v_{i_k} \cdot \x) \\
    &= 2 \sum_{i=2}^{\infty} \hat\phi_i^2 \left( 1 - \inprod{\e_k}{\v_{i_k}}^i \right) \\
    &\le 2 C_\phi^2 \eps_v. 
  \end{align*}
  Meanwhile, for $k \ne l$, we have 
  \begin{align*}
    & \E \left( \phi(x_k) - \phi(\v_{i_k} \cdot \x) \right) \left( \phi(x_l) - \phi(\v_{i_l} \cdot \x) \right) \\
    =\;& \E  \phi(x_k) \phi(x_l) 
      + \E  \phi(\v_{i_k} \cdot \x) \phi(\v_{i_l} \cdot \x) 
      - \E \phi(x_k) \phi(\v_{i_l} \cdot \x) 
      - \E \phi(\v_{i_k} \cdot \x) \phi(x_l) \\
    =\;& 
      \sum_{i=2}^{\infty} \hat\phi_i^2 \left(
        \inprod{\v_{i_k}}{\v_{i_l}}^i 
        - v_{i_l, k}^i 
        - v_{i_k, l}^i 
      \right). 
  \end{align*}
  Note that $v_{i_l, k}^2 \vee v_{i_k, l}^2 \le \eps_v$ and 
  \[ 
    \inprod{\v_{i_k}}{\v_{i_l}}^2
    \le 2 v_{i_l, k}^2 + 2 \inprod{\v_{i_k} - \e_k}{\v_{i_l}}^2 
    \le 2 \eps_v + 2 \norm{\v_{i_k} - \e_k}^2 
    = 2 \eps_v + 4 \left( 1 - v_{i_k, k} \right) 
    \le 6 \eps_v. 
  \] 
  As a result, 
  \[
    \E \left( \phi(x_k) - \phi(\v_{i_k} \cdot \x) \right) \left( \phi(x_l) - \phi(\v_{i_l} \cdot \x) \right) \\
    \le 10 C_\phi^2 \eps_v. 
  \]
  Combining these two cases, we obtain 
  \[
    \Loss 
    = \E \left( f_*(\x) - f(\x; \a_*, \V)  \right)^2 
    \le 20 C_\phi^2 P^2 \eps_v. 
  \]
\end{proof}

Now, we are ready to prove the following generalization bound for Stage~2. The proof of it is adapted from 
Section~B.8 of \cite{oko_learning_2024}, which in turn is based on (\cite{damian_neural_2022,abbe_merged-staircase_2022,%
ba_high-dimensional_2022}). 

\begin{lemma}
  \label{lemma: stage 2: generalization error}
  Suppose that for each $p \in [P]$, there exists a first-layer neuron $\v_{i_p}$ with 
  $v_{i_p, p}^2 \ge 1 - \eps_v$ for some small positive $\eps_v = O(1/P)$. Then, there exists some $\lambda > 0$
  such that the ridge estimator $\hat{\a}$ we obtain in Stage~2 satisfies 
  \[
    \norm{f(\cdot; \hat{\a}, \V) - f_*}_{L^1(D)} 
    \le \frac{8 \norm{\a_*} \sqrt{m}}{\sqrt{N} \delta_{\P}} + \sqrt{ 10 L P^2 \eps_v }, 
  \]
  with probability at least $1 - 2 \delta_{\P}$. 
\end{lemma}
\begin{proof}
  For notational simplicity, let $D = \mcal{N}(0, 1)$ and $\hat{D} = \frac{1}{N} \sum_{n=1}^{N} \delta_{\x_{T+n}}$ denote the empirical 
  distribution of the samples we use in Stage~2. In addition, we write $f_{\a}$ for $f(\cdot; \a, \V)$ where $\V$ 
  is the first-layer weights we have obtained in Stage~1 and $\X = (\x_{T+n})_{n=1}^N$. 
  
  Let $\a_* \in \R^m$ denote the second-layer weights we constructed in 
  Lemma~\ref{lemma: stage 2: existence of good a} and $\hat{\a} \in \R^m$ denote the ridge estimator obtained 
  via minimizing $\a \mapsto \norm{ f_* - f_{\a} }_{L^2(\hat{D})}^2 + \lambda \norm{\a}^2$. By the equivalence 
  between norm-constrained linear regression and ridge regression, there exists $\lambda > 0$ such that 
  \[
    \norm{ f_* - f_{\hat{\a}} }_{L^2(\hat{D})}^2
    \le \norm{ f_* - f_{\a_*} }_{L^2(\hat{D})}^2
    \quad\text{and}\quad 
    \norm{\hat\a} \le \norm{\a_*}. 
  \]
  Choose this $\lambda$ and let $\cF := \braces{ f(\cdot; \a) \,:\, \norm{\a} \le \norm{\a_*} }$ be our hypothesis 
  class. Note that $f_{\hat{\a}} \in \cF$. Moreover, we have 
  \begin{align*}
    \norm{f_{\hat{\a}} - f_*}_{L^1(D)} 
    &= \left( \norm{f_{\hat{\a}} - f_*}_{L^1(D)} - \norm{f_{\hat{\a}} - f_*}_{L^1(\hat{D})} \right) 
      + \norm{f_{\hat{\a}} - f_*}_{L^1(\hat{D})} \\
    &\le \sup_{\a\,:\,\norm{\a} \le \norm{\a_*}} \left( \norm{f_{\a} - f_*}_{L^1(D)} - \norm{f_{\a} - f_*}_{L^1(\hat{D})} \right) 
      + \norm{f_{\hat{\a}} - f_*}_{L^1(\hat{D})} \\
    &\le \sup_{\a\,:\,\norm{\a} \le \norm{\a_*}} \left( \norm{f_{\a} - f_*}_{L^1(D)} - \norm{f_{\a} - f_*}_{L^1(\hat{D})} \right) 
      + \norm{f_{\a_*} - f_*}_{L^2(\hat{D})}, 
  \end{align*}
  where we used the fact that $\norm{f_{\hat{\a}} - f_*}_{L^1(\hat{D})} \le \norm{f_{\hat{\a}} - f_*}_{L^2(\hat{D})}
  \le \norm{f_{\a_*} - f_*}_{L^1(\hat{D})}$ in the last line. 

  Now, we bound the first term. Let $\bsigma := (\sigma_n)_{n=1}^N$ be i.i.d.~Rademacher variables that are also independent of 
  everything else. By symmetrization and Theorem~7 of \cite{meir_generalization_2003}, we have 
  \begin{align*}
    & \E_{\X}\left[ \sup_{\a\,:\,\norm{\a} \le \norm{\a_*}} \left( \norm{f_{\a} - f_*}_{L^1(D)} - \norm{f_{\a} - f_*}_{L^1(\hat{D})} \right)  \right] \\
    \le\;& 2 \E_{\X, \bsigma} \sup_{\a\,:\,\norm{\a} \le \norm{\a_*}} \frac{1}{N} \sum_{t=1}^N \sigma_t \left| f_a(\x_{T+n}) - f_*(\x_{T+n}) \right| \\
    \le\;& 2 \E_{\X, \bsigma} \sup_{\a\,:\,\norm{\a} \le \norm{\a_*}} \frac{1}{N} \sum_{t=1}^N \sigma_t \left( f_a(\x_{T+n}) - f_*(\x_{T+n}) \right) \\
    \le\;& 
      \frac{2}{N} \E_{\X, \bsigma} \sup_{\a\,:\,\norm{\a} \le \norm{\a_*}} \sum_{t=1}^N \sigma_t f_a(\x_{T+n}) 
      + \cancelto{0}{2 \E_{\X, \bsigma} \frac{1}{N} \sum_{t=1}^N \sigma_t f_*(\x_{T+n})}. 
  \end{align*}
  Note that the first term is two times the Rademacher complexity $\mrm{Rad}_N(\cF)$ of $\cF$ (see, for example, 
  Chapter~4 of \cite{wainwright_high-dimensional_2019}). By (the proof of) Lemma~48 of \cite{damian_neural_2022}, we have 
  \begin{align*}
    \mrm{Rad}_N(\cF) 
    \le \frac{\norm{\a_*}}{\sqrt{N}} \sqrt{ \E_{\x \sim \Gaussian{0}{\Id_d}} \norm{ \phi(\V\x) }^2 } 
    &= \frac{\norm{\a_*}}{\sqrt{N}} 
      \sqrt{ \sum_{k=1}^{m} \E_{\x \sim \Gaussian{0}{\Id_d}}  \phi^2(\v_k \cdot \x) } \\
    &= \frac{\norm{\a_*} \sqrt{m}}{\sqrt{N}} 
      \sqrt{ \E_{x_1 \sim \Gaussian{0}{1}}  \phi^2(x_1) } \\
    &= \frac{2 \norm{\a_*} \sqrt{m}}{\sqrt{N}}.  
  \end{align*}
  In other words, we have 
  \[
    \E \sup_{\a\,:\,\norm{\a} \le \norm{\a_*}} \left( \norm{f_{\a} - f_*}_{L^1(D)} - \norm{f_{\a} - f_*}_{L^1(\hat{D})} \right)
    \le \frac{4 \norm{\a_*} \sqrt{m}}{\sqrt{N}}.  
  \]
  Hence, for any $\delta_{\P} \in (0, 1)$, by Markov's inequality, we have 
  \[
    \sup_{\a\,:\,\norm{\a} \le \norm{\a_*}} \left( \norm{f_{\a} - f_*}_{L^1(D)} - \norm{f_{\a} - f_*}_{L^1(\hat{D})} \right)
    \le \frac{4 \norm{\a_*} \sqrt{m}}{\sqrt{N} \delta_{\P}},
  \]
  with probability at least $1 - \delta_{\P}$. 
  Apply the same argument to $\norm{f_{\a_*} - f_*}_{L^2(\hat{D})}$ and recall from 
  Lemma~\ref{lemma: stage 2: existence of good a} that $\norm{f_{\a_*} - f_*}_{L^2(D)}^2 \le 10 L P^2 \eps_v$, and 
  we obtain
  \[ 
    \norm{f_{\hat{\a}} - f_*}_{L^1(D)} 
    \le \frac{8 \norm{\a_*} \sqrt{m}}{\sqrt{N} \delta_{\P}} + \sqrt{ 10 L P^2 \eps_v }, 
  \] 
  with probability at least $1 - 2 \delta_{\P}$.
\end{proof}

\section{Proof of the main theorem}
\label{sec: proof of the main theorem}

\mainThm*

\begin{proof}
  First, by Lemma~\ref{lemma: Typical structure at initialization}, we should choose 
  $m = \Theta\left( P \log(P/\delta_{\P}) \right)$ and the $\delta_0$ in Lemma~\ref{lemma: main lemma of stage 1.1}
  and Lemma~\ref{lemma: main lemma of stage 1.2} can be chosen to be $\Theta(1/\log P)$. 
  Meanwhile, by Lemma~\ref{lemma: stage 2: generalization error}, to achieve target $L^1$-error $\eps_*$ with 
  probability at least $1 - O(\delta_{\P})$, we need 
  \[
    N 
    \gtrsim \frac{P m}{\eps_*^2 \delta_{\P}^2}
    = \Theta\left( \frac{P^2 \log(P/\delta_{\P})}{\eps_*^2 \delta_{\P}^2} \right), \quad 
    \eps_v = O_\phi\left( \frac{\eps_*^2}{P^2} \right). 
  \]
  By Lemma~\ref{lemma: persample gradient}, we have $M_Z \lesssim_\phi P^{1/2}$ and 
  $\hat{M}_Z \lesssim_\phi  P^{1/2} \log^{\theta}\log(P/\delta_{\P})$ where $\theta = 1 / (2(1 + q))$.
  Then, to meet the conditions of Lemma~\ref{lemma: main lemma of stage 1.1} and Lemma~\ref{lemma: main lemma of stage 
  1.2} (uniformly over those $P$ good neurons), it suffices to choose 
  \[
    \eta 
    \lesssim_\phi 
    \frac{1}{\log^{2\theta+3}(d/\delta_{\P}) }  \left(
      \frac{1}{d P^{L/2}} 
      \wedge 
      \frac{\eps_*^2}{P \log(1/\eps_*)} 
    \right)
  \]
  Then, by Lemma~\ref{lemma: main lemma of stage 1.1} and Lemma~\ref{lemma: main lemma of stage 1.2}, 
  the numbers of iterations needed for Stage~1.1 and Stage~1.2 are $O_\phi( \log(d/P) / \eta )$ and 
  $O_\phi\left( \left( P^{L/2-1} + \log(1/\eps_v) \right) / \eta \right)$, respectively. Thus, the total number of 
  iterations is bounded by 
  \[
    T 
    = O_\phi\left( \frac{\log d + P^{L/2-1} + \log(P/\eps_*)}{\eta} \right)
    = \tilde{O}_\phi\left( 
      d P^{L-1}  
      \vee
      \frac{P^{L/2} \log(1/\eps_*)} {\eps_*^2}
    \right).
  \]
\end{proof}

\section{Stochastic Induction}
\label{sec: stochastic induction}

Our proof is essentially a large induction: When certain properties hold, we know how to analyze the dynamics and 
can show certain quantities are bounded with high probability. Meanwhile, certain properties hold as long as those 
quantities are still well-controlled. In the deterministic setting, this seemingly looped argument can be made 
formal by either mathematical induction (in discrete time) or the continuity argument (in continuous time). 
In this subsection, we show the same can also be done in the presence of randomness and derive a stochastic 
version of Gronwall's lemma and its generalizations. 

We start with an example where Doob's submartingale inequality can be directly used. 
Let $(\Omega, \cF, (\cF_t)_t, \P)$ be our filtered probability space and $(Z_t)_t$ be a martingale difference sequence. 
Suppose that $\E[ Z_{t+1}^2 \mid \cF_t ]$ is uniformly bounded by $\sigma_Z^2$. Then, by Doob's submartingale inequality, 
for any $M > 0$ and $T > 0$, we have 
\[
  \P\left[ \sup_{t \le T} \left| \sum_{s=1}^{t} Z_s \right| \ge M \right]
  \le M^{-2} \E\left( \sum_{s=1}^T Z_s \right)^2 
  = \frac{T \sigma_Z^2}{M^2}. 
\]
In particular, this implies that when $M = \omega( \sigma_Z \sqrt{T} )$, we have 
$\sup_{t \le T} \left| \sum_{s=1}^{t} Z_s \right| \le M$ with high probability. 

Note that there is no need to do any kind of ``induction'' in the above example because of the unconditional uniform 
bound on $\E[ Z_{t+1}^2 \mid \cF_t ]$.
However, things become subtle if instead of assuming $\E[ Z_{t+1}^2 \mid \cF_t ]$ is always bounded by $\sigma_Z^2$, we 
assume it to be bounded by $\sigma_Z^2$ when $\sup_{s \le t} \left| \sum_{r=1}^s Z_r \right| \le M$. Intuitively, since 
$M$ is chosen so that $\sup_{t \le T} \left| \sum_{s=1}^{t} Z_s \right| \le M$ holds with high probability, the bounds 
$\E[ Z_{t+1}^2 \mid \cF_t ] \le \sigma_Z^2$ should also hold with high probability and we can still use  
Doob's submartingale inequality as before. Now, we formalize this argument. 

\begin{lemma}
  Let $(Z_t)_t$ be a martingale difference sequence. Suppose that there exists $M, \sigma_Z > 0$ such that 
  if $\sup_{s \le t} \left| \sum_{r=1}^s Z_s \right| \le M$, then we have $\E[ Z_{t+1}^2 \mid \cF_t ] \le \sigma_Z^2$.
  Then, we have 
  \[
    \P\left[ \sup_{t \le T} \left| \sum_{s=1}^{t} Z_s \right| > M \right]
    \le \frac{T \sigma_Z^2}{M^2}. 
  \]
  Note that this bound is the same as the one we obtained with the assumption that 
  $\E[ Z_{t+1}^2 \mid \cF_t ] \le \sigma_Z^2$ always holds. 
\end{lemma}
\begin{proof}
  Consider the stopping time $\tau := \inf\{ t \ge 0 \;:\; \left| \sum_{s=1}^{t} Z_s \right| > M  \}$. 
  By definition, we have $\sup_{s \le t} \left| \sum_{r=1}^s Z_s \right| \le M$ for all $t \le \tau$.
  Then, we define $Y_{t+1} = Z_{t+1} \indi\{ t < \tau \}$. Note that $(Y_t)$ is a martingale difference sequence
  with $\E[ Y_{t+1}^2 \mid \cF_t ] \le \sigma_Z^2$. As a result, by Doob's submartingale inequality, we have 
  \(
    \P\left[ \sup_{t \le T} \left| \sum_{s=1}^{t} Y_s \right| > M \right]
    \le T \sigma_Z^2 / M^2. 
  \)
  To relate it to $(Z_t)_t$, we compute 
  \begin{align*}
    \P\left[ \sup_{t \le T} \left| \sum_{s=1}^{t} Z_s \right| > M \right]
    = \P\left[ 
        \sup_{t \le T} \left| \sum_{s=1}^{t} Z_s \right| > M
        \wedge
        \tau \le T 
      \right] 
    &= \P\left[ 
        \left| \sum_{s=1}^{\tau} Z_s \right| > M
        \wedge
        \tau \le T 
      \right] \\
    &= \P\left[ 
        \left| \sum_{s=1}^{\tau} Y_s \right| > M
        \wedge
        \tau \le T 
      \right] \\
    &\le \frac{T \sigma_Z^2}{M^2}, 
  \end{align*}
  where the first and second identities comes from the definition of $\tau$ and the third from the fact $Z_t = Y_t$
  for all $t \le \tau$. 
\end{proof}

Now, we consider a more complicated case, where the process of interest is not a pure martingale. Suppose that the 
process $(X_t)_t$ satisfies 
\[
  X_{t+1} = (1 + \alpha) X_t + \xi_{t+1} + Z_{t+1},  
  \quad X_0 = x_0 > 0,
\]
where the signal growth rate $\alpha > 0$ and initialization $x_0 > 0$ are given and fixed, $(\xi_t)_t$ is an 
adapted process, and $(Z_t)_t$ is a martingale difference sequence. In most cases, $(\xi_t)_t$ will represent the 
higher-order error terms. 

Our goal is control the difference between $X_t$ and its deterministic counterpart 
$x_t = (1 + \alpha)^t x_0$.
To this end, we recursively expand the RHS to obtain 
\begin{align*}
  X_{t+1}
  &= (1 + \alpha)^2 X_{t-1} 
    + (1 + \alpha) \xi_t + \xi_{t+1} 
    + (1 + \alpha) Z_t + Z_{t+1} \\
  &= (1 + \alpha)^{t+1} x_0 
    + \sum_{s=1}^t (1 + \alpha)^{t-s} \xi_{s+1} 
    + \sum_{s=1}^t (1 + \alpha)^{t-s} Z_{s+1}.  
\end{align*}
Divide both sides with $(1 + \alpha)^{t+1}$ and replace $t+1$ with $t$. Then, the above becomes 
\[
  X_t (1 + \alpha)^{-t} 
  = x_0 
    + \sum_{s=1}^t (1 + \alpha)^{-s} \xi_s 
    + \sum_{s=1}^t (1 + \alpha)^{-s} Z_s.  
\]
Note that $\left( (1 + \alpha)^{-t} Z_t \right)_t$ is still a martingale difference sequence. Ideally, $|\xi_t|$ should 
be small as it represents the higher-order error terms, and we have bounds on the conditional variance of 
$Z_t$ so that we can apply Doob's submartingale inequality to the last term. Unfortunately, in many cases, 
since $\xi_{t+1}$ and $Z_{t+1}$, particularly their maximum and (conditional) variance, can potentially depend on 
$(X_s)_{s \le t}$, we may only be able to assume $|\xi_{t+1}| \le (1 + \alpha)^t \Xi$ with probability at least $1 - \delta_{\P, \xi}$ 
(for each $t$) and $\E[ Z_{t+1}^2 \mid \cF_t ] \le (1 + \alpha)^t \sigma_Z^2$ for some $\xi_{\P, \xi}$, $\Xi$ and $\sigma_Z^2$ when, say, $X_t = (1 \pm 0.5) x_t$. Still, we can use the previous argument to 
estimate the probability that $X_t \notin (1 \pm 0.5) x_t$ for some $t \le T$. We now formalize this argument. 
In addition, instead of Doob's $L^2$ submartingale inequality, we will use the following extension of Freedman's 
inequality, which allows us to improve the dependence on failure probability from linear to poly-logarithmic. The 
proof of this lemma is deferred to the end of this section.

\begin{lemma}[Freedman's inequality with subweibull variables]
  \label{lemma: freedman, subweibull}
  Let $\{Z_t\}_t$ be a martingale difference sequence that is conditionally $(\sigma^2, \theta)$-subweibull, i.e.,
  \[
    \P\left[ |Z_t| \ge M \mid \cF_{t-1} \right]
    \le C \exp\left( -\left( M / \sigma \right)^{1/\theta} \right), \quad 
    \forall M \ge 0, 
  \]
  for some universal constant $C > 0$. Then, for any $\delta_{\P} \in (0, 1)$, we have 
  \[
    \left| \sum_{t=1}^{T} Z_t \right|
    \lesssim_\theta \sigma \sqrt{T \log^{\theta+1}\left( T/\delta_{\P} \right) },
    \quad\text{with probability at least $1 - \delta_{\P}$}.
  \]
\end{lemma}

\stochasticInductionGronwall*
\begin{remark}
  This lemma can be easily generalized to cases where we have multiple induction hypotheses. For example, 
  if we have another process $X_{t+1}' = (1 + \alpha') X_t' + \xi_{t+1}' + Z_{t+1}'$ and we need both $X_t 
  = (1 \pm 0.5) x_t$ and $X_t' = (1 \pm 0.5) x_t'$ for the bounds on $|\xi_{t+1}|, |\xi'_{t+1}|$, $\E[ Z_{t+1}^2 \mid \cF_t]$,
  $\E[(Z'_{t+1})^2 \mid \cF_t]$ to hold. In this case, the final failure probability will be bounded by 
  $T(\delta_{\P, \xi} + \delta_{\P, \xi'}) + 2 \delta_{\P}$. 
\end{remark}
\begin{remark}
  If the recurrence relationship is $X_{t+1} \le (1 + \alpha) X_t + \xi_{t+1} + Z_{t+1}$, and we only want an upper 
  bound, then we can replace $x_0$ with any $x_0^+ \ge x_0$ in \eqref{eq: conditions of stochastic gronwall} and 
  the definition of the deterministic process $(x_t)$.
\end{remark}
\begin{proof}
  Let $\tau := \inf\braces{ t \ge 0 \,:\, X_t \notin (1 \pm \delta) x_t }$ and set 
  $\hat{\xi}_{t+1} := \xi_{t+1} \indi\{t \le \tau\}$ and $\hat{Z}_{t+1} := Z_{t+1} \indi\{t \le \tau\}$.
  Clear that $\tau$ is a stopping time, $\hat\xi$ is adapted, and $\hat{Z}$ is still a martingale difference sequence. 
  Moreover, by our hypotheses, we have $|\hat{\xi}_t| \le (1 + \alpha)^t \Xi$ with probability at least 
  $1 - \delta_{\P, \xi}$ and $\hat{Z}_{t+1}$ is conditionally $((1 + \alpha)^t \sigma_Z^2, \theta)$-subweibull.
  As a result, 
  \[
    \left| \sum_{s=1}^t (1 + \alpha)^{-s} \hat{\xi}_s \right| 
    \le \Xi t 
    \le T \Xi \quad 
    \text{with probability at least $1 - T \delta_{\P, \xi}$},
  \]
  and by Lemma~\ref{lemma: freedman, subweibull}, 
  \[
    \sup_{t \in [T]} \left| \sum_{s=1}^t (1 + \alpha)^{-s} Z_s \right|
    \le \sigma_Z \sqrt{ T \log^{\theta+1}(T / \delta_{\P}) }
    \quad 
    \text{with probability at least $1 - \delta_{\P}$}.
  \]
  Hence, for any $\delta_{\P} \in (0, 1)$, if we assume 
  \[
    \Xi \lesssim \frac{x_0}{T}
    \quad\text{and}\quad 
    \sigma_Z^2  
    \lesssim \frac{x_0^2}{T \log^{\theta+1}(T / \delta_{\P}) }, 
  \]
  then with probability at least $1 - \delta_{\P} - T \delta_{\P, \xi}$, we have 
  \[
    \left| 
      \sum_{s=1}^t (1 + \alpha)^{-s} \hat{\xi}_s 
      + \sum_{s=1}^t (1 + \alpha)^{-s} \hat{Z}_s
    \right| 
    \le \frac{x_0}{2}, 
    \quad 
    \forall t \in [T]. 
  \]
  Recall that 
  \[
    X_t 
    = (1 + \alpha)^t \left(
      x_0 
      + \sum_{s=1}^t (1 + \alpha)^{-s} \xi_s 
      + \sum_{s=1}^t (1 + \alpha)^{-s} Z_s.  
    \right)
    \quad\text{and}\quad 
    x_t = (1 + \alpha)^t x_0.
  \]
  Then, we compute 
  \begin{align*}
    \P\left[ \exists t \in [T], X_t \notin (1 \pm 0.5) x_t \right]
    &= \P\left[ \exists t \in [T], X_t \notin (1 \pm 0.5) x_t \wedge \tau \le T \right] \\
    &= \P\left[ X_{\tau} \notin (1 \pm 0.5) x_{\tau} \wedge \tau \le T \right] \\
    &= \P\left[ 
        \left|
          \sum_{s=1}^\tau (1 + \alpha)^{-s} \xi_s 
          + \sum_{s=1}^\tau (1 + \alpha)^{-s} Z_s
        \right|
        \ge 0.5 x_0
        \wedge \tau \le T 
      \right] \\
    &= \P\left[ 
        \left|
          \sum_{s=1}^\tau (1 + \alpha)^{-s} \hat{\xi}_s 
          + \sum_{s=1}^\tau (1 + \alpha)^{-s} \hat{Z}_s
        \right|
        \ge 0.5 x_0
        \wedge \tau \le T 
      \right] \\
    &\le \delta_{\P} + T \delta_{\P, \xi}. 
  \end{align*}
\end{proof}

The above lemmas will be used in Stage~1.1 to estimate the growth rate of the signals. The next lemma considers 
the case where $\alpha$ is $0$ and will be used to show the gap between the largest and the second-largest coordinates
can be preserved during Stage~1.1.

\begin{lemma}
  \label{lemma: stochastic gronwall when alpha = 0}
  Suppose that $(X_t)_t$ satisfies
  \[
    X_{t+1} \le X_t + \xi_{t+1} + Z_{t+1},  
    \quad X_0 = x_0 > 0,
  \]
  where the signal growth rate $\alpha > 0$ and initialization $x_0 > 0$ are given and fixed, $(\xi_t)_t$ is an 
  adapted process, and $(Z_t)_t$ is a martingale difference sequence. 
  
  Let $T > 0$ and $\delta_{\P} \in (0, 1)$ be given. 
  Suppose that there exists some $\delta_{\P, \xi} \in (0, 1)$ and $\Xi, \sigma_Z > 0$ such that 
  for every $t \le T$, $|\xi_t| \le \Xi$ with probability at least $1 - \delta_{\P, \xi}$ and 
  $Z_{t+1}$ is conditionally $(\sigma_Z^2, \theta)$-subweibull. 
  Then, we have 
  \[
    \sup_{t \le T} \left| X_t - x_0 \right|
    \le T \Xi + \sigma_Z \sqrt{ T \log^{\theta+1}(T/\delta_{\P}) }
    \quad\text{with probability at least $1 - T \delta_{\P, \xi} - \delta_{\P}$}. 
  \]
\end{lemma}
\begin{proof}
  Recursively expand the RHS, and we obtain 
  \[
    X_t \le x_0 + \sum_{s=1}^{t} \xi_s + \sum_{s=1}^{t} Z_s. 
  \]
  Clear that 
  \[ 
    \sup_{t \le T} \left| \sum_{s=1}^{t} \xi_t \right|
    \le T \Xi 
    \quad\text{with probability at least $1 - T \delta_{\P, \xi}$}. 
  \] 
  Meanwhile, by Lemma~\ref{lemma: freedman, subweibull}, we have 
  \[
    \sup_{t \le T} \left| \sum_{s=1}^{t} \hat{Z}_s \right|
    \le \sigma_Z \sqrt{ T \log^{\theta+1}(T/\delta_{\P}) }
    \quad\text{with probability at least $1 - \delta_{\P}$}. 
  \]
  Combine the above bounds and we complete the proof. 
\end{proof}

Now, we consider the case where the signal grows at a polynomial instead of linear rate. This lemma will be used in 
Stage~1.2, where the $L$-th order terms dominate. We will need the following estimations on the corresponding deterministic 
process. Its proof is deferred to the end of this section.  

\begin{lemma}
  \label{lemma: xt+1 = xt + a xtp}
  Consider the process $x_{t+1} = x_t + \alpha x_t^p$ where $x_0, \alpha$ are small positive real numbers and $p > 1$. 
  Let $T$ be the time $x_t$ first goes above $1$. We have 
  \[
    T \lesssim \frac{1}{(p - 1) \alpha x_0^{p-1}} 
    \quad\text{and}\quad
    \sum_{t=0}^{T-1} x_t 
    \lesssim \frac{1}{p \alpha x_0^{p-2}},
    \quad 
    \text{if $\alpha \lesssim x_0^{p-1}/p$}.
  \]
\end{lemma}
\begin{remark}
  This lemma provides upper bounds on the time needed for $x_t$ to grow from $x_0 = o(1)$ to $1$ and 
  the sum of $x_t$ in this process. Note that the second upper bound is essentially $T x_0$. Intuitively, this is 
  because due to the sharp transition behavior of this polynomial system, $x_t \approx x_0$ for most of the time. 
\end{remark}

\begin{lemma}
  \label{lemma: appendix: stochastic gronwall (polynomial)}
  Let $(X_t)_t$ be a non-negative stochastic process satisfying
  \begin{equation}
    \label{eq: appendix: stochastic gronwall (polynomial)}
    X_{t+1}
    \ge X_t + \alpha X_t^p + Z_{t+1} + \xi_{t+1}, 
    \quad 
    X_0 = x_0 > 0,
  \end{equation}
  where $\alpha > 0$, $(Z_{t+1})_t$ is a martingale difference sequence, and $(\xi_t)_t$ is an adapted process. 
  Let $\hat{x}_t$ be the solution to the deterministic recurrence relationship
  \(
    \hat{x}_{t+1} = \hat{x}_t + \alpha \hat{x}_t^p, \hat{x}_0 = x_0 / 2.
  \)

  Let $\delta_{\P} \in (0, 1)$ be given and 
  \(
    T := \inf\braces{ t \ge 0 \,:\, X_t \ge 1 }.
  \)
  Suppose that there exists $\Xi, \sigma_Z > 0$ and $\delta_{\P, \xi} \in (0, 1)$ such that if 
  $X_t \ge \hat{x}_t$ and $t \le T$, we have $|\xi_t| \le \Xi X_t$ with probability at least $1 - \delta_{\P, \xi}$ and 
  $Z_{t+1}$ is conditionally $(\sigma_Z^2 X_t, \theta)$-subweibull.
  Then, if 
  \[
    \alpha \lesssim x_0^{p-1}/p, \quad 
    \Xi \lesssim p \alpha x_0^{p-1}, \quad 
    \sigma_Z^2 \lesssim_\theta \frac{ \alpha x_0^p }{\log^{\theta+1}\left( \log(1/x_0) / ( \alpha x_0^{p-1} \delta_{\P} ) \right)},
  \]
  then with probability at least $1 - \delta_{\P, \xi} / \left( \alpha (x_0/2)^{p-1}\right) - \delta_{\P}$, we have 
  $T \lesssim \left( p \alpha (x_0/2)^{p-1}\right)\inv$ and $X_t \ge \hat{x}_t$ for all $t \le T$.
\end{lemma}
\begin{proof}
  Note that we can rewrite \eqref{eq: appendix: stochastic gronwall (polynomial)} as $X_{t+1} \ge X_t (1 + \alpha X_t^{p-1})
  + \xi_t + Z_t$ and view it as the linear recurrence relationship in Lemma~\ref{lemma: stochastic discrete gronwall}
  with a non-constant growth rate. This suggests defining the counterpart of $(1 + \alpha)^t$ as 
  \[
    P_{s, t} 
    := \begin{cases}
      \prod_{r=s}^{t-1} ( 1 + \alpha X_r^{p-1} ), & t > s, \\
      1, & t = s. 
    \end{cases}
  \]
  Then, we can unroll \eqref{eq: appendix: stochastic gronwall (polynomial)} as 
  \begin{align*}
    X_1 
    &\ge X_0 \left( 1 + \alpha X_0^{p-1} \right) + \xi_1 + Z_1, \\
    X_2 
    &\ge \left(
        X_0 \left( 1 + \alpha X_0^{p-1} \right) + \xi_1 + Z_1
      \right) 
      \left( 1 + \alpha X_1^{p-1} \right) 
      + \xi_2 + Z_2 \\
    &\ge X_0 \left( 1 + \alpha X_0^{p-1} \right) \left( 1 + \alpha X_1^{p-1} \right) 
      + \left( 1 + \alpha X_1^{p-1} \right) \left( \xi_1 + Z_1 \right)
      + \xi_2 + Z_2 \\
    &= X_0 P_{0, 2}
      + P_{1, 2} \left( \xi_1 + Z_1 \right)
      + \xi_2 + Z_2, \\
    X_3 
    &\ge X_2 \left( 1 + \alpha X_2^{p-1} \right) + \xi_3 + Z_3 \\
    &\ge \left(
        X_0 P_{0, 2}
        + P_{1, 2} \left( \xi_1 + Z_1 \right)
        + \xi_2 + Z_2
      \right)
      \left( 1 + \alpha X_2^{p-1} \right) 
      + \xi_3 + Z_3 \\
    &= X_0 P_{0, 3} 
        + P_{1, 3} \left( \xi_1 + Z_1 \right)
        + P_{2, 3}  \left( \xi_2 + Z_2 \right)
        + \xi_3 + Z_3. 
  \end{align*}
  Continue the above expansion, and eventually we obtain 
  \[
    X_t 
    \ge X_{t_0} P_{t_0, t}
      + \sum_{s=t_0}^{t-1} P_{s+1, t} \left( \xi_{s+1} + Z_{s+1} \right), \quad \forall t \ge t_0 \ge 0.
  \]
  Since $X$ is non-negative, we have $P_{s, t} \ge 1 > 0$. Hence, we can rewrite the above as 
  \[
    P_{t_0, t}\inv X_t 
    \ge X_{t_0}
      + \sum_{s=t_0}^{t-1} P_{t_0, s+1}\inv \left( \xi_{s+1} + Z_{s+1} \right), \quad \forall t \ge t_0 \ge 0.
  \]

  We wish the repeat the argument in the proof of Lemma~\ref{lemma: stochastic discrete gronwall}, showing that 
  the last term is smaller than $x_0 / 2$. Unfortunately, this approach will not work directly.
  We have only assumed $|\xi_{t+1}| \le \Xi X_t$ and $\E[Z_{t+1}^2 | \cF_t] \le \sigma_Z^2 X_t$. Since 
  $X_t$ can be much larger than $\hat{x}_t$, we cannot directly use our assumption to control the size of noises. 
  On the other hand, note that if $X_t \gg \hat{x}_t$, the induction hypothesis will less likely be violated, so 
  in principle, $X_t \gg \hat{x}_t$ should help us. To ``enforce'' the $X_t \lesssim \hat{x}_t$ condition, we consider 
  the following recoupling strategy: whenever $X_t \ge 4 \hat{x}_t$, we restart $\hat{x}_t$ at $X_t / 2$. 
  This recoupling will only increase the value of $\hat{x}_t$, and it ensures $X_t \lesssim \hat{x}_t$ always hold.
  
  We now formalize the above argument. To this end, let $\Phi_t: \R_{>0} \to \R_{>0}$ be the 
  flow map of the recurrence relationship $x_{t+1} = x_t + \alpha x_t^p$. That is, $\Phi_s(x)$ is the value of 
  $x_s$ if $(x_s)_s$ is generated by $x_{t+1} = x_t + \alpha x_t^p$ with $x_0 = x$. Then, we inductively define the 
  following sequences of ``deterministic'' processes and stopping times:
  \begin{align*}
    & \hat{x}_t\ps{0}
    = \Phi_t( X_0/2 ), \quad 
    && \iota\ps{1}
    = \inf\braces{ t \ge 0 \,:\, X_t \ge 4 \hat{x}_t\ps{0} },  \\
    & \hat{x}_t\ps{k}
    = \Phi_{t - \iota\ps{k}}\left( X_{\iota\ps{k}}/2 \right),  
    && \iota\ps{k+1}
    = \inf\braces{ t > \iota\ps{k} \,:\, X_t \ge 4 \hat{x}_t\ps{k} }, \quad \forall k \ge 1.
  \end{align*}
  In words, $\iota\ps{k}$ is the time we switch to the $k$th coupling. By construction, 
  By construction, $\hat{x}_t\ps{k}$ in non-decreasing in both $t$ and $k$, $0 =: \iota\ps{0} <  \cdots < \iota\ps{k} < \cdots$,
  and $\hat{x}\ps{k}_{\iota\ps{k}} = X_{\iota\ps{k}}/2 \ge 2 \hat{x}\ps{k-1}_{\iota\ps{k}} 
  \ge 2 \hat{x}\ps{k-1}_{\iota\ps{k-1}} \ge \cdots \ge 2^{k-1} x_0$. In particular, the last property implies that 
  there are only finitely many couplings before $\hat{x}\ps{k}_t$ reaches any fixed constant. 

  Then, we abuse notations, redefining 
  \[
    \hat{x}_t := \sum_{k=0}^{\infty} \indi\braces{ \iota\ps{k} \le t < \iota\ps{k+1} } \hat{x}\ps{k}_t. 
  \]
  Clear that at each $t$, only one summand is nonzero. By construction, we always have $X_t \le 4 \hat{x}_t$.
  Since this $\hat{x}_t$ is no smaller than the original one, it suffices to bound the probability that 
  $X_t \le \hat{x}_t$ for some $t \le T$. 
  Note that $X_t \le \hat{x}_t$ if and only if 
  there exists some $k \in \mbb{N}_{\ge 0}$ with $t \in \left[ \iota\ps{k}, \iota\ps{k+1} \right)$ such that 
  \[
    X_{\iota\ps{k}} + \sum_{\iota\ps{k}=0}^{t-1} P_{\iota\ps{k}, s+1}\inv \left( \xi_{s+1} + Z_{s+1} \right) 
    \le P_{\iota\ps{k}, t}\inv \hat{x}_t\ps{k}. 
  \]
  In addition, note that if $X_s \ge \hat{x}_s$ for all $s < t$, then we have 
  $P_{\iota\ps{k}, t}\inv \hat{x}_t\ps{k} \le \hat{x}\ps{k}_{\iota\ps{k}} = X_{\iota\ps{k}}/2$.
  Therefore,
  \[
    \exists t, X_t \le \hat{x}_t 
    \;\Rightarrow\;
    \exists k \in \mbb{N}_{\ge 0}, t \in \left[ \iota\ps{k}, \iota\ps{k+1} \right) \text{ s.t. }
    \left\{
    \begin{aligned}
      & X_s \ge \hat{x}_s, \forall s < t, \\
      & \sum_{s=\iota\ps{k}}^{t-1} P_{\iota\ps{k}, s+1}\inv \left( \xi_{s+1} + Z_{s+1} \right) 
        \le -\hat{x}\ps{k}_{\iota\ps{k}}.
    \end{aligned}
    \right.
  \]
  In other words, it suffices to upper bound the probability that RHS happens before $t$. 
  To this end, we define 
  $\tau := \inf \braces{ t \ge 0 \,:\, X_t \le \hat{x}_t }$, $\hat{\xi}_{t+1} = \xi_{t+1}\indi\{ t < \tau \}$,
  and $\hat{Z}_{t+1} = Z_{t+1} \indi\{ t < \tau \}$. Then, we can further rewrite the above as 
  \begin{multline*}
    \exists t \le T, X_t \le \hat{x}_t   \\
    \Rightarrow\;
    \exists k \in \mbb{N}_{\ge 0}, t \in \left[ \iota\ps{k}, \iota\ps{k+1} \right) \text{ s.t. }
      \left|
        \sum_{s=\iota\ps{k}}^{(t \wedge T) -1} P_{\iota\ps{k}, s+1}\inv \hat{\xi}_{s+1} 
      \right|
      + \left|
        \sum_{s=\iota\ps{k}}^{(t \wedge T)-1} P_{\iota\ps{k}, s+1}\inv \hat{Z}_{s+1} 
      \right|
      \ge \hat{x}\ps{k}_{\iota\ps{k}}.
  \end{multline*}

  We now estimate the last term as follows. First, for $(\xi_t)_t$, we have $|\hat{\xi}_{t+1}| 
  \le \Xi X_t \le 4 \Xi \hat{x}\ps{k}_t$ if $t \in [\iota\ps{k}, \iota\ps{k+1})$. Therefore, 
  \[
    \left|
      \sum_{s=\iota\ps{k}}^{(t \wedge T) -1} P_{\iota\ps{k}, s+1}\inv \hat{\xi}_{s+1} 
    \right|
    \le 4 \Xi \left| \sum_{s=\iota\ps{k}}^{(t \wedge T) -1} \hat{x}\ps{k}_{s+1} \right|
    \lesssim \frac{\Xi}{ p \alpha [ \hat{x}\ps{k}_{\iota\ps{k}} ]^{p-2} },
  \]
  where the second inequality comes from Lemma~\ref{lemma: xt+1 = xt + a xtp}. For the RHS to be smaller than 
  $\hat{x}\ps{k}_{\iota\ps{k}}$, it suffices to require
  \[
    \Xi 
    \lesssim p \alpha [ \hat{x}\ps{k}_{\iota\ps{k}} ]^{p-1} 
    \quad\Leftrightarrow\quad 
    \Xi \lesssim p \alpha x_0^{p-1} .
  \]
  Also, by Lemma~\ref{lemma: xt+1 = xt + a xtp}, when the induction hypothesis is true, we have 
  $T \lesssim \left( \alpha x_0^{p-1} \right)\inv$. Thus, the above implies that with probability at least 
  $1 - \delta_{\P, \xi} / ( \alpha x_0^{p-1} )$, the total contribution of $(\xi_t)_t$ is small, as long as 
  $\Xi \lesssim p \alpha x_0^{p-1} $. 

  Then, we consider the martingale difference terms. Note that $P_{\iota\ps{k}, t+1}\inv \hat{Z}_{t+1}$ is a martingale
  difference sequence that is conditionally $( 4 \sigma_Z^2 \hat{x}\ps{k}_{\iota\ps{k}}, \theta)$-subweibull. Hence, for each 
  $k$, by Lemma~\ref{lemma: freedman, subweibull}, we have 
  \begin{align*}
    \left| \sum_{s=\iota\ps{k}}^{(t \wedge T)-1} P_{\iota\ps{k}, s+1}\inv \hat{Z}_{s+1}  \right|
    &\lesssim_\theta 
      \sqrt{ 
        \sigma_Z^2 \hat{x}\ps{k}_{\iota\ps{k}}
        \frac{\log^{\theta+1}\left( \left( \alpha [\hat{x}\ps{k}_{\iota\ps{k}} ]^{p-1} \delta_{\P, Z} \right)\inv \right)}{
          \alpha [\hat{x}\ps{k}_{\iota\ps{k}} ]^{p-1}
        } 
      } \\
    &\lesssim_\theta 
      \sigma_Z
      \sqrt{  
        \frac{\log^{\theta+1}\left( 1 / ( \alpha x_0^{p-1} \delta_{\P, Z} ) \right)}{
          \alpha x_0^{p-2}
        } 
      },
  \end{align*}
  with probability at least probability at least $1 - \delta_{\P, Z}$. For the RHS to be smaller than $x_0$, it suffices 
  to require 
  \[
    \sigma_Z^2 
    \lesssim_\theta 
      \frac{ \alpha x_0^p }{\log^{\theta+1}\left( 1 / ( \alpha x_0^{p-1} \delta_{\P, Z} ) \right)} .
  \]
  Recall that we recouple at most $O(\log(1/x_0))$ times. Hence, it suffices to replace $\delta_{\P, Z}$
  with $\delta_{\P}/\log(1/x_0)$ to ensure the total contribution of $(Z_t)_t$ is small. 
\end{proof}

\subsection{Deferred proofs}

\begin{proof}[Proof of Lemma~\ref{lemma: freedman, subweibull}]
  First, consider the case where $\sigma = 1$. 
  Let $M \ge 1$ be a parameter to be chosen later. Then, define $\hat{Z}_t = Z_t\indi\braces{ |Z_t| \le M }$ and write 
  \[
    \sum_{t=1}^{T} Z_t 
    = \sum_{t=1}^{T} \left( \hat{Z}_t - \E\left[ \hat{Z}_t \mid \cF_{t_1} \right] \right)
      + \sum_{t=1}^{T} \E\left[ \hat{Z}_t \mid \cF_{t_1} \right]
      + \sum_{t=1}^{T} Z_t \indi\braces{ |Z_t| > M }   
    =: \Term_1 + \Term_2 + \Term_3.
  \]
  Since $Z_t$ is conditionally $(1, \theta)$-subweibull, we have 
  \[
    \P(\Term_3 \ne 0) 
    \le \P\left( \exists t \in [T], |Z_t| \ge M \right) 
    \le C T \exp\left( -M^{1/\theta} \right).
  \]
  For the last term to be bounded by $\delta_{\P}$, it suffices to choose 
  \[
    M \ge \log^\theta\left( CT / \delta_{\P} \right).
  \]
  Then, we consider $\Term_2$. Since $\E[ Z_t \mid \cF_{t-1} ] = 0$, we have
  \[
    \E\left[ \hat{Z}_t \mid \cF_{t-1} \right] 
    = \E\left[ \hat{Z}_t - Z_t \mid \cF_{t-1} \right]
    = \E\left[ Z_t \indi\{ |Z_t| > M \} \mid \cF_{t-1} \right].
  \]
  For the last term, using the layer cake representation, we obtain 
  \begin{align*}
    \left| \E\left[ Z_t \indi\{ |Z_t| > M \} \mid \cF_{t-1} \right] \right|
    &\le \E\left[ |Z_t| \indi\{ |Z_t| \ge M \} \mid \cF_{t-1} \right] \\
    &= \int_0^\infty \P\left( 
        |Z_t| \indi\{ |Z_t| \ge M \} \ge s
        \mid \cF_{t-1}
      \right) \,\rd s \\
    &= \int_0^\infty \P\left( |Z_t| \ge M \vee s \mid \cF_{t-1} \right) \,\rd s \\
    &= M \P\left( |Z_t| \ge M \mid \cF_{t-1} \right) 
      + \int_M^\infty \P\left( |Z_t| \ge s \mid \cF_{t-1} \right) \,\rd s. 
  \end{align*}
  Therefore, for each summand in $\Term_2$, we have 
  \begin{align*}
    \left| \E\left[ \hat{Z}_t \mid \cF_{t-1} \right]  \right|
    &\le C M \exp\left( - M^{1/\theta} \right)
      + \int_M^\infty 
        C \exp\left( - s^{1/\theta} \right)
      \,\rd s \\
    &= C M \exp\left( - M^{1/\theta} \right)
      + C \sigma \int_{M^{1/\theta}}^\infty e^{-s} s^{\theta-1} \,\rd s. 
  \end{align*}

  Note that if $s / \log s \ge 2 (\theta - 1)$, we have $e^{-s} s^{\theta-1} = e^{-s + (\theta - 1)\log s } \le e^{-s/2}$.
  Hence, if we choose $M$ such that 
  \[
    \frac{ M^{1/\theta} }{ \log\left( M^{1/\theta} \right) }
    \ge 2 (\theta - 1)
    \quad\Leftarrow\quad 
    \frac{ M  }{ \log^\theta M  }
    \ge \left( \frac{2 (\theta - 1)}{\theta} \right)^\theta,
  \] 
  then we have 
  \[
    C \int_{M^{1/\theta}}^\infty e^{-s} s^{\theta-1} \,\rd s
    \le C \int_{M^{1/\theta}}^\infty e^{-s/2} \,\rd s
    = 2 C \exp\left( - \frac{1}{2} M^{1/\theta} \right) .
  \]
  As a result, we have 
  \[
    |\Term_2|
    \le C T \left( 
        M \exp\left( - M^{1/\theta} \right)
        + 2 \exp\left( - \frac{1}{2} M^{1/\theta} \right)
      \right) 
    \le 4 C M T \exp\left( - \frac{1}{2} M^{1/\theta} \right).
  \]
  Finally, consider $\Term_1$. Note that $( \hat{Z}_t - \E[\hat{Z}_t \mid \cF_{t-1}] )_t$ is a martingale difference
  that is bounded by $2 M$ and has conditional variance bounded by $C_\theta$ for some $C_\theta > 0$. Therefore,
  by Bernstein's inequality, we have 
  \[
    \P( |\Term_1| \ge K )
    \le 2 \exp\left( - \frac{( K / \sqrt{T} )^2}{ C_\theta + 2 M } \right).
  \]
  For the RHS to be bounded by $\delta_{\P}$, it suffices to require
  \[
    K \ge \sqrt{T} \sqrt{ \left( C_\theta + 2 M \right) \log\left( 2 / \delta_{\P} \right) }
  \]
  Finally, combining the above analysis, we obtain 
  \[
    \left| \sum_{t=1}^{T} Z_t \right|
    \le 4 C M T \exp\left( - M^{1/\theta} / 2 \right)
      + \sqrt{T} \sqrt{ \left( C_\theta + 2 M \right) \log\left( 2 / \delta_{\P} \right) },
  \]
  with probability at least $1 - 2 \delta_{\P}$, where $M \ge \log^\theta\left( CT / \delta_{\P} \right)$ and 
  $M / \log^\theta M  \ge \left( \frac{2 (\theta - 1)}{\theta} \right)^\theta$. 
  Now, we simplify the RHS as follows. 
  Note that 
  \begin{multline*}
    4 C M T \exp\left( - \frac{1}{2} M^{1/\theta} \right)
    \le \sqrt{T} \sqrt{ 2 M \log\left( 2 / \delta_{\P} \right) } \\
    \Leftarrow\quad
    \exp\left( \frac{1}{2} \log M - \frac{1}{2} M^{1/\theta} \right)
    \le \frac{\sqrt{ 2 \log\left( 2 / \delta_{\P} \right) }}{4 C \sqrt{T}}  \\
    \Leftarrow\quad
    \frac{M}{\log^\theta M}  \ge 2^\theta , \quad 
    M \ge  4^\theta \log^\theta\left( \frac{8 C^2 T}{ \log\left( 2 / \delta_{\P} \right) }\right). 
  \end{multline*}
  In other words, we can choose 
  \[
    M = \Theta_{\theta}\left( \log^\theta(T/\delta_{\P}) \right), 
  \]
  and obtain 
  \[
    \left| \sum_{t=1}^{T} Z_t \right|
    \lesssim_\theta  \sqrt{T \log^{\theta+1}\left( T/\delta_{\P} \right) }. 
  \]
  Finally, for general $\sigma > 0$, it suffices to note that if $X$ is $(\sigma^2, \theta)$-subweibull, then $X/\sigma$ is 
  $(1, \theta)$-subweibull.

\end{proof}

\begin{proof}[Proof of Lemma~\ref{lemma: xt+1 = xt + a xtp}]
  First, we consider the upper bound on $T$. We compute
  \begin{align*}
    \alpha
    = \frac{x_{t+1} - x_t}{x_t^p}
    = \frac{x_{t+1}^p}{x_t^p} \frac{x_{t+1} - x_t}{x_{t+1}^p}
    &= \frac{x_{t+1}^p}{x_t^p} \int_{x_t}^{x_{t+1}} \frac{1}{x_{t+1}^p} \,\rd y \\
    &\le \frac{x_{t+1}^p}{x_t^p} \int_{x_t}^{x_{t+1}} \frac{1}{y^p} \,\rd y
    = \frac{x_{t+1}^p}{x_t^p} \frac{1}{p - 1} \left( \frac{1}{x_t^{p-1}} - \frac{1}{x_{t+1}^{p-1}} \right). 
  \end{align*}
  In addition, note that 
  \(
    x_{t+1}^p / x_t^p
    = \left( 1 + \alpha x_t^{p-1} \right)^p
    \le \left( 1 + \alpha \right)^p
    \le e^{\alpha p}.
  \)
  Therefore, 
  \[
    \alpha
    \le \frac{e^{\alpha p}}{p - 1} \left( \frac{1}{x_t^{p-1}} - \frac{1}{x_{t+1}^{p-1}} \right)
    \quad\Rightarrow\quad
    \frac{1}{x_{t+1}^{p-1}}  \le \frac{1}{x_t^{p-1}} - \frac{(p - 1) \alpha }{e^{\alpha p}} .
  \]
  Sum both sides from $0$ to $t-1$ and we get 
  \[
    \frac{1}{x_t^{p-1}}  \le \frac{1}{x_0^{p-1}} - \frac{t (p - 1) \alpha }{e^{\alpha p}}
    \quad\Rightarrow\quad 
    x_t \ge \left( \frac{1}{x_0^{p-1}} - e^{-\alpha p} (p - 1) \alpha t \right)^{-\frac{1}{p-1}}.
  \]
  In particular, this implies 
  \[
    T \le \left( \frac{1}{x_0^{p-1}} - 1 \right) \frac{e^{\alpha p}}{(p - 1) \alpha }.
  \]

  Now, we consider the upper bound on $\sum_t x_t$. Let $(\tilde{x}_h)_h$ be the solution to the continuous-time ODE
  $\frac{\rd}{\rd t} \tilde{x}_h = \tilde{x}_h^p$ with $\tilde{x}_0 = x_0$. Note that $\tilde{x}$ is increasing and 
  therefore
  \[
    \tilde{x}_{(t+1)\alpha}
    = \tilde{x}_{t\alpha} + \int_0^\alpha \tilde{x}_{t\alpha + r}^p \,\rd r
    \ge \tilde{x}_{t\alpha} + \alpha a \tilde{x}_{t\alpha}^p.  
  \]
  Hence, by induction, we have $\tilde{x}_{t\alpha} \ge x_t$ for all $t$. 
  In addition, $\tilde{x}_h$ has the closed-form formula:
  \[
    \tilde{x}_h
    = \left( \frac{1}{x_0^{p-1}} - (p - 1) h \right)^{-\frac{1}{p-1}}.
  \]
  Thus, 
  \begin{align*}
    \sum_{t=0}^{T-1} x_t 
    \le \sum_{t=0}^{T-1} \tilde{x}_{t\alpha}
    \le \alpha\inv \sum_{t=0}^{T-1} \int_{t\alpha}^{(t+1)\alpha} \tilde{x}_s \,\rd s
    &= \frac{1}{\alpha} \int_0^{T\alpha} \tilde{x}_h \,\rd h  \\
    &= \frac{1}{\alpha} \int_0^{T\alpha} \left( \frac{1}{x_0^{p-1}} - (p - 1) h \right)^{-\frac{1}{p-1}} \,\rd h.
  \end{align*}
  When $p = 2$, we have 
  \[
    \sum_{t=0}^{T-1} x_t 
    \le \frac{1}{\alpha} \int_0^{T\alpha} \left( \frac{1}{x_0} - h \right)\inv \,\rd h
    = \frac{1}{\alpha} \log\left( \frac{1}{1 - x_0 T \alpha} \right)
    \le \frac{1}{\alpha} \log\left( \frac{1}{ 1 - e^{2 \alpha} + x_0 e^{2 \alpha} } \right)
    \le \frac{2}{\alpha},
  \]
  as long as $\alpha \lesssim x_0$ so that $\tilde{x}_{T\alpha} = O(1)$. 
  When $p > 2$, to have $\tilde{x}_{T \alpha} \le 2$, it suffices to have 
  \[
    \frac{1}{x_0^{p-1}} - (p - 1) \alpha T  
    \ge \frac{1}{2^{p-1}}
    \quad\Leftarrow\quad
    e^{\alpha p}  
    - \frac{e^{\alpha p} - 1 }{x_0^{p-1}}
    \ge \frac{1}{2^{p-1}}
    \quad\Leftarrow\quad
    \alpha 
    \lesssim x_0^{p-1}/p.
  \]
  Let $\tilde{T}$ be the time $\tilde{x}_h$ reaches $2$. We have $\tilde{T} \le \frac{1}{(p - 1) x_0^{p-1}}$ and 
  \begin{align*}
    \sum_{t=0}^{T-1} x_t 
    &\le \frac{1}{\alpha} \int_0^{\tilde{T}} \left( \frac{1}{x_0^{p-1}} - (p - 1) h \right)^{-\frac{1}{p-1}} \,\rd h \\
    &= \frac{1}{\alpha} 
      \frac{1}{p - 2}
      \left( \frac{1}{x_0^{p-1}} - (p - 1) \tilde{T} \right)^{-\frac{1}{p-1}} \\
      &\qquad\times
      \left(
        (p - 1) \tilde{T}
        + \frac{1}{x_0^{p-1}} 
        \left(
          -1 
          + \left(\frac{1}{x_0^{p-1}}\right)^{-\frac{1}{p-1}}
           \left( \frac{1}{x_0^{p-1}} - (p - 1) T \right)^{\frac{1}{p-1}}
        \right)
      \right) \\
    &\lesssim 
      \frac{1}{(p - 2) \alpha} 
      \frac{1}{x_0^{p-2}}. 
  \end{align*}
\end{proof}

\section{Simulation}
\label{sec: simulation}

We include simulation results for Stage~1 in this section. The goal here is to provide empirical evidence that 
(i) if we have both the second- and $L$-th order terms, then the sample complexity of online SGD scales 
linearly with $d$ and (ii) without the higher-order terms, online SGD cannot recovery the exact directions. 

The setting is the same as the one we have described in Section~\ref{sec: setup}. We choose the hyperparameters 
roughly according to Theorem~\ref{thm: main}.
To reduce the demand of computational resources, we choose $m = \Theta(P^2)$ instead of $\tilde{\Omega}(P^8)$. Note 
that by the Coupon Collector problem, we need $m = \Omega(P \log P)$ to ensure that for each $p \in [P]$, there exists
at least one neuron $\v$ with $v_p^2 \ge \max_{q \le P} v_q^2$. Since we are mostly interested in the dependence
on $d$, for the learning rate, we choose $\eta = c / d$, where $c$ is a tunable constant that is independent of $d$ 
but can depend on everything else. $T$ is chosen according to Theorem~\ref{thm: main} and we early-stop the 
training when for all $p \in [P]$, there exists a neuron with $v_p^2 \ge 0.95$ (in the moving average sense). 

All experiments are performed on the authors' laptop without using GPUs, and it takes less than one day to complete
the experiments.

\begin{figure}[h]
  \centering
  \begin{subfigure}{.45\textwidth}
    \centering
    \includegraphics[width=.95\linewidth]{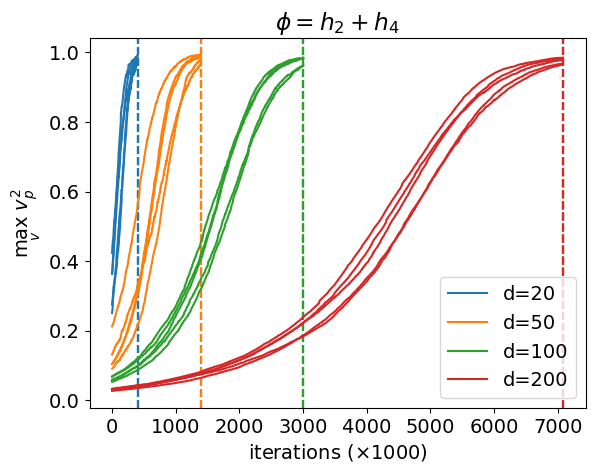}
  \end{subfigure}%
  \begin{subfigure}{.45\textwidth}
    \centering
    \includegraphics[width=.95\linewidth]{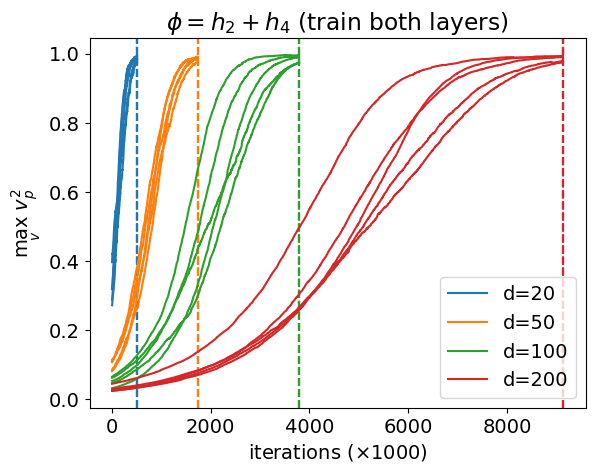}
  \end{subfigure} \\
  \begin{subfigure}{.45\textwidth}
    \centering
    \includegraphics[width=.95\linewidth]{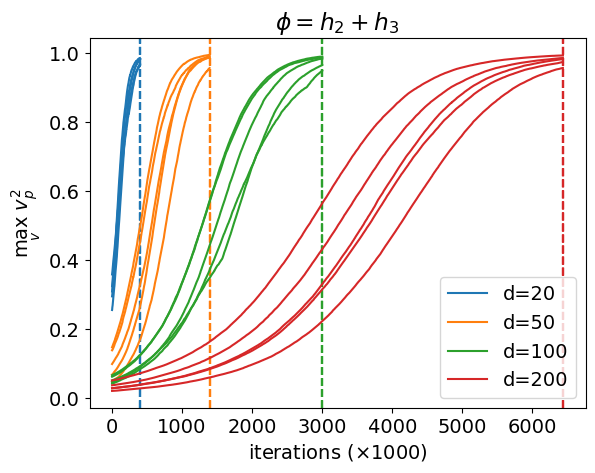}
  \end{subfigure}
  \begin{subfigure}{.45\textwidth}
    \centering
    \includegraphics[width=.95\linewidth]{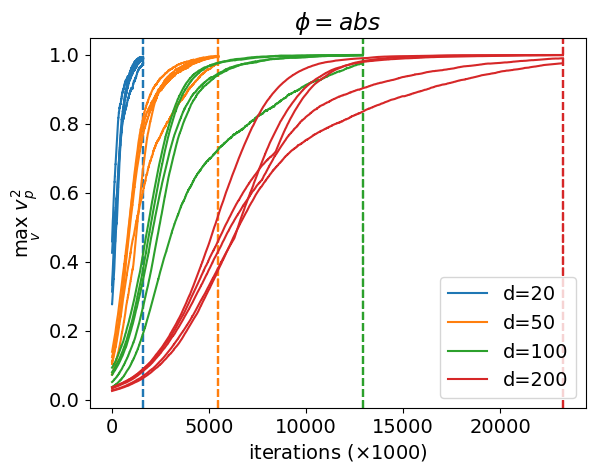}
  \end{subfigure} \\
  \caption{Recovery of directions. 
    The above plots show the evolution of the correlation with each of the ground-truth directions.
    We fix the relevant dimension $P = 5$ and vary the ambient dimension $d$. Different colors represent different 
    $d$. For each color, one curve represents $\max_{\v} v_p^2$ for one $p \in [P]$. 
    In the first row, the link function is $\phi = h_2 + h_4$.  
    In the left plot, we use the algorithm \eqref{eq: training algorithm}, while in the right plot, 
    we train both layers simultaneously. 
    The second row contains simulation results for other link functions. 
  }
  \label{fig: learning the directions (h2 + h4 and abs)}
  \end{figure}

\end{document}